\documentclass{article}

 \usepackage{fullpage}

\usepackage{graphicx}
\usepackage{subfigure}
\usepackage{amsmath, amsfonts,amssymb}
\usepackage{url}
\usepackage{epstopdf, epsfig}
\usepackage{algorithm}
\usepackage{algorithmic}
\usepackage{caption}
\usepackage{latexsym,xspace,makeidx,bm}
\usepackage{color}

\newcommand{\calE}{\ensuremath{\mathcal{E}}}
\newcommand{\calF}{\ensuremath{\mathcal{F}}}
\newcommand{\calA}{\ensuremath{\mathcal{A}}}

\newcommand{\calI}{\ensuremath{\mathcal{I}}}
\newcommand{\calS}{\ensuremath{\mathcal{S}}}

\newcommand{\bP}{\ensuremath{\mathbf{P}}}

\newcommand{\bx}{\ensuremath{\mathbf{x}}}

\newcommand{\bw}{\ensuremath{\mathbf{w}}}

\newcommand{\bs}{\ensuremath{\mathbf{s}}}

\newcommand{\bSigma}{\ensuremath{\boldsymbol \Sigma}}
\newcommand{\balpha}{\ensuremath{\boldsymbol \alpha}}

\newcommand{\bdelta}{\ensuremath{\boldsymbol \delta}}
\newcommand{\bmu}{\ensuremath{\boldsymbol \mu}}

\newcommand{\bOmega}{\ensuremath{\boldsymbol \Omega}}
\newcommand{\btheta}{\ensuremath{\boldsymbol \theta}}
\newcommand{\B}{\ensuremath{\text{Beta}}}
\newcommand{\Dir}{\ensuremath{\text{Dir}}}
\newcommand{\var}{\ensuremath{\text{VaR}}}
\newcommand{\cvar}{\ensuremath{\text{CVaR}}}
\newcommand{\brho}{\ensuremath{\boldsymbol \rho}}
\newcommand{\E}{\mathbb{E}}
\newcommand{\V}{\text{Var}}

\newcommand{\argmax}{\mathop{\mathrm{arg\,max}{}}}
\newcommand{\tE}{\widetilde{\mathbb{E}}}

\newtheorem{theorem}{Theorem}[section]
\newtheorem{lemma}[theorem]{Lemma}
\newtheorem{proposition}[theorem]{Proposition}
\newtheorem{corollary}[theorem]{Corollary}
\newtheorem{assumption}[theorem]{Assumption}

\newenvironment{proof}[1][Proof]{\begin{trivlist}
\item[\hskip \labelsep {\bfseries #1}]}{\end{trivlist}}

\newenvironment{remark}[1][Remark]{\begin{trivlist}
\item[\hskip \labelsep {\bfseries #1}]}{\end{trivlist}}

\newcommand{\qed}{\nobreak \ifvmode \relax \else
      \ifdim\lastskip<1.5em \hskip-\lastskip
      \hskip1.5em plus0em minus0.5em \fi \nobreak
      \vrule height0.75em width0.5em depth0.25em\fi}

\title{Statistical Decision Making for Optimal Budget Allocation in Crowd Labeling}

\author{Xi Chen \\ UC Berkeley \\ xichen@cs.cmu.edu \and Qihang Lin \\ University of Iowa  \\qihang-lin@uiowa.edu \and Dengyong Zhou \\ Microsoft Research \\dengyong.zhou@microsoft.com }

\begin{document}

\date{}
\maketitle

\begin{abstract}
  There is an increasing popularity in crowdsourcing  data labeling tasks to non-expert workers or annotators recruited through commercial internet services such as  Amazon Mechanical Turk. Those crowdsourcing workers need to be paid for each label they provide, while a task  requester usually only has a limited amount of budget for data labeling. So it is desirable to have an optimal policy to wisely allocate the budget among workers and data instances which need to label by considering worker reliability and task difficulty such that the quality of the finally aggregated labels can be maximized. We formulate such a  budget allocation problem as a Bayesian Markov decision process (MDP) which simultaneously conducts learning and decision making. Under our framework,  the optimal allocation policy can be obtained  by applying dynamic programming (DP), but DP quickly becomes computationally intractable when the size of the problem increases. To solve this challenge, we propose a computationally efficient approximate policy called optimistic knowledge gradient policy. Experiments on both synthetic and real data show that at the same budget level our policy results in higher quality labels  than  existing policies. 
\end{abstract}

\section{Introduction}

In many machine learning applications, data are usually collected without labels. For example, a digital camera does not automatically tag a picture as a portrait or a landscape. 
A traditional way for data labeling is to hire a small group of experts to provide labels  for the entire set of data. However, for large-scale data, such an approach  becomes inefficient and very costly. Thanks to the advent of online crowdsourcing services such as Amazon Mechanical Turk, a much more efficient way is to post unlabeled data to a crowdsourcing marketplace, where a big crowd of low-paid workers can be hired instantaneously to perform labeling tasks.


Despite of its high efficiency and immediate availability, crowd labeling raises many new challenges. Since labeling tasks are tedious and workers are usually non-experts, labels generated by the crowd suffer from low quality. As a remedy,  most crowdsourcing services resort to labeling redundancy to reduce the labeling noise, which is achieved by collecting multiple labels from different workers for each data instance. In particular, a crowd labeling process can be described as a two phase procedure:
\begin{enumerate}
  \item Assign unlabeled data  to a crowd of workers and each data instance is asked to label multiple times;
  \item Aggregate the collected raw labels to infer the true labels.
\end{enumerate}
In principle, more raw labels will lead to a higher chance of recovering the true label. However, each raw label comes with a cost: the requester has to pay workers  pre-specified monetary reward for each label they provide, usually, regardless of the label's correctness. For example, a worker typically earns 10 cents by categorizing a website as porn or not. In practice, the requester has only a limited amount of budget which essentially restricts the total number of raw labels that he/she can collect. This raises a challenging question central in crowd labeling: \emph{What is the best way to allocate the budget among data instances and workers so that the overall accuracy of aggregated labels is maximized ? }


The most important factors that decide how to allocate the budget are the intrinsic characteristics of data instances and workers: \emph{labeling difficulty/ambiguity for each data instance} and \emph{reliability/quality of each worker}. In particular, an instance is less ambiguous if its label can be decided based on the common knowledge and a vast majority of reliable workers will provide the same label for it. In principle, we should avoid spending too much budget on those easy instances since excessive raw labels will not bring much additional information. In contrast, for an ambiguous instance which falls near the boundary of categories, even those reliable workers will still disagree with each other and generate inconsistent labels. For those ambiguous instances, we are facing a challenging decision problem on how much budget that we should spend on them. On one hand, it is worth to collect more labels to boost the accuracy of the aggregate label. On the other hand, since our goal is to  maximize the \emph{overall} labeling accuracy, when the budget is limited, we should simply put those few highly ambiguous instances aside to save budget for labeling less difficult instances. In addition to the ambiguity of data instances, the other important factor is the reliability of each worker and, undoubtedly, it is desirable to assign  more instances to those reliable workers. Despite of their importance in deciding how to allocate the budget, both the data ambiguity and workers' reliability are unknown parameters at the beginning and need to be updated based on the stream of collected raw labels in an online fashion. This further suggests that the budget allocation policy should be dynamic and simultaneously conduct parameter estimation and decision making.

To search for an optimal budget allocation policy, we model the data ambiguity and workers' reliability using two sets of random variables drawn from known prior distributions. Then, we formulate the problem into a finite-horizon Bayesian Markov Decision Process (MDP) \cite{Puterman:05}, whose state variables are the posterior distributions of these variables, which are updated by each new label. Here, the Bayesian setting is necessary. We will show that an optimal policy only exists in the Bayesian setting. Using the MDP formulation, the optimal budget allocation policy for any finite budget level can be readily obtained via the dynamic programming (DP). However, DP is computationally intractable for large-scale problems since the size of the state space grows exponentially in budget level. The existing widely-used approximate policies, such as approximate Gittins index rule \cite{Gittins:89} or knowledge gradient (KG) \cite{Gupta:96,Frazier:08}, either has a high computational cost or poor performance in our problem. In this paper, we propose a new policy, called \emph{optimistic knowledge gradient (Opt-KG)}. In particular, the Opt-KG policy dynamically chooses the next instance-worker pair  based on the optimistic outcome of the marginal improvement on the accuracy, which is a function of state variables. We further propose a more general Opt-KG policy using the conditional value-at-risk measure \cite{Rockafellar:02}. The Opt-KG is computationally efficient, achieves superior empirical performance and has some asymptotic theoretical guarantees.

To better present the main idea of our MDP formulation and the Opt-KG policy, we start from the binary labeling task (i.e., providing the category, either positive or negative, for each instance). We first consider the \emph{pull marketplace} (e.g., Amazon Mechanical Turk or Galaxy Zoo) , where the labeling requester can only post instances to the general worker pool with either anonymous or transient workers, but cannot assign to an identified worker. In a pull marketplace, workers are typically treated as \emph{homogeneous} and one models the entire worker pool instead of each individual worker.  We further assume that workers are fully reliable (or noiseless) such that the chance that they make an error only depend on instances' own ambiguity. At a first glance, such an assumption may seem oversimplified. In fact, it turns out that the budget-optimal crowd labeling under such an assumption has been highly non-trivial. We formulate this problem into a Bayesian MDP and propose the computational efficient Opt-KG policy. We further prove that the Opt-KG policy in such a setting is asymptotically consistent, that is,  when the budget goes to infinity, the accuracy converges to 100\% almost surely.

Then, we extend the MDP formulation to deal with \emph{push marketplaces} with \emph{heterogeneous} workers. In a push marketplace (e.g., data annotation team in   Microsoft Bing group), once an instance is allocated to an identified worker, the worker is required to finish the instance in a short period of time. Based on the previous model for fully reliable workers, we further introduce another set of parameters to characterize workers' reliability. Then our decision process simultaneously selects the next instance to label and the next worker for labeling the instance according to the optimistic knowledge gradient policy. 
In fact, the proposed MDP framework is so flexible that we can further extend it to incorporate  contextual information of instances whenever they are available (e.g., as in many web search and advertising applications \cite{Li:10}) and to handle multi-class labeling.

In summary, the main contribution of the paper consists of the three folds: (1) we formulate the budget allocation in crowd labeling into a MDP and characterize the \emph{optimal} policy using DP; (2) computationally, we propose an efficient approximate policy, optimistic knowledge gradient;  (3) the proposed MDP framework can be used as a general framework to address various budget allocation problems in crowdsourcing (e.g., rating and ranking tasks).

The rest of this paper is organized as follows. In Section \ref{sec:binary}, we first present the modeling of budget allocation process for binary labeling tasks with fully reliable workers and motivate our Bayesian modeling. In Section \ref{sec:MDP}, we present the Bayesian MDP and the optimal policy via DP. In Section \ref{sec:Opt-KG}, we propose a computationally efficient approximate policy, Opt-KG. In Section \ref{sec:worker}, we extend our MDP to model heterogeneous workers with different reliability.  In Section \ref{sec:extension}, we present other important extensions, including incorporating contextual information and multi-class labeling. In Section \ref{sec:related}, we discuss the related works. In Section \ref{sec:exp},  we present numerical
results on both simulated and real datasets, followed by conclusions in Section \ref{sec:conclusion}.

\section{Binary Labeling with Homogeneous Noiseless Workers}
\label{sec:binary}

We first consider the budget allocation problem in a pull marketplace with homogeneous noiseless workers for binary labeling tasks. We note that such a simplification is important for investigating this problem, since the incorporation of workers' reliability and extensions to multiple categories  become rather straightforward once this problem is correctly modeled (see Section \ref{sec:worker} and \ref{sec:extension}).


Suppose that there are $K$ instances and each one is associated with a latent true label $Z_i \in \{-1, 1\}$ for $1 \leq i \leq K$.  Our goal is to infer the set of  positive instances, denoted by $H^*=\{i: Z_i=1\}$. Here, we assume that the homogeneous worker pool is \emph{fully reliable} or \emph{noiseless}. We note that it does not mean that each worker knows the true label $Z_i$. Instead, it means that fully reliable workers will do their best to make judgements but their labels may be still incorrect due to the instance's ambiguity. Further, we model  the \emph{labeling difficulty/ambiguity} of each instance by a latent soft-label $\theta_i$, which can be interpreted as the percentage of workers in the homogeneous noiseless crowd who will label the $i$-th instance as positive. In other words, if we randomly choose a worker from a large crowd of fully reliable workers, we will receive a positive label for the $i$-th instance with probability $\theta_i$ and a negative label with probability $1-\theta_i$. In general, we assume the crowd is large enough so that the value of $\theta_i$ can be any value in $[0,1]$. To see how $\theta_i$ characterizes the labeling difficulty of the $i$-th instance, we consider a concrete example where a worker is asked to label a person as adult (positive) or not (negative) based on the photo of that person. If the person is more than 25 years old, most likely, the corresponding $\theta_i$ will be close to 1, generating positive labels consistently. On the other hand, if the person is younger than 15, she may be labeled as negative by almost all the reliable workers since $\theta_i$ is close to 0. In both of this cases, we regard the instance (person) easy to label since $Z_i$ can be inferred with a high accuracy based on only a few raw labels. On the contrary, for a person is one or two years below or above 18, the $\theta_i$ is near 0.5 and the numbers of positive and negative labels become relatively comparable so that the corresponding labeling task is very difficult. Given the definition of soft labels, we further make the following assumption:

\begin{assumption}
  We  assume that the soft-label $\theta_i$ is consistent with the true label in the sense that $Z_i=1$ if and only if $\theta_i \geq 0.5$, i.e., the majority of the crowd are correct, and hence $H^*= \{i: \theta_i \geq 0.5 \}$.
  \label{assump:sf_consistent}
\end{assumption}

Given the total budget, denoted by $T$, we suppose that each label costs one unit of budget. As discussed in the introduction, the crowd labeling has two phases. The first phase is the \emph{budget allocation phase}, which is a dynamic decision process with $T$ stages. In each stage  $0 \leq t \leq T-1$, an instance $i_t\in\calA=\{1,\ldots, K\}$ is selected based on the historical labeling results. Once $i_t$ is selected, it will be labeled by a random worker  from the homogeneous noiseless worker pool. According to the definition of $\theta_{i_t}$, the label received, denoted by $y_{i_t} \in \{-1,1\}$, will follow the Bernoulli distribution with the parameter $\theta_{i_t}$:
\begin{equation}
\Pr\left(y_{i_t}=1\right)=\theta_{i_t}  \qquad  \text{and} \qquad \Pr\left(y_{i_t}=-1\right)=1-\theta_{i_t}.
\end{equation}
We note that, at this moment, all workers are assumed to be homogeneous and noiseless so that $y_{i_t}$ only depends on $\theta_{i_t}$ but not on which worker provides the label. Therefore, it is suffice for the decision maker (e.g., requester or crowdsourcing service) to select the instance in each stage instead of an instance-worker pair.

The second phase is the \emph{label aggregation phase}. When the budget is exhausted, the decision maker needs to infer true labels $\{Z_i\}_{i=1}^n$ by aggregating all the collected labels. According to Assumption \ref{assump:sf_consistent}, it is equivalent to infer the set of positive instances whose $\theta_i\geq 0.5$. 
Let $H_T$ be the estimated positive set. The final overall accuracy is measured by $|H_T \cap H^*| + |(H_T)^c \cap (H^*)^c|$, the size of the mutual overlap between $H^*$ and $H_T$.

Our goal is to determine the \emph{optimal allocation policy},  $(i_0, \ldots, i_{T-1})$, so that overall accuracy is maximized. Here, a natural question to ask is whether the \emph{optimal} allocation policy  exists and what assumptions do we need for the existence of the optimal policy. To answer this question, we provide a concrete example, which motivates our Bayesian modeling.


\subsection{Why we need a Bayesian modeling}
\label{sec:example}

\renewcommand{\arraystretch}{1.2}
\begin{table}[!t]
\renewcommand{\tabcolsep}{15pt}
\centering
    \caption{A toy example with 3 instances to label.  Five labels have been collected. Assume that we have the budget for one more label.  Which instance should be selected to label? }
    \begin{tabular}{|c|c|c|c|}\hline
    Instance 1 ($\theta_1$) & 1 & 1 & \text{label?}\\ \hline
    Instance 2 ($\theta_2$) & 1 & $-1$ & \text{label?} \\ \hline
    Instance 3 ($\theta_3$) & 1 & \text{label?} &  \\ \hline
    \end{tabular}
    \label{tab:example}
\end{table}


\renewcommand{\arraystretch}{1.2}
\begin{table}[!t]
\renewcommand{\tabcolsep}{6pt}
\centering
\caption{Expected improvements in accuracy for collecting an extra label, i.e., the expected accuracy of obtaining one more label minus the current expected accuracy.  The 3rd and 4th columns contain the accuracies with the next label being 1 and $-1$. The 5th is the expected accuracy which is computed by taking $\theta$ times the 3rd column plus $(1-\theta)$ times the 4th. The last column contains the expected improvements which is computed by taking the difference between the 5th and 2nd columns.}
\begin{tabular}{|c|c|c|c|c|c|c|c|} \hline
               & Current Accuracy   & $y=1$       & $y=-1$      & Expected Accuracy                    & Improvement   \\ \hline
$\theta_1>0.5$  & 1       & 1         & 1         & 1                           &  0        \\ \hline
$\theta_1<0.5$  & 0       & 0         & 0         & 0                           &  0        \\ \hline
$\theta_2>0.5$  & 0.5     & 1         & 0         & $\theta_2$                  & $\theta_2-0.5>0$ \\ \hline
$\theta_2<0.5$  & 0.5     & 0         & 1         & $1-\theta_2$                & $0.5-\theta_2>0$\\ \hline
$\theta_3>0.5$  & 1       & 1         & 0.5       & $\theta_3+0.5(1-\theta_3)$  & $0.5(\theta_3-1)<0$ \\ \hline
$\theta_3<0.5$  &  0      & 0         & 0.5       & $0.5(1-\theta_3)$            & $0.5 (1-\theta_3)>0$ \\ \hline
\end{tabular}
\label{tab:acc}
\end{table}

\begin{figure}[!t]
\centering
  \includegraphics[width=0.5\textwidth]{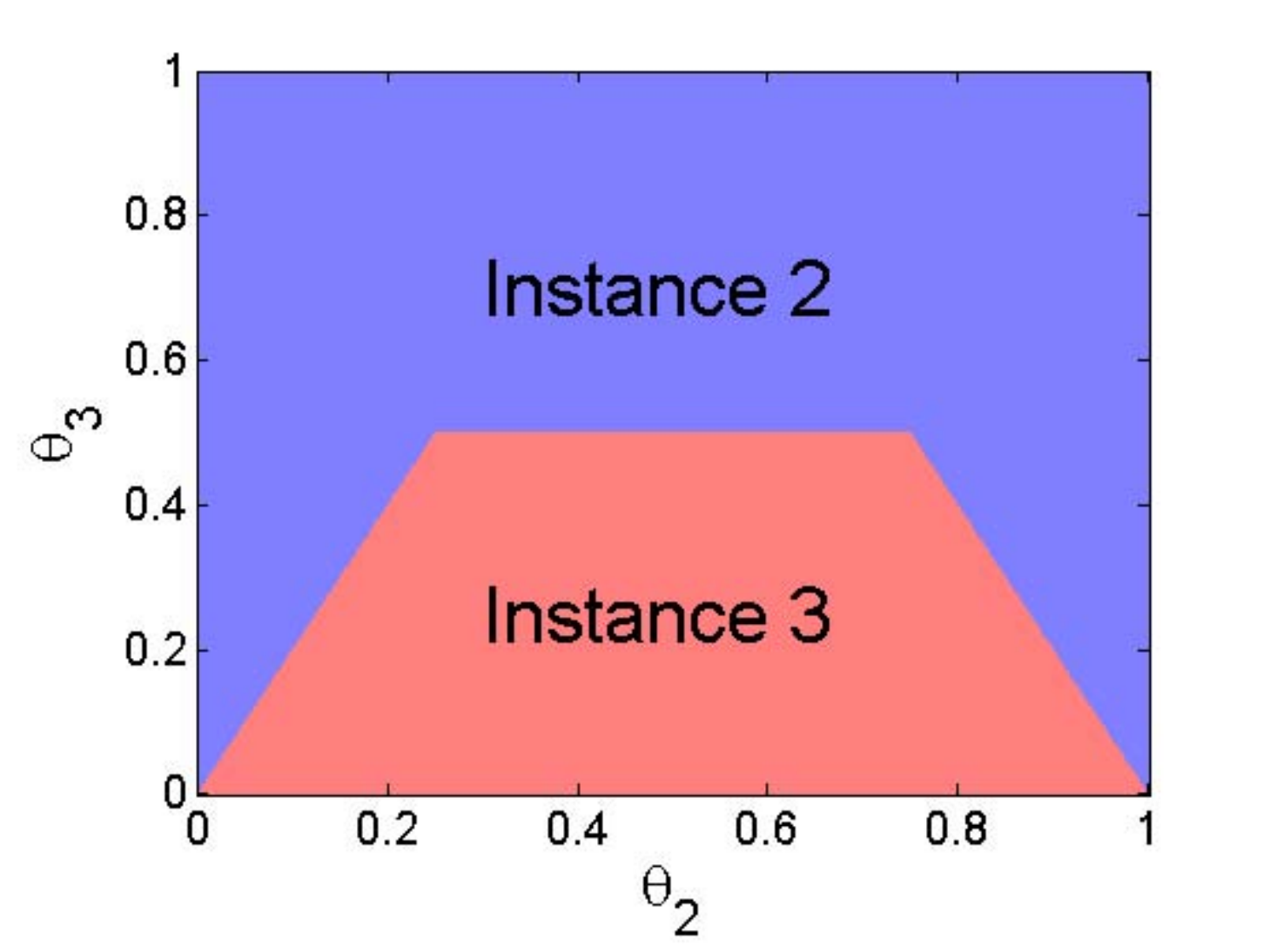}
  \caption{Decision Boundary.}
  \label{fig:dec}
\end{figure}

Let us check a toy example with 3 instances and 5 collected labels (see Table \ref{tab:example}). We assume that the workers are homogenous noiseless and the label aggregation is performed by the majority vote rule. \emph{Now if we only have the budget to get one more label, which instance  should be chosen to label?} It is obvious that we should not put the remaining budget on the first instance since we are relatively more confident on what its true label should be. Thus, the problem becomes how to choose between the second and third instances. In what follows, we shall show that \emph{there is no optimal policy under the frequentist setting}. To be more explicit, the optimal policy leads to the expected accuracy which is at least as good as that of all other policies for any values of $\{\theta_i\}_{i=1}^n$.

Let us compute the expected improvement in accuracy in terms of the frequentist risk in Table \ref{tab:acc}. We assume that $\theta_i \neq 0.5$ and if the number of 1 and $-1$ labels are the same for an instance, the accuracy is 0.5 based on a random guess. From Table \ref{tab:acc}, we should not label the first instance since the improvement is always 0. This coincides with our intuition. When $\max(\theta_2-0.5, 0.5-\theta_2)>0.5(1-\theta_3)$ or $\theta_3>0.5$, which corresponds to the blue region in Figure \ref{fig:dec}, we should choose to label the second instance. Otherwise, we should ask the label for the third one. Since the true value of $\theta_2$ and $\theta_3$ are unknown, a  optimal policy does not exist under the frequentist paradigm. Further, it will be difficult to estimate $\theta_2$ and $\theta_3$ accurately when the budget is very limited. 

In contrast, in a Bayesian setting with prior distribution on each $\theta_i$, the optimal policy is defined as the policy which leads to the highest expected accuracy under the given prior instead of for any possible values of $\{\theta_i\}_{i=1}^n$. Therefore, we can optimally determine the next instance to label by taking another expectation over the distribution of $\theta_i$. In this paper, we adopt the Bayesian modeling to formulate the budget allocation problem in crowd labeling.

\section{Bayesian MDP and Optimal Policy}
\label{sec:MDP}
In this section, we introduce a Bayesian MDP framework and discuss its optimal policy.
\subsection{Bayesian Modeling}

We  assume that each $\theta_i$ is drawn from a known Beta prior $\B(a^0_i, b^0_i)$. Beta is a rich family of distributions in the sense that it exhibits a fairly wide variety of shapes on the domain of $\theta_i$, i.e., the unit interval $[0,1]$. 
For presentation simplicity, instead of considering a full Bayesian model with hyperpriors on $a^0_i$ and $b^0_i$, we fix $a_i^0$ and $b_i^0$ at the beginning.
In practice, if the budget is sufficient, one can first label each instance equally many times to pre-estimate $\{a_i^0, b_i^0\}_{i=1}^K$ before the dynamic labeling procedure is invoked. Otherwise, when there is no prior knowledge, we can simply assume $a_i^0=b_i^0=1$  so that the prior is a uniform distribution. According to our simulated experimental results in Section \ref{sec:exp_prior_instance}, uniform prior works reasonably well unless the data is highly skewed in terms of class distribution. Other commonly used uninformative priors such as Jeffreys prior or reference prior ($\B(1/2, 1/2)$) or Haldane prior ($\B(0,0)$) can also be adopted (see \cite{Bayesian:07} for more on uninformative priors). Choices of prior distributions are discussed in more details in Section \ref{sec:discussion}.

At each stage $t$ with $\B(a_i^t, b_i^t)$ as the current posterior distribution for $\theta_i$, we make a decision by choosing an instance $i_t \in \mathcal{A}=\{1,\ldots, K\}$ and acquire its label $y_{i_t} \sim \text{Bernoulli} (\theta_{i_t})$. Here $\mathcal{A}$ denotes the \emph{action set}. By the fact that Beta is the conjugate prior of the Bernoulli, the posterior of $\theta_{i_t}$  in the stage $t+1$ will be updated as:
 \begin{eqnarray*}
 \B(a_{i_t}^{t+1}, b_{i_t}^{t+1}) = \begin{cases}
   \B(a_{i_t}^t+1, b_{i_t}^t) \quad &  \text{if}   \quad y_{i_t}=1;\\
   \B(a_{i_t}^t, b_{i_t}^t+1) \quad  & \text{if}   \quad y_{i_t}=-1.
 \end{cases}
\end{eqnarray*}
We put $\{a_i^t, b_i^t\}_{i=1}^K$ into a $K \times 2$ matrix $S^t$, called a \emph{state matrix}, and let $S_i^t=(a_i^t, b_i^t)$ be the $i$-th row of $S^t$. The update of the state matrix can be written in a more compact form:
\begin{equation}
  S^{t+1}=\begin{cases}
    S^t+(\mathbf{e}_{i_t}, \mathbf{0}) & \text{if} \; y_{i_t}=1; \\
    S^t+( \mathbf{0}, \mathbf{e}_{i_t}) &  \text{if} \;  y_{i_t}=-1,
  \end{cases}
\end{equation}
where $\mathbf{e}_{i_t}$ is a $K \times 1$ vector with $1$ at the $i_t$-th entry and 0 at all other entries.
As we can see, $\{S^t\}$ is a Markovian process because $S^{t+1}$ is completely determined by the current state $S^t$, the action $i_t$ and the obtained label $y_{i_t}$. It is easy to calculate the \emph{state transition probability} $\Pr(y_{i_t} | S^t, i_{t})$, which is the posterior probability that we are in the next state $S^{t+1}$ if we choose $i_t$ to be label in the current state $S^t$:
\begin{eqnarray}
  \Pr(y_{i_t}=1 | S^t, i_t) = \E(\theta_{i_t} | S^t) = \frac{a^t_{i_t}}{a^t_{i_t}+b^t_{i_t}} \quad \text{and} \quad \Pr(y_{i_t}=-1| S^t, i_t)= \frac{b^t_{i_t}}{a^t_{i_t}+b^t_{i_t}}
  \label{eq:tran_prob}
\end{eqnarray}

Given this labeling process,  the budget allocation policy is defined as a sequence of decisions: $\pi=(i_0, \ldots, i_{T-1})$. Here, we require decisions depend only upon the previous information.  To make this more formal, we define a filtration $\{\calF_t\}_{t=0}^T$, where $\calF_t$ is the information collected until the stage $t-1$. More precisely, $\calF_t$ is the
the $\sigma$-algebra generated by the sample path $(i_0, y_{i_0}, \ldots, i_{t-1}, y_{i_{t-1}})$.  We require the action $i_t$ is determined based on the historical labeling results up to the stage $t-1$, i.e., $i_t$ is $\calF_t$-measurable.



\subsection{Inference about the True Labels}

As described in Section \ref{sec:binary}, the budget allocation process has two phases: the dynamic budget allocation phase and  the label aggregation   phase.  
Since the goal of the dynamic budget allocation in the first phase is to maximize the accuracy of aggregated labels in the second phase, we first present how to infer the true label via label aggregation in the second phase.

When the decision process terminates at the stage $T$,  we need to determine a positive set $H_T$ to maximize the \emph{conditional} expected accuracy conditioning on $\calF_T$, which corresponds to minimizing the posterior risk:

\begin{equation}
H_T = \argmax_{H \subset \{1,\ldots, K\}} \E\left( \sum_{i=1}^K \bigl(\mathbf{1}(i \in H) \cdot \mathbf{1}(i \in H^*) +  \mathbf{1}(i \not \in H) \cdot \mathbf{1}(i \not \in H^*) \bigr) \Bigg| \calF_T \right),
\label{eq:H_T}
\end{equation}
where $\mathbf{1}(\cdot)$ is the indicator function\footnote{For example, $\mathbf{1}(i \in H^*)=1$ if $i\in H^*$ and 0 if $i \not \in H^*$.}.  The term inside expectation in \eqref{eq:H_T} is the binary labeling accuracy which can also be written as $|H \cap H^*| + |H^c \cap (H^*)^c|$.

We first observe that, for $0\leq t\leq T$, the conditional distribution $\theta_i|\calF_t$ is exactly the posterior distribution $\B(a_i^t,b_i^t)$, which depends on the historical sampling results only through   $S_i^t=(a_i^t, b_i^t)$. Hence, we  define 
\begin{align}
& I(a,b) \doteq \Pr(\theta \geq 0.5 | \theta \sim \mathrm{Beta}(a,b)), \label{Iab} \\
& P^t_i \doteq \Pr(i \in H^* | \calF_t) =  \Pr(\theta_i \geq 0.5 | \calF_t) = \Pr(\theta_i \geq 0.5 | S_i^t)=I(a_i^t,b_i^t), \label{Pti}
\end{align}
As shown in  \cite{Xie:12}, the optimal positive set $H_T$ can be determined by the Bayes decision rule as follows.
\begin{proposition}
$H_T = \{i: \Pr(i \in H^* | \calF_T) \geq 0.5\}= \{i: P_i^T \geq 0.5\}$ solves \eqref{eq:H_T}.
\label{prop:H}
\end{proposition}
The proof of Proposition \ref{prop:H} is given in the appendix for completeness.

With Proposition \ref{prop:H} in place, we plug the optimal positive set $H_T$ into the right hand side of \eqref{eq:H_T} and the conditional expected accuracy given $\calF_T$ can be simplified as:
\begin{equation}
  \E\left( \sum_{i=1}^K \bigl(\mathbf{1}(i \in H_T) \cdot \mathbf{1}(i \in H^*) +  \mathbf{1}(i \not \in H_T) \cdot \mathbf{1}(i \not \in H^*) \bigr) \Bigg| \calF_T \right) =  \sum_{i=1}^K h(P_i^T),
\end{equation}
where $h(x) \doteq \max(x,1-x)$. We also note that $P_i^T$ provides not only the estimated label for the $i$-th instance but also  how confident the estimated label is correct. According to the next corollary with the proof in the appendix, we show that the optimal $H_T$ is constructed based on a refined \emph{majority vote} rule which incorporates the prior information.

\begin{corollary}
$I(a,b)>0.5$ if and only if $a>b$ and $I(a,b)=0.5$ if and only if $a=b$.
Therefore, $H_T =\{i: a^T_i \geq b_i^T\}$ solves \eqref{eq:H_T}.
\label{cor:majority_vote}
\end{corollary}

By viewing $a_i^0$ and $b_i^0$ as  pseudo-counts of 1s and $-1$s at the initial stage, the parameters $a_i^T$ and $b_i^T$ are the total counts of 1s and $-1$s. The estimated positive set $H_T =\{i: a^T_i \geq b_i^T\}$ consists of instances with more (or equal) counts of 1s than that of $-1$s. When  $a_i^0=b_i^0$, $H_T$ is constructed exactly according to the vanilla \emph{majority vote} rule.

To find the optimal allocation policy which maximizes the expected accuracy, we need to solve the following optimization problem:
\begin{align}
 V(S^0) \doteq &   \sup_{\pi}\E^{\pi} \left[  \E\left( \sum_{i=1}^K \bigl(\mathbf{1}(i \in H_T) \cdot \mathbf{1}(i \in H^*) +  \mathbf{1}(i \not \in H_T) \cdot \mathbf{1}(i \not \in H^*) \bigr) \Bigg| \calF_T \right)\right ] \nonumber  \\
 = & \sup_{\pi }\E^{\pi} \left(\sum_{i=1}^K h(P_i^T) \right),
\label{eq:value_func}
\end{align}
where $\E^{\pi}$ represents the expectation taken over the sample paths $(i_0, y_{i_0},\dots, i_{T-1}, y_{i_{T-1}})$ generated by a policy $\pi$. The second equality  is due to Proposition \ref{prop:H} and $V(S^0)$ is called value function at the initial state $S^0$. 
The optimal policy $\pi^*$ is any policy $\pi$ that attains the supremum in \eqref{eq:value_func}.

\subsection{Markov Decision Process}
\label{sec:sub_MDP}
The optimization problem in \eqref{eq:value_func} is essentially a Bayesian multi-armed bandit (MAB) problem, where each instance corresponds to an arm and the decision is which instance/arm to be sampled next. However, it is  different from the classical MAB problem \cite{UCB:02, Bubeck:Survey:12}, which assumes that each sample of an arm yields independent and identically distributed (\emph{i.i.d.}) reward according to some unknown distribution associated with that arm. Given the total budget $T$, the goal is to determine a sequential allocation policy so that the collected rewards can be maximized. We contrast this problem with our problem: instead of collecting intermediate independent rewards on the fly, our objective in \eqref{eq:value_func} merely involves the final ``reward'', i.e., overall labeling accuracy, which is only available at the final stage when the budget runs out.  Although there is no intermediate reward in our problem, we can still decompose the final expected accuracy into sum of \emph{stage-wise rewards} using the technique from \cite{Xie:12}, which further leads to our MDP formulation. Since these \emph{stage-wise rewards} are artificially created, they are no longer i.i.d. for each instance. We also note that the problem in \cite{Xie:12} is an infinite-horizon one which optimizes the stopping time while our problem is \emph{finite-horizon} since the decision process must be stopped at the stage $T$. 


\begin{proposition}
  Define the stage-wise expected reward as:
  \begin{equation}
     R(S^t, i_t)=\E \left(\sum_{i=1}^K h(P_{i}^{t+1}) - \sum_{i=1}^K h(P_{i}^{t}) \big|S^{t}, i_t \right)= \E \left( h(P_{i_t}^{t+1}) -   h(P_{i_t}^{t}) |S^{t}, i_t \right),
     \label{eq:reward}
  \end{equation}
  then the value function \eqref{eq:value_func} becomes:
    \vspace{-1mm}
    \begin{equation}
          V(S^0) = G_0(S^0) + \sup_{\pi} \E^{\pi} \left( \sum_{t=0}^{T-1}  R(S^t, i_t )  \right),
          \label{eq:decomp_value}
    \end{equation}
    where $G_0(S^0)=\sum_{i=1}^K h(P_i^0) $ and the optimal policy $\pi^*$ is any policy $\pi$ that attains the supremum.
    \label{prop:reward}
\end{proposition}

The  proof of Proposition \ref{prop:reward} is presented in the appendix. In fact, the stage-wise reward in \eqref{eq:reward} has a straightforward interpretation. According to \eqref{eq:value_func}, the term $\sum_{i=1}^K h(P_{i}^{t})$ is the expected accuracy at the $t$-th stage. The stage-wise reward $R(S^t, i_t)$ takes the form of the difference between the expected accuracy at the $(t+1)$-stage and  the $t$-th stage, i.e., the expected gain in accuracy for collecting another label for the $i_t$-th instance.  The second equality in \eqref{eq:reward} holds simply because: only the $i_t$-th instance receives the new label and the corresponding $P_{i_t}^t$ changes while all other $P_{i}^t$ remain  the same. Since the expected reward \eqref{eq:reward} only depends on $S_{i_t}^t=(a_{i_t}^t, b_{i_t}^t)$,  
we write
\begin{equation}
R(S^t, i_t) =R\left(S^t_{i_t}\right) = R\left( a^t_{i_t}, b^t_{i_t}\right),
\end{equation}
 and use them interchangeably. The function $R(a,b)$ with two parameters $a$ and $b$ has an analytical representation as follows. For any state  $(a,b)$ of a single instance,  the reward of getting a label 1 and a label $-1$ are:
\begin{align}
  R_1(a,b) & =h(I(a+1,b))-h(I(a,b)),  \label{eq:R_1} \\
  R_2(a,b) &=h(I(a,b+1))-h(I(a,b)).   \label{eq:R_2}
\end{align}
The expected reward takes the following form:
\begin{equation}
R(a,b)= p_1 R_1 + p_2 R_2,
\label{eq:exp_reward_R}
\end{equation}
where $p_1 = \frac{a}{a+b}$ and $p_2 = \frac{b}{a+b}$ are the  transition probabilities in \eqref{eq:tran_prob}.

With Proposition \ref{prop:reward}, the maximization problem \eqref{eq:value_func} is formulated as a $T$-stage \emph{Markov Decision Process} (MDP) as in \eqref{eq:decomp_value}, which is associated with a tuple: $$\{T,  \{\mathcal{S}^t \}, \mathcal{A}, \Pr(y_{i_t} | S^t, i_t), R(S^t, i_t)\}.$$ Here, the state space at the stage $t$, $\calS^t$,  is all possible states that can be reached at $t$. Once we collect a label $y_{i_t}$, one element in $S^t$ (either $a_{i_t}^t$ or $b_{i_t}^t$) will add one. Therefore, we have
\begin{equation}
\mathcal{S}^t = \left \{ \{a_i^t, b_i^t\}_{i=1}^K :  a_i^t \geq a_i^0, b_i^t \geq b_i^0,  \sum_{i=1}^{K} (a_i^t-a_i^0)+(b_i^t-b_i^0) = t \right\}.
\label{eq:state_space}
\end{equation}
The action space is the set of instances that could be labeled next: $\mathcal{A}=\{1, \ldots, K\}$. The transition probability $\Pr(y_{i_t} | S^t, i_t)$ is defined in \eqref{eq:tran_prob} and the expected reward at each stage  $R(S^t, i_t)$ is defined in \eqref{eq:reward}.

\begin{remark}
We can also view Proposition \ref{prop:reward} as a consequence of applying the reward shaping technique \cite{Ng99} to the original problem \eqref{eq:value_func}. In fact, we can add an artificial absorbing state, named $S_{obs}$, to the original state space \eqref{eq:state_space} and assume that, when the budget allocation process finishes, the state must transit one more time to reach $S_{obs}$ regardless of which action is taken. Hence, the original problem \eqref{eq:value_func} becomes a MDP that generates a zero transition reward until the state enters $S_{obs}$ where the transition reward is $\sum_{i=1}^K h(P_i^T)$. Then, we define a potential-based shaping function \cite{Ng99} over this extended state space as $\Phi(S^t)=\sum_{i=1}^K h(P_i^t)$ for $S^t\in\mathcal{S}^t$ and $\Phi(S_{obs})=0$. After this, \eqref{prop:reward} can be viewed as a new MDP whose transition reward equals that of \eqref{eq:value_func} plus the shaping-reward function $\Phi(S')-\Phi(S)$ when the state transits from $S$ to $S'$. According to Theorem 1 in \cite{Ng99}, \eqref{prop:reward} and \eqref{eq:value_func} have the same optimal policy. This provides an alternative justification for Proposition \ref{prop:reward}.
\end{remark}






\subsection{Optimal Policy via DP}
\label{sec:DP}

With the MDP in place, we can apply the dynamic programming (DP) algorithm  (a.k.a. backward induction) \cite{Puterman:05} to compute the optimal policy:
\begin{enumerate}
      \item Set $V_{T-1}(S^{T-1}) = \max_{i \in \{1,\ldots, K\}} R(S^{T-1},i)$ for \emph{all possible states} $S^{T-1} \in \mathcal{S}^{T-1}$. The optimal decision $i_{T-1}^*(S^{T-1})$ is the decision $i$ that achieves the maximum when the state is $S^{T-1}$.
      \item Iterate for $t=T-2, \ldots ,0$, compute the $V_t(S^t)$ for all possible  $S^t \in \calS^t$ using the Bellman equation:
      \small
     \begin{equation*}
        V_{t}(S^t) =  \max_{i}  \Bigl( R(S^{t}, i)  +  \Pr(y_{i}=1 | S^t, i) V_{t+1} \left(S^t+(\mathbf{e}_{i}, \mathbf{0})\right)      + \Pr(y_{i}= -1 | S^t, i) V_{t+1} \left(S^t+(\mathbf{0}, \mathbf{e}_{i})\right) \Bigr),         
     \end{equation*}
     \normalsize
      and $i_t^*(S^t)$ is the $i$ that achieves the maximum.  
\end{enumerate}

The optimal policy $\pi^*=(i_0^*, \ldots, i_T^*)$.  For an illustration purpose, we use DP to calculate the optimal instance to be labeled next in the toy example in Section \ref{sec:example} under the uniform prior $B(1,1)$ for all $\theta_i$. Since we assume that there is only one labeling chance remaining, which corresponds to the last stage of DP, we should choose the instance $i_{T-1}^*(S^{T-1})= \argmax_{i \in \{1,\ldots, K\}} R(S^{T-1},i)$. According to the calculation in Table \ref{tab:toy}, there is a unique optimal instance for labeling, which is the second instance.
\renewcommand{\arraystretch}{1.2}
\begin{table}[!t]
\renewcommand{\tabcolsep}{6pt}
\centering
\caption{Calculation of the expected reward for the toy example in Table \ref{tab:example} according to \eqref{eq:R_1}, \eqref{eq:R_2} and \eqref{eq:exp_reward_R}. }
\begin{tabular}{|c|c|c|c|c|c|c|} \hline
Instance $i$    & \multicolumn{1}{|c|}{$S^{T-1}_i$} & \multicolumn{1}{|c|}{$p_1$}       & \multicolumn{1}{|c|}{$p_2$}      & \multicolumn{1}{|c|}{$R_1(S^{T-1}_i)$} &  \multicolumn{1}{|c|}{$R_2(S^{T-1}_i)$}  & \multicolumn{1}{|c|}{$R(S^{T-1},i) =R(S^{T-1}_i)$} \\ \hline
1               & (3,1)       & $3/4$         & $1/4$        & $1/16$           &  $-3/16$           &  $3/4\times 1/16 +  1/4  \times(-3/16)=0$      \\ \hline
2               & (2,2)       & $1/2$         & $1/2$        &  $3/16$          &  $3/16$            &  $1/2 \times 3/16 +  1/2 \times  3/16 = 3/16$      \\ \hline
3               & (2,1)       & $2/3$         & $1/3$        &  $1/8$          &   $-1/4$            &  $2/3   \times 1/8  + 1/3 \times (-1/4)=0$ \\ \hline
\end{tabular}
\label{tab:toy}
\end{table}

Although DP finds the optimal policy, its computation is intractable since the size of the state space $|\calS^t|$ grows exponentially in $t$ according to \eqref{eq:state_space}. Therefore, we need to develop a computationally efficient approximate policy, which is the goal of the next section.

\section{Approximate Policies}
\label{sec:Opt-KG}


Since DP is computationally intractable, approximate policies are needed for large-scale applications. The simplest policy is the uniform sampling (a.k.a, pure exploration), i.e., we choose the next instance uniformly and independently at random: $i_{t}  \sim \text{Uniform}(1,\ldots, K)$.  However, this policy does not explore any structure of the problem.

With the decomposed reward function, our problem is essentially a finite-horizon Bayesian MAB problem. Gittins \cite{Gittins:89} showed that Gittins index policy is optimal for infinite-horizon MAB with the discounted reward. It has been applied to the infinite-horizon version of problem \eqref{eq:decomp_value} in \cite{Xie:12}. Since our problem is finite-horizon, Gittins index is no longer optimal while it can still provide us a good heuristic index rule. However, the computational cost of Gittins index is very high: the state-of-art-method proposed by \cite{Mora:11} requires $O(T^6)$ time and space complexity.

A computationally more attractive policy is the knowledge gradient (KG) \cite{Gupta:96,Frazier:08}.  It is essentially a single-step look-ahead policy, which greedily selects the next instance with the largest expected reward:
\begin{equation}
   i_t= \argmax_{i \in \{1,\ldots, K\} } \left(R(a^t_{i}, b^t_{i}) \doteq \frac{a_{i}^t}{a_{i}^t+b_{i}^t} R_1(a_{i}^t,b_{i}^t) +  \frac{b_{i}^t}{a_{i}^t+b_{i}^t}  R_2(a_{i}^t,b_{i}^t)\right).
\end{equation}
As we can see, this policy corresponds to the last stage in DP and hence KG policy is optimal if only one labeling chance is remaining.

When there is a tie, if we select the smallest index $i$, the policy is referred to \emph{deterministic KG} while if we randomly break the tie,  the policy is referred to \emph{randomized KG}.  Although KG has been successfully applied to many MDP problems \cite{Powell:07}, it will fail in our problem as shown in the next proposition with the proof in the appendix.

\begin{proposition}
Assuming that $a_i^0$ and $b_i^0$ are positive integers and letting $\calE=\{i: a_i^0=b_i^0\}$, then the deterministic KG policy will acquire one label for each instance in $\calE$ and then consistently obtain the label for the first instance even if the budget $T$ goes to infinity.
\label{prop:det_KG}
\end{proposition}

According to Proposition \ref{prop:det_KG}, the deterministic KG is \emph{not} a  \emph{consistent policy}, where the consistent policy refers to the policy that will provide correct labels for all instances (i.e., $H_T=H^*$) almost surely when   $T$ goes to infinity.  We note that randomized KG policy can address this problem. However, from the proof of Proposition \ref{prop:det_KG}, randomized KG behaves similarly to the uniform sampling policy in many cases and its empirical performance is undesirable according to   Section \ref{sec:exp}. In the next subsection, we will propose a new approximate allocation policy based on KG which is a consistent policy with superior empirical performance.


\subsection{Optimistic Knowledge Gradient}

\begin{algorithm}[!t]
    \caption{Optimistic Knowledge Gradient} 
    \label{algo:opt_KG}
    \begin{algorithmic}
    \STATE {\bfseries Input:} Parameters of prior distributions for instances $\{a_i^0, b_i^0\}_{i=1}^K$ and the   budget $T$.
    \medskip
    \FOR{$t = 0, \ldots, T-1$}

        \STATE Select the next instance $i_t$ to label according to:
        \begin{equation}
            \qquad i_t= \argmax_{i \in \{1,\ldots, K\}} \left( R^{+}(a^t_i, b^t_i) \doteq \max(R_1(a^t_i, b^t_i), R_2(a^t_i, b^t_i)) \right).
            \label{eq:opt_KG}
        \end{equation}

        \STATE  Acquire the label $y_{i_t}\in \{-1,1\}$.

        \IF{$y_{i_t}=1$}
            \STATE $a^{t+1}_{i_t}=a^{t}_{i_t}+1, b^{t+1}_{i_t}= b^{t}_{i_t}$;  $a^{t+1}_{i}=a^{t}_{i}, b^{t+1}_{i}= b^{t}_{i}$ for all $i \neq i_t$.
        \ELSE
            \STATE $a^{t+1}_{i_t}=a^{t}_{i_t}, b^{t+1}_{i_t}= b^{t}_{i_t}+1$;  $a^{t+1}_{i}=a^{t}_{i}, b^{t+1}_{i}= b^{t}_{i}$ for all $i \neq i_t$.
        \ENDIF
    \ENDFOR
    \medskip
    \STATE {\bfseries Output:} The positive set $H_T=\{i: a^T_i \geq b^T_i\}$.
    \end{algorithmic}
\end{algorithm}

The stage-wise reward can be viewed as a random variable with a two point distribution, i.e., with the probability $p_1=\frac{a}{a+b}$ of being $R_1(a,b)$ and the probability $p_2=\frac{b}{a+b}$ of being $R_2(a,b)$. The KG policy selects the instance with the largest \emph{expected} reward. However, it is not consistent.

In  this section, we introduce a new index policy called ``optimistic knowledge gradient" (Opt-KG) policy.   The Opt-KG policy assumes that decision makers are optimistic in the sense that they select the next instance based on the optimistic outcome of the reward. As a simplest version of the Opt-KG policy, for any state $(a_i^t, b_i^t)$, the optimistic outcome of the reward $R^+(a_i^t, b_i^t)$ is defined as maximum over the reward of obtaining the label 1, $R_1(a^t_i, b^t_i)$, and the reward of obtaining the label $-1$,  $R_2(a^t_i, b^t_i)$.  Then the optimistic decision maker selects the next instance $i$ with the largest $R^+(a_i^t, b_i^t)$ as in \eqref{eq:opt_KG} in Algorithm \ref{algo:opt_KG}.  The overall decision process using the Opt-KG policy is highlighted in Algorithm \ref{algo:opt_KG}.

In the next theorem, we prove that Opt-KG policy is consistent.
\begin{theorem}
Assuming that $a_i^0$ and $b_i^0$ are positive integers, the Opt-KG is a consistent policy, i.e, as $T$ goes to infinity, the accuracy will be $100\%$ (i.e., $H_T=H^*$) almost surely.
\label{thm:opt_KG}
\end{theorem}
The key of proving the consistency is to show that when $T$ goes to infinity, each instance will be labeled infinitely many times. We prove this fact by showing that for any pair of positive integers $(a,b)$, $R^+(a,b)=\max(R_1(a,b), R_2(a,b))>0$ and  $R^{+}(a,b) \rightarrow 0$ when $a+b \rightarrow \infty$.  As an illustration, the values of $R^{+}(a,b)$ are plotted in Figure \ref{fig:cvar_right}. Then, by strong law of large number, we obtain the consistency of the Opt-KG as stated in Theorem \ref{thm:opt_KG}. The details are presented in the appendix. We have to note that asymptotic consistency is the minimum guarantee for a good policy. However, it does not necessarily guarantee the good empirical performance for the finite budget level. We will use experimental results to show the superior performance of the proposed policy.

\begin{figure}[!t]
\centering
  \includegraphics[width=0.5\textwidth]{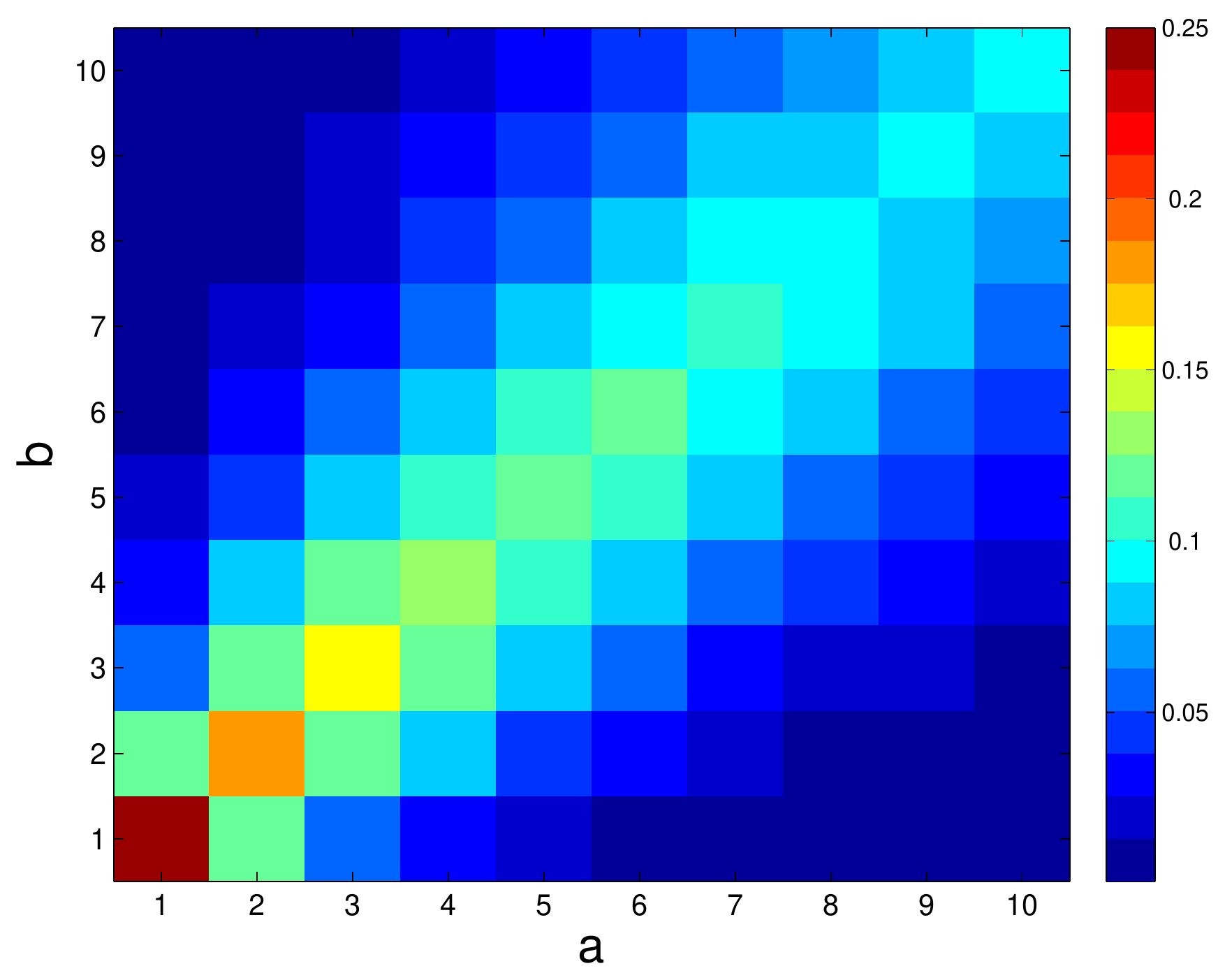}
  \caption{Illustration of $R^{+}(a,b)$.}
  \label{fig:cvar_right}
\end{figure}


The proposed Opt-KG policy is a general framework for budget allocation in crowd labeling. We can extend the allocation policy based on the maximum over the two possible rewards (Algorithm \ref{algo:opt_KG}) to a more general policy using the conditional value-at-risk (CVaR) \cite{Rockafellar:02}.  We note that here, instead of adopting the CVaR as a risk measure, we apply it to the reward distribution.  In particular, for a random variable $X$ with the support $\mathcal{X}$ (e.g., the random reward with the two point distribution),  let $\alpha$-quantile function be denoted as $Q_\alpha(X) = \inf\{x \in \mathcal{X}: \alpha \leq F_X(x)\}$, where $F_X(\cdot)$ is the CDF of $X$. The value-at-risk $\var_{\alpha}(X)$ is the smallest value such that the probability that $X$ is less than (or equal to) it is greater than (or equal to) $1-\alpha$: $\var_{\alpha}(X)=Q_{1-\alpha}(X)$. The  conditional value-at-risk ($\cvar_{\alpha}(X)$)  is defined as the expected reward exceeding (or equal to) $\var_{\alpha}(X)$. An illustration of CVaR is shown in Figure \ref{fig:cvar}.

\begin{figure}[!t]
\centering
  \includegraphics[width=0.6\textwidth, height=4cm]{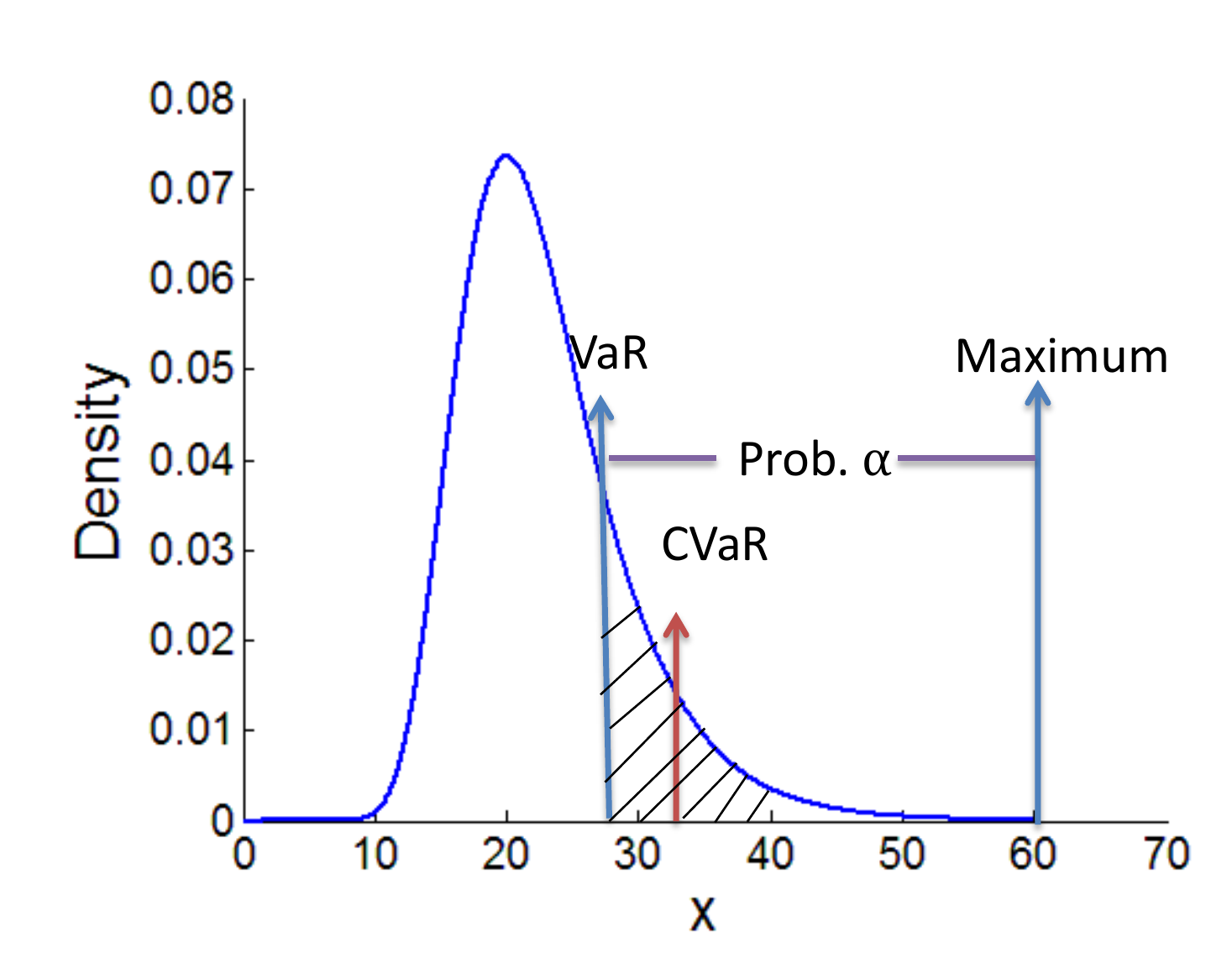}
  \caption{Illustration of Conditional Value-at-Risk.}
  \label{fig:cvar}
\end{figure}

For our problem, according to \cite{Rockafellar:02}, $\cvar_{\alpha}(X)$ can be expressed as a simple linear program:
\begin{align*}
    \text{CVaR}_{\alpha}(X) = &  \max_{\{q_1 \geq 0, q_2 \geq 0\}} q_1R_1+q_2R_2, \\
      & \text{s.t.} \quad q_1 \leq \frac{1}{\alpha} p_1, \; q_2 \leq \frac{1}{\alpha} p_2, \; q_1+q_2=1.
\end{align*}
As we can see, when $\alpha=1$, $\text{CVaR}_{\alpha}(X)= p_1R_1+p_2R_2$, which is the expected reward; when $\alpha \rightarrow 0$, $\text{CVaR}_{\alpha}(X) =\max(R_1,R_2)$, which is used as the selection criterion in \eqref{eq:opt_KG} in Algorithm \ref{algo:opt_KG}.  In fact, a more general Opt-KG policy could be selecting the next instance with the largest  $\text{CVaR}_{\alpha}(X)$ with a tuning parameter $\alpha \in [0,1]$.  We can extend Theorem \ref{thm:opt_KG} to prove that the policy based on $\text{CVaR}_{\alpha}(X)$ is consistent for any $\alpha<1$.  According to our own experience, $\alpha \rightarrow 0$ usually has a better performance in our problem especially when the budget is very limited. Therefore, for the sake of presentation simplicity, we introduce the Opt-KG using $\max(R_1,R_2)$ (i.e., $\alpha \rightarrow 0$ in $\text{CVaR}_{\alpha}(X)$) as the selection criterion.

Finally, we highlight that the Opt-KG policy is computationally very efficient. For $K$ instances with $T$ units of the budget, the overall time and space complexity are $O(KT)$ and $O(K)$ respectively. It is much more efficient that the Gittins index policy which requires $O(T^6)$ time and space complexity.

\subsection{Discussions}
\label{sec:discussion}

It is interesting to see the connection between the idea of making the decision based on the optimistic outcome of the reward and the UCB (upper confidence bounds) policy \cite{UCB:02} for the classical multi-armed bandit problem as described in  Section \ref{sec:sub_MDP}. 
In particular, the UCB policy selects the next arm with the maximum \emph{upper confidence index}, which is defined as the current average reward plus the one-sided confidence interval. As we can see, the upper confidence index can be viewed as an ``optimistic'' estimate of the reward. However, we note that since we are in a Bayesian setting and our  stage-wise rewards are artificially created and thus not \emph{i.i.d.} for each arm, the UCB policy \cite{UCB:02} cannot be directly applied to our problem.

In fact, our Opt-KG follows a more general principle of ``optimism in the face uncertainty'' \cite{Szita:2008}. Essentially, the non-consistency of KG is due to its nature of pure exploitation while a consistent policy should typically utilizes exploration. One of the common techniques to handle the exploration-exploitation dilemma is to take an action based on an optimistic estimation of the rewards (see \cite{Szita:2008} and \cite{EvenDar:2001} ), which is the role $R^+(a,b)$ plays in Opt-KG.

For our problem, it is also straightforward to design the ``pessimistic knowledge gradient'' policy which selects the next instance $i_t$ based on the pessimistic outcome of the reward, i.e., $$i_t= \argmax_{i} \left( R^{-}(a^t_i, b^t_i) \doteq \min(R_1(a^t_i, b^t_i), R_2(a^t_i, b^t_i)) \right).$$
 However, as shown in the next proposition with the proof in the appendix, the pessimistic KG policy is inconsistent under the uniform prior.
\begin{proposition}
When starting from the uniform prior (i.e., $a_i^0=b_i^0=1$) for all $\theta_i$,   the pessimistic KG policy will acquire one label for each instance  and then consistently acquire the label for the first instance even if the budget $T$ goes to infinity.
\label{prop:pessimistic_KG}
\end{proposition}

Finally, we discuss some other possible choices of prior distributions. For presentation simplicity, we only consider the Beta prior for each $\theta_i$ with the fixed parameters $a_i^0$ and $b_i^0$. In practice, more complicated priors can be easily incorporated into our framework. For example, in instead of using only one Beta prior, one can adopt a mixture of Beta distributions as the prior and the posterior will also follow a mixture of Beta distributions, which allows an easy inference about the posterior. As we show in the experiments (see Section \ref{sec:exp_prior_instance}), the uniform prior does not work well when the data is highly skewed in terms of class distribution. To address this problem, one possible choice is to adopt the prior $p(\theta)=w_1 \B(c,1) + w_2 \B(1,1)+ w_3\B(1,c)$ where $w_1, w_2$ and $w_3$ are the weights and $c$ is a constant larger than 1 (e.g., $c=5$). In such a prior, $B(c,1)$ corresponds to the data with more positive labels while $B(1,c)$ to the data with more negative labels. In addition to the mixture Beta prior, one can adopt the hierarchical Bayesian approach which puts hyper-priors on the parameters in the Beta priors. The inference can be performed using empirical Bayes approach \cite{Gelman:13,Bayesian:07}. In particular, one can periodically re-calculate the MAP estimate of the hyper-parameters based on the available data and update the model, but otherwise proceed with the given hyper-parameters.  For common choices of hyper-priors of Beta, please refer to Section 5.3 in \cite{Gelman:13}. These approaches can also be applied to model the workers' reliability as we introduced in the next Section. For example, one can use a mixture of Beta distributions as the prior for the workers' reliability, where $\B(c,1)$ corresponds to reliable workers, $\B(1,1)$ to random workers and $\B(1,c)$ to malicious or poorly informed workers.



\section{Incorporate Reliability of Heterogeneous Workers}
\label{sec:worker}

In push crowdsourcing marketplaces, it is important to model workers' reliability so that the decision maker could assign more instances to reliable workers.  Assuming that there are $M$ workers in a push marketplace, we can capture the reliability of the $j$-th worker by introducing an extra parameter $\rho_j \in  [0,1]$ as in \cite{Dawid:79, Vikas:10, Oh:12}, which is defined as the probability of getting the same label as the one from a random fully reliable worker. Recall that the soft-label $\theta_i$  is the $i$-th instance's probability of being labeled as positive by a fully reliable worker and let $z_{ij}$ be the label provided by the $j$-th worker for the $i$-th instance. We model the distribution of  $z_{ij}$ for given $\theta_i$ and $\rho_j$ using the  \emph{one-coin} model \cite{Dawid:79, Oh:12}\footnote{We can further extend it to a more complex \emph{two-coin} model \cite{Dawid:79, Vikas:10} by introducing a pair of parameters $(\rho_{j1}, \rho_{j2})$ to model the $j$-th worker's reliability. In particular, $\rho_{j1}$ and $\rho_{j2}$ are the probabilities of getting the positive and negative labels when a fully reliable worker  provides the same label.}:
\small
\begin{align}
 \Pr(z_{ij}=1| \theta_i, \rho_j)  & =   \Pr(z_{ij}=1| y_i=1, \rho_j) \Pr(y_i=1 |\theta_i) +\Pr(z_{ij}=1| y_i=-1, \rho_j) \Pr(y_i=-1 |\theta_i) \nonumber \\
                                  & =  \rho_j  \theta_i+  (1-\rho_j)(1-\theta_i);  \label{eq:Z_1}\\
 \Pr(z_{ij}=-1| \theta_i, \rho_j) & =   \Pr(z_{ij}=-1| y_i=-1, \rho_j) \Pr(y_i=-1 |\theta_i) +\Pr(z_{ij}=-1 | y_i=1, \rho_j) \Pr(y_i=1 |\theta_i) \nonumber \\
                                  & =   \rho_j (1-\theta_i) + (1-\rho_j)\theta_i, \label{eq:Z_0}
\end{align}
\normalsize
where $y_i$ denotes the label provided a random fully reliable worker for the $i$-th instance. Here we make the following implicit assumption:
\begin{assumption}
  We assume that different workers make independent judgements and, for each single worker, the labels provided by him/her to different instances are also independent.
\end{assumption}

As the parameter $\rho_j$ increases from $0$ to $1$, the $j$-th worker's reliability also increases in the sense that $\Pr(z_{ij}=1| \theta_i, \rho_j)$ gets more and more close to $\theta_i$, which is the probability of getting a positive label from a random fully reliable worker.  Different types of workers can be easily characterized by $\rho_j$. When all $\rho_j=1$,  it recovers the previous model with fully reliable workers since $ \Pr(z_{ij}=1| \theta_i, \rho_j) = \theta_i$, i.e, each worker provides the label only according to the underlying soft-label of the instance.  When $\rho_j=0.5$, we have $\Pr(z_{ij}=1 | \theta_i, \rho_j)=\Pr(z_{ij}=-1 | \theta_i, \rho_j)=0.5$, which indicates that the $j$-th worker is a spammer, who randomly submits positive or negative labels. When $\rho_j=0$, it indicates that the $j$-th worker is poorly informed or misunderstands the instruction such that he/she always assigns wrong labels.


We assume that instances' soft-label $\{\theta_i\}_{i=1}^K$ and workers' reliability $\{\rho_j\}_{j=1}^M$ are drawn from known Beta prior distributions: $\theta_i \sim \B(a_i^0, b_i^0)$ and $\rho_j\sim \B(c_j^0, d_j^0)$. At each stage, we need to make the decision on both the next instance $i$ to be labeled and the next worker $j$ to label the instance $i$ (we omit $t$ in $i,j$ here for notational simplicity). In other words, the action space  $\mathcal{A}=\{(i,j): (i,j) \in \{1,\ldots, K\}\times \{1, \ldots, M\}\}$.
Once the decision is made, the distribution of the outcome $z_{ij}$ is given by \eqref{eq:Z_1} and \eqref{eq:Z_0}. Given the prior distributions and likelihood functions in \eqref{eq:Z_1} and \eqref{eq:Z_0}, the Bayesian Markov Decision process can be formally defined as in Section \ref{sec:MDP}. Similar to the homogeneous worker setting,  the optimal inferred positive set $H_T$ takes the form of $H_T=\{i: P_i^T \geq 0.5\}$ as in Proposition \ref{prop:H} with $P_i^t= \Pr(i \in H^*| \calF_t)= \Pr\left(\theta_i \geq 0.5 | \calF_t \right)$. The value function $V(S^0)$  still takes the form of \eqref{eq:value_func}, which can be further decomposed into the sum of stage-wise rewards in \eqref{eq:reward} using Proposition \ref{prop:reward}. Unfortunately, in the heterogenous worker setting,
the posterior distributions of $\theta_i$ and $\rho_j$ are highly correlated with a sophisticated joint distribution, which makes the computation of stage-wise rewards in \eqref{eq:reward} much more challenging. In particular, given the prior $\theta_i \sim \B(a_i^0, b_i^0)$ and $\rho_j\sim \B(c_j^0, d_j^0)$, the posterior distribution of $\theta_i$ and $\rho_j$ given the label $z_{ij}=z \in \{-1,1\}$ takes the following form:
\begin{eqnarray}
    p(\theta_i, \rho_j |z_{ij}=z) = \frac{\Pr(z_{ij}=z| \theta_i, \rho_j) \B(a_i^0,b_i^0)\B(c_j^0,d_j^0)}{\Pr(z_{ij}=z)},
\end{eqnarray}
where $\Pr(z_{ij}=z| \theta_i, \rho_j)$ is the likelihood function defined in \eqref{eq:Z_1} and \eqref{eq:Z_0} and
\begin{align*}
  \Pr(z_{ij}=1)&=\E(\Pr(z_{ij}=1| \theta_i, \rho_j)) =\E(\theta_i)\E(\rho_j) + (1-\E(\theta_i)) (1-\E(\rho_j)) \\
               &=\frac{a_i^0}{a_i^0 + b_i^0} \frac{c_j^0}{c_j^0+d_j^0}+\frac{b_i^0}{a_i^0 + b_i^0} \frac{d_j^0}{c_j^0+d_j^0}.
\end{align*}
As we can see, the posterior distribution $p(\theta_i, \rho_j |z_{ij}=z)$ no longer takes the form of the product of the distributions of
$\theta_i$ and $\rho_j$ and the marginal posterior of $\theta_i$ is no longer a Beta distribution.  As a result, $P_i^t$ does not have a simple representation as in \eqref{Iab}, which makes the  computation of the reward function much more difficult as the number of stages increases. Therefore, to apply our Opt-KG policy to large-scale applications, we need to use some approximate posterior inference techniques.


\begin{algorithm}[!t]
\renewcommand{\thealgorithm}{2}
    \caption{Optimistic Knowledge Gradient for Heterogeneous Workers} 
    \label{algo:opt_KG_worker}
    \begin{algorithmic}
    \STATE {\bfseries Input:} Parameters of prior distributions for instances $\{a_i^0, b_i^0\}_{i=1}^K$ and for workers  $\{c_j^0, d_j^0\}_{j=1}^M$. The total budget $T$.
    \medskip
    \FOR{$t = 0, \ldots, T-1$}
     \smallskip
        \STATE \textbf{1.} Select the next instance $i_t$ to label and the next worker $j_t$ to label $i_t$ according to:
        \small
        \begin{equation}
             (i_t, j_t)= \argmax_{(i,j) \in \{1,\ldots, K\}\times \{1, \ldots, M\}}\left( R^{+}(a^t_i, b^t_i, c^t_j, d^t_j) \doteq \max(R_1(a^t_i, b^t_i, c^t_j, d^t_j), R_2(a^t_i, b^t_i, c^t_j, d^t_j)) \right).
        \label{eq:opt_KG_worker}
        \end{equation}
        \normalsize
        \STATE  \textbf{2.} Acquire the label $z_{i_t j_t}\in \{-1,1\}$ of the $i$-th instance from the $j$-th worker.

        \STATE \textbf{3.} Update the posterior by setting:
        \begin{align*}
          a^{t+1}_{i_t}  =\tilde{a}^t_{i_t}(z_{i_t j_t}) \qquad b^{t+1}_{i_t} =\tilde{b}^t_{i_t}(z_{i_t j_t}) \qquad
          c^{t+1}_{j_t}  =\tilde{c}^t_{j_t}(z_{i_t j_t}) \qquad  d^{t+1}_{j_t}  =\tilde{d}^t_{j_t}(z_{i_t j_t}) ,
        \end{align*}
         and all parameters for $i \neq i_t$ and $j \neq j_t$ remain the same.
    \ENDFOR
   \medskip
    \STATE {\bfseries Output:} The positive set $H_T=\{i: a^T_i \geq b^T_i\}$.
    \end{algorithmic}
\end{algorithm}

When applying Opt-KG, we need to perform $2\cdot K \cdot M \cdot T$ inferences of the posterior distribution in total. Each approximate inference should be computed very efficiently, hopefully in a closed-form. For large-scale problems, most traditional approximate inference techniques such as Markov Chain Monte Carlo (MCMC) or variational Bayesian methods (e.g., \cite{Beal:03, Paisley:12}) may lead to higher computational cost since each inference is an iterative procedure. To address the computational challenge, we apply the variational approximation with the moment matching technique so that each inference of the approximate posterior can be computed in a closed-form. In fact, any highly efficient approximate inference can be utilized to compute the reward function. Since the main focus of the paper is on the MDP model and Opt-KG policy, we omit the discussion for other possible approximate inference techniques. In particular, we first adopt the variational approximation by assuming the conditional independence of $\theta_i$ and $\rho_j$:
 \begin{equation*}
  p(\theta_i, \rho_j | z_{ij}=z ) \approx  p(\theta_i| z_{ij}=z )p(\rho_j| z_{ij}=z )
 \end{equation*}
 We further approximate $p(\theta_i| z_{ij}=z)$ and $p(\rho_j| z_{ij}=z)$ by two Beta distributions:
 \begin{equation*}
    p(\theta_i| z_{ij}=z ) \approx  \B(\tilde{a}_i(z), \tilde{b}_i(z));   \qquad  p(\rho_j| z_{ij}=z ) \approx \B(\tilde{c}_j(z), \tilde{d}_j(z)),
 \end{equation*}
 where the parameters $\tilde{a}_i(z)$, $\tilde{b}_i(z)$, $\tilde{c}_j(z)$, $\tilde{d}_j(z)$ are computed using moment matching  with the analytical form presented in the appendix. After this approximation, the new posterior distributions of $\theta_i$ and $\rho_j$ still have the same structure as their prior distribution, i.e., the product of two Beta distributions, which allows a repeatable use of this approximation every time when a new label is collected. Moreover, due to the Beta distribution approximation of $p(\theta_i| z_{ij}=z)$, the reward function takes a similar form as in the previous setting.
 In particular, assuming at a certain stage, $\theta_i$ has the posterior distribution  $\B(a_i,b_i)$ and $\rho_j$ has the posterior distribution $\B(c_j,d_j)$. The reward of getting positive and negative labels for the $i$-th instance from the $j$-th worker are presented in \eqref{eq:pos_label_worker} and \eqref{eq:neg_label_worker}:
\begin{align}
  R_1(a_i,b_i,c_j,d_j)&=h(I(\tilde{a}_i(z=1), \tilde{b}_i(z=1)))-h(I(a_i, b_i)), \label{eq:pos_label_worker} \\
  R_2(a_i,b_i,c_j,d_j)&=h(I(\tilde{a}_i(z=-1), \tilde{b}_i(z=-1)))-h(I(a_i, b_i)), \label{eq:neg_label_worker}
\end{align}
With the reward in place,  we present Opt-KG for budget allocation in the heterogeneous worker setting in Algorithm \ref{algo:opt_KG_worker}.  We also note that due to the variational approximation of the posterior, establishing the consistency results of Opt-KG becomes very challenging in the heterogeneous worker setting.

\section{Extensions}
\label{sec:extension}
Our MDP formulation is a general framework to address many complex settings of dynamic budget allocation problems in crowd labeling. In this section, we briefly  discuss two important extensions, where for both extensions, Opt-KG can be directly applied as an approximate policy. We note that for the sake of presentation simplicity, we only present these extensions in the noiseless homogeneous worker setting. Further extensions to the heterogeneous setting are rather straightforward using the technique from Section \ref{sec:worker}.

\subsection{Utilizing Contextual Information}
\label{sec:contextual}

When the contextual information is available for instances, we could easily extend our model to incorporate such an important information. In particular, let the contextual information for the $i$-th instance be represented by a $p$-dimensional feature vector $\bx_i \in \mathbb{R}^p$.  We could utilize the feature information by assuming a logistic model for $\theta_i$:
\begin{eqnarray*}
\theta_i \doteq \frac{\exp\{\langle \bw,  \bx_i \rangle \}}{1+\exp\{\langle \bw,  \bx_i \rangle \}},
\end{eqnarray*}
where $\bw$ is assumed to be drawn from a Gaussian prior  $N(\bmu_0, \bSigma_0)$. At the $t$-th stage with the current state $(\bmu_t, \bSigma_t)$, the decision maker  determines the instance $i_t$ and acquire its label $y_{i_t} \in \{-1, 1\}$. Then we update the posterior $\bmu_{t+1}$ and $\bSigma_{t+1}$ using the Laplace method as in Bayesian logistic regression \cite{Bishop:PRML}. Variational methods can be applied to further accelerate the posterior update \cite{Jaakkola:00} . The details are provided in the appendix.

\subsection{Multi-Class Categorization}
\label{sec:multi}

Our MDP formulation can also be extended to deal with multi-class categorization problems, where each instance is a multiple choice question with several possible options (i.e., classes). More formally, in a multi-class setting with $C$ different classes,  we assume that the $i$-th instance  is associated with a probability vector $\btheta_i=(\theta_{i1}, \ldots \theta_{iC})$, where $\theta_{ic}$ is the probability that the $i$-th instance will be labeled as the class $c$ by a random fully reliable worker and $\sum_{i=1}^C \theta_{ic}=1$.  We assume that $\btheta_i$ has a Dirichlet prior $\btheta_i \sim \Dir(\balpha_i^0)$ and the initial state $S^0$ is a $K \times C$ matrix with  $\balpha_i^0$ as its $i$-th row. At each stage $t$ with the current state $S^t$, we determine the next  instance $i_t$  to be labeled and collect its label $y_{i_t} \in \{1,\ldots, C\}$, which follows the categorical distribution: $p(y_{i_t})= \prod_{c=1}^C  \theta_{i_tc}^{I(y_{i_t}=c)}$. Since the Dirichlet is the conjugate prior of the categorical distribution, the next state induced by the posterior distribution is: $S^{t+1}_{i_t}=S^{t}_{i_t}+ \bdelta_{y_{i_t}}$ and $S^{t+1}_{i}=S^{t}_{i}$ for all $i \neq i_t$. Here $\bdelta_c$ is a row vector with one at the $c$-th entry and zeros at all other entries. The transition probability is: $$\Pr(y_{i_t}=c | S^t, i_{t})=\E(\theta_{i_t c}| S^t)=  \frac{\alpha^t_{i_{t}c}}{\sum_{c=1}^C \alpha^t_{i_{t}c}}.$$

We denote the true set of instances in class $c$  by $H^*_c=\{i: \theta_{ic} \geq \theta_{ic'}, \forall c' \neq c  \}$. By a similar argument as in Proposition \ref{prop:H}, at the final stage $T$, the estimated set of instances belonging to class $c$ is $$ H^T_c= \{i:  P^T_{ic}\geq P^T_{ic'}, \forall c' \neq c\}, $$ where $P^t_{ic}= \Pr(i \in H^*_c | \calF_t) = \Pr (  \theta_{ic} \geq  \theta_{ic'}, \;\; \forall \;\; c'\neq c| S^t)$. We note that if the $i$-th instance belongs to more than one $H^T_c$, we only assign it to the one with the smallest index $c$ so that $\{H_c^T\}_{c=1}^C$ forms a partition of $\{1,\ldots,K\}$. Let $\bP_i^t =(P_{i1}^t, \ldots, P_{iC}^t)$ and $h(\bP_i^t) =\max_{1 \leq c \leq C} P_{ic}^t$.  The expected reward takes the form of: $$R(S^t, i_t)=\E \left( h(\bP_{i_t}^{t+1}) -   h(\bP_{i_t}^{t}) |S^{t}, i_t \right).$$ With the reward function in place, we can formulate the problem into a MDP and use DP to obtain the optimal policy and Opt-KG to compute an approximate policy. The only computational challenge is how to calculate $P^t_{ic}$ efficiently so that the reward can be evaluated. We present an efficient method in the appendix. 
We can further use Dirichlet distribution to model workers reliability as in \cite{Chao_Liu:12}. Using multi-class Bayesian logistic regression, we can also incorporate contextual information into the multi-class setting in a straightforward manner.

\section{Related Works}
\label{sec:related}

Categorical crowd labeling is one of the most popular tasks in crowdsourcing since it requires less effort of the workers to provide categorical labels than other tasks such as language translations. Most work in categorical crowd labeling are solving a static problem, i.e., inferring true labels and workers' reliability based on a static labeled dataset \cite{Dawid:79, Vikas:10, Chao_Liu:12, Peter:10, Whitehill:09, Zhou:12, QiangLiu:12,  Gao:13}. The first work that incorporates diversity of worker reliability is \cite{Dawid:79}, which uses EM to perform the point estimation on both worker reliability and true class labels. Based on that, \cite{Vikas:10} extended \cite{Dawid:79} by introducing Beta prior for workers' reliability and features of instances in the binary setting; and \cite{Chao_Liu:12} further introduced Dirichlet prior for modeling workers' reliability in the multi-class setting. Our work utilizes the modeling techniques in these two static models as basic building blocks but extends to dynamic budget allocation settings.

In recent years, there are several works that have been devoted into online learning or budget allocation in crowdsourcing \cite{Karger:13, Oh:12,  Thore:12, Ho:13, Rudin:12, Yan:11, Ece:12, Panos:13}. The method proposed in \cite{Oh:12} is  based on the one-coin model. In particular, it assigns instances to workers according to a random regular bipartite  graph. Although the error rate is proved to achieve the minimax rate, its analysis is asymptotic and method is not optimal when the budget is limited. \cite{Karger:13} further extended \cite{Oh:12} to the multi-class setting.  The new labeling uncertainty method in \cite{Panos:13} is one of the state-of-the-art methods for repeated labeling. However, it does not model each worker's reliability and incorporate it into the allocation process.   \cite{Ho:13}  proposed an online primal dual method for adaptive task assignment and investigated the sample complexity to guarantee that the probability of making an error for each instance is less that a threshold. However, it requires gold samples to estimate workers' reliability. \cite{Ece:12} used MDP to address a different decision problem in crowd labeling, where the decision maker collects labels for each instance one after another and only decides whether to hire an additional worker or not. Basically, it is an optimal stopping problem since there is no pre-fixed amount of budget and one needs to balance the accuracy v.s. the amount of budget. Since the accuracy and the amount of budget are in different metrics, such a balance could be very subjective. Furthermore, the MDP framework in \cite{Ece:12} cannot distinguish different workers. To the best of our knowledge, there is no existing method that characterizes the \emph{optimal} allocation policy for finite $T$. In this work,  with the MDP formulation and DP algorithm, we characterize the optimal policy for budget allocation in crowd labeling under any budget level.


We also note that the budget allocation in crowd labeling is  fundamentally different from noisy active learning \cite{Settles:09, Nowak:09}. Active learning usually does not model the variability of labeling difficulties among instances and assumes a single (noisy) oracle; while in crowd labeling, we need to model both instances' labeling difficulty and different workers' reliability. Secondly, active learning requires the feature information of instances for the decision, which could be unavailable in crowd labeling. Finally, the goal of the active learning is to label as few instances as possible to learn a good classifier. In contrast,  for budget allocation in crowd labeling, the goal is to infer the true labels for as many instances as possible.


In fact, our MDP formulation is essentially a finite-horizon Bayesian multi-armed bandit (MAB) problem.
While the infinite-horizon Bayesian MAB has been well-studied and the optimal policy can be computed via Gittins index \cite{Gittins:89}, for finite-horizon Bayesian MAB, the Gittins index rule is only an approximate policy with high computational cost. 
The proposed Opt-KG and a more general conditional value-at-risk based KG could be general policies for Bayesian MAB. Recently, a Bayesian UCB policy was proposed to address  a different Bayesian MAB problem \cite{Kaufmann:12}. However, it is not clear how to directly apply the policy to our problem since we are not updating the posterior of the mean of rewards as in \cite{Kaufmann:12}. 
We note that our problem is also related to optimal stopping problem. The main difference is that the optimal stopping problem is infinite-horizon while our problem is finite-horizon and the decision process must stop when the budget is exhausted.


\section{Experiments}
\label{sec:exp}

In this section, we conduct empirical study to show some interesting properties of the proposed Opt-KG policy and compare its performance to  other methods. We note that, first, we observe that several commonly used priors such as the uniform prior ($\B(1,1)$), Jeffery prior ($\B(1/2, 1/2)$) and Haldane prior ($\B(0,0)$) for instances' soft-label $\{\theta_i\}_{i=1}^K$ lead to very similar performance. Therefore, we adopt the uniform prior ($\B(1,1)$) unless otherwise specified. Second, for each simulated experiment, we randomly generate 20 different sets of data and report the averaged accuracy. Here, the accuracy is defined as $\left(|H_T \cap H^*| + |(H_T)^c \cap (H^*)^c|\right)/K$, which is normalized between $[0,1]$.  The deviations for different methods are similar and quite small and thus omitted for the purpose of better visualization and space-saving.

\subsection{Simulated Study}

\subsubsection{Study on Labeling Frequency}

\begin{figure}[!t]
\centering
\subfigure[$T=5K=105$]{
  \includegraphics[width=0.31\textwidth]{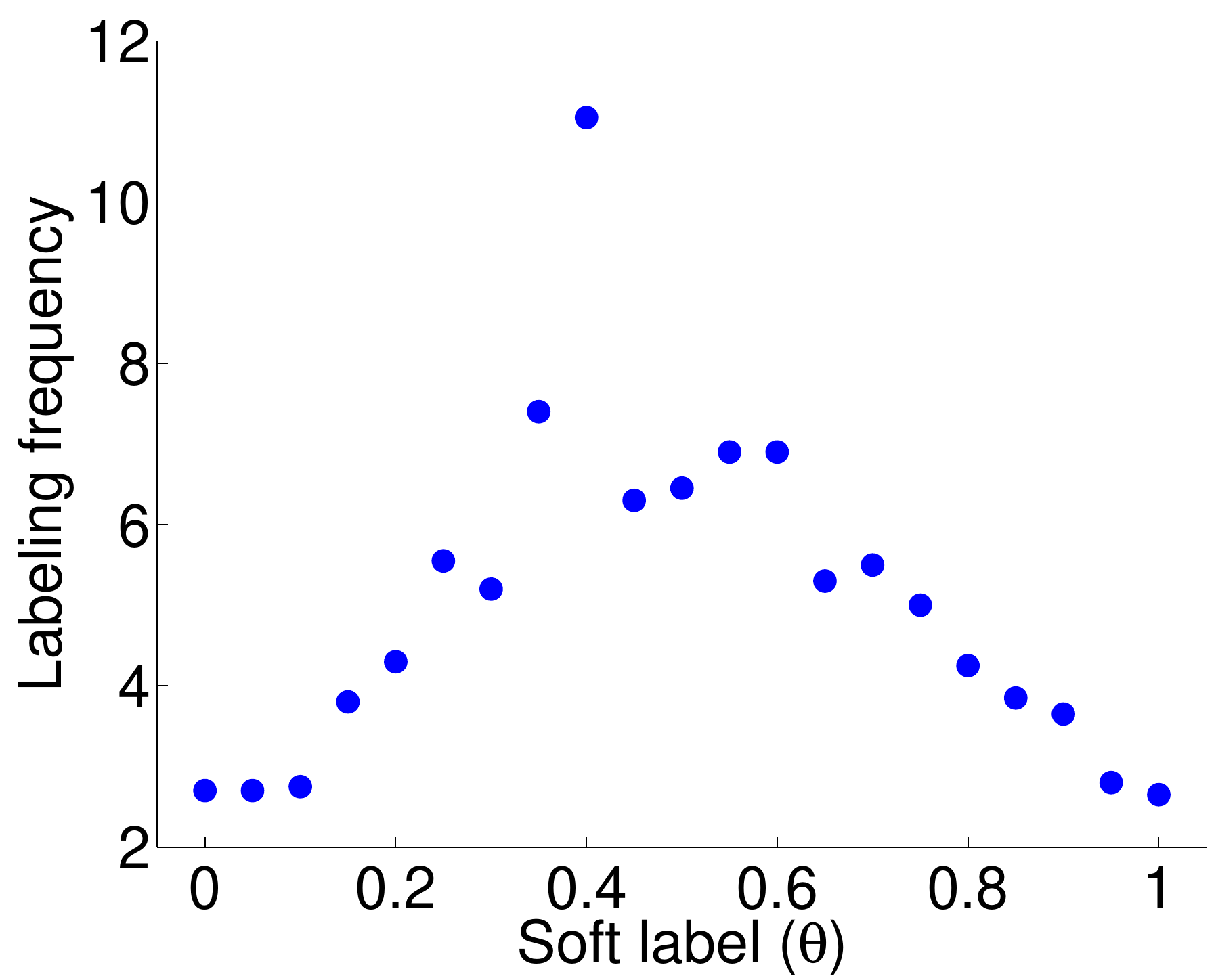}
	    \label{fig:task_cnt_105}
}\subfigure[$T=15K=315$]{
  \includegraphics[width=0.31\textwidth]{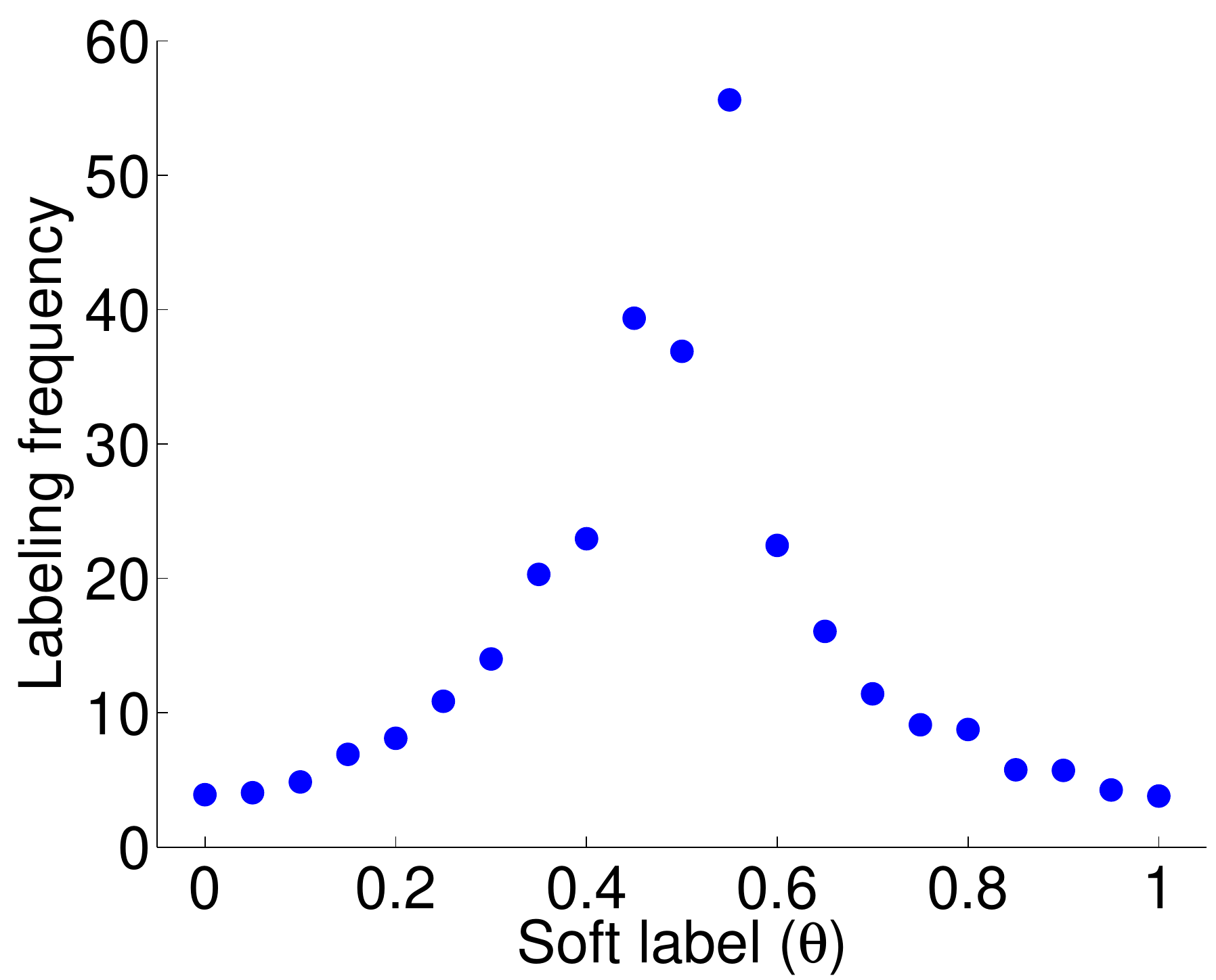}
	    \label{fig:task_cnt_315}
}
\subfigure[$T=50K=1050$]{
  \includegraphics[width=0.31\textwidth]{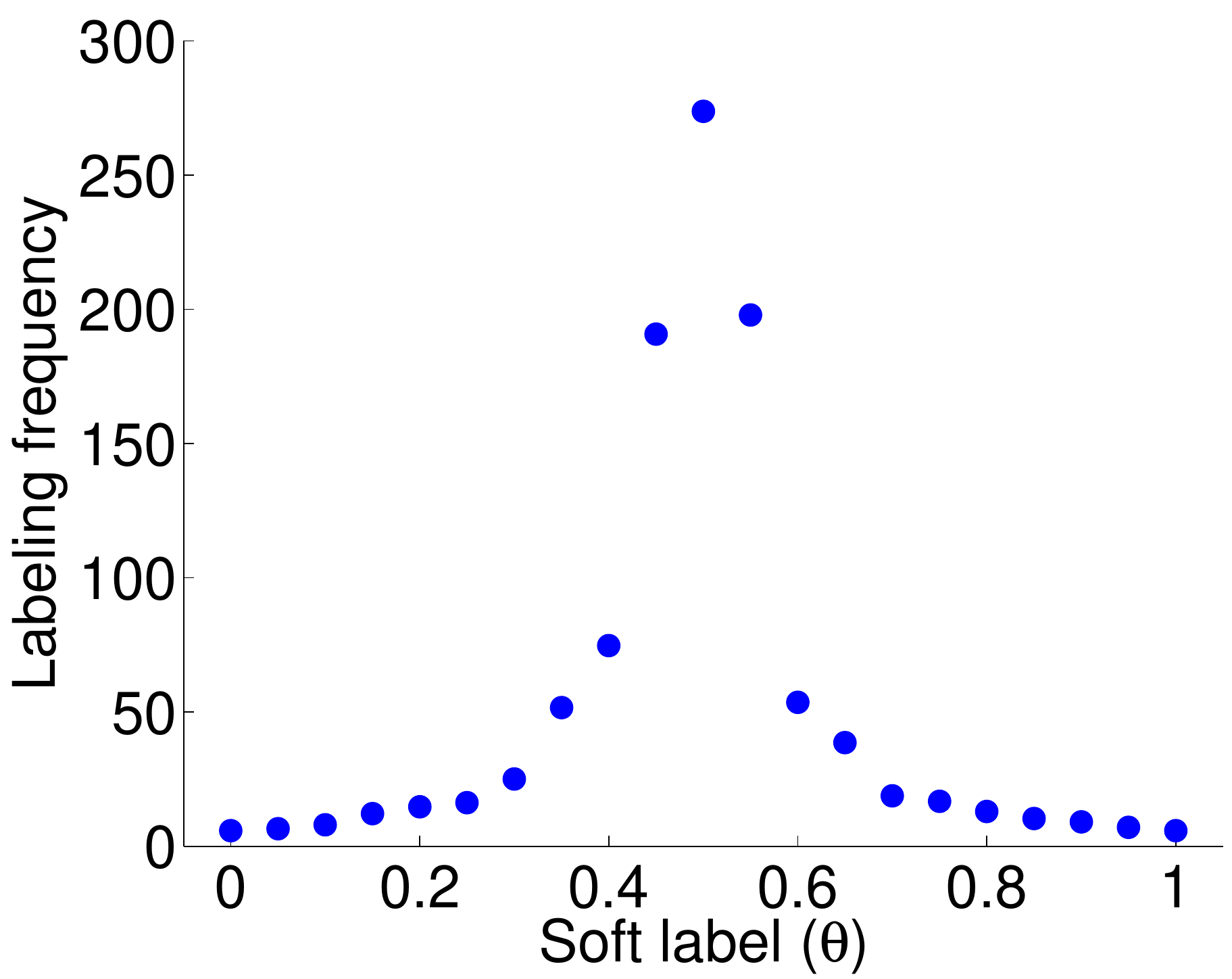}
	    \label{fig:task_cnt_1050}
}
\caption{Labeling counts for instances with different levels of ambiguity.}
\label{fig:task_cnt}
\end{figure}

\begin{figure}[!t]
\centering
\subfigure[$T=5K=105$]{
  \includegraphics[width=0.31\textwidth]{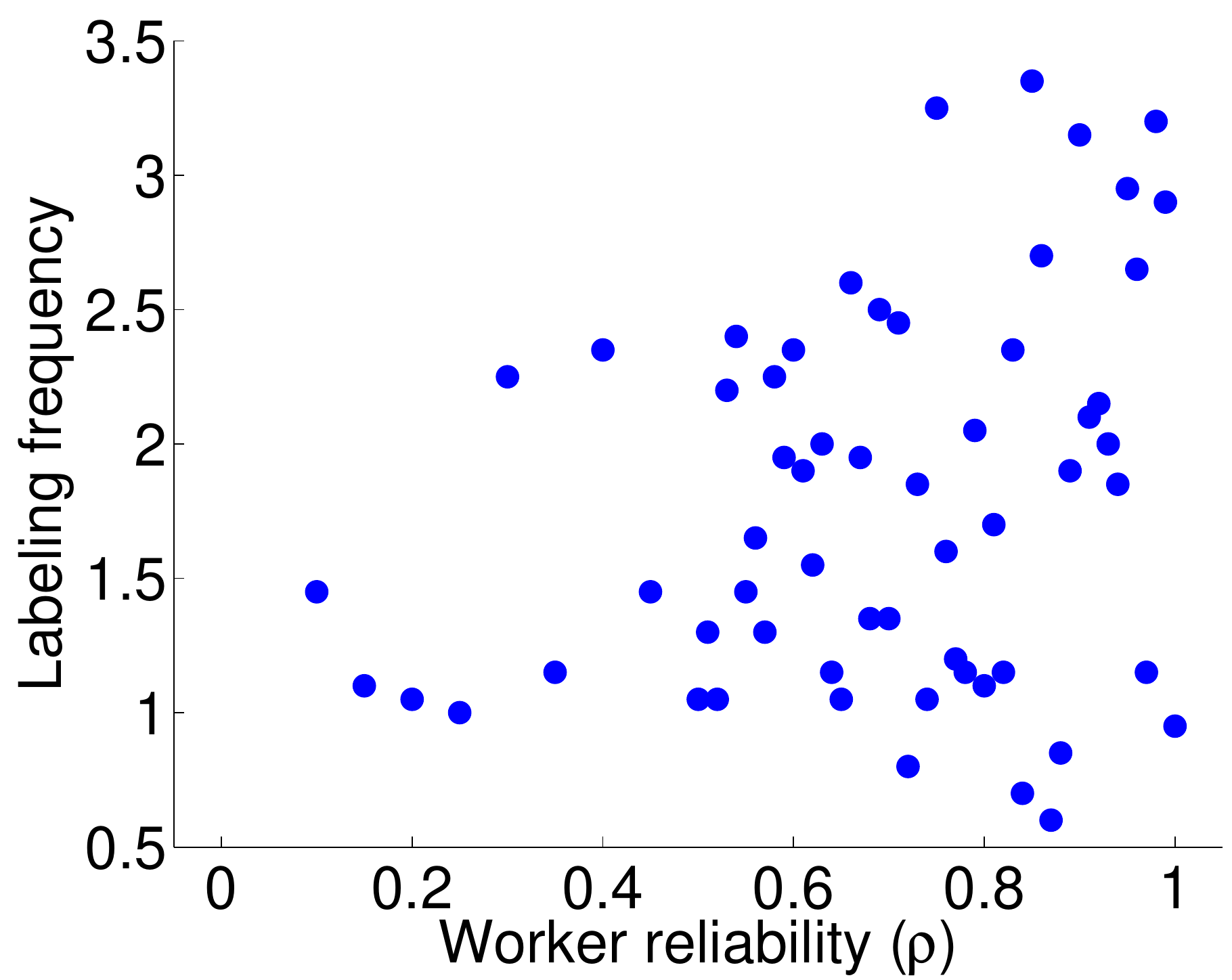}
	    \label{fig:worker_cnt_105}
}\subfigure[$T=15K=315$]{
  \includegraphics[width=0.31\textwidth]{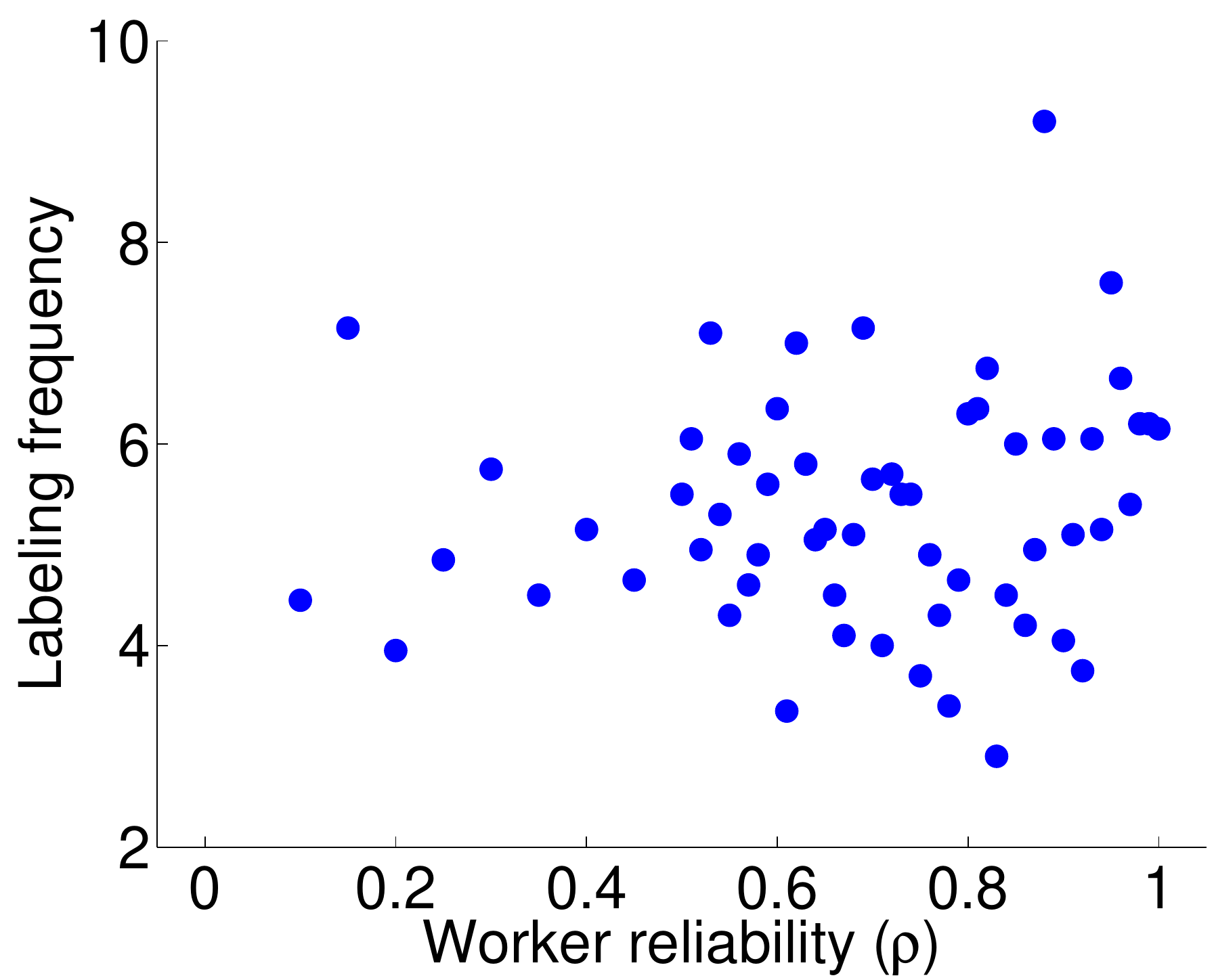}
	    \label{fig:worker_cnt_315}
}
\subfigure[$T=50K=1050$]{
  \includegraphics[width=0.31\textwidth]{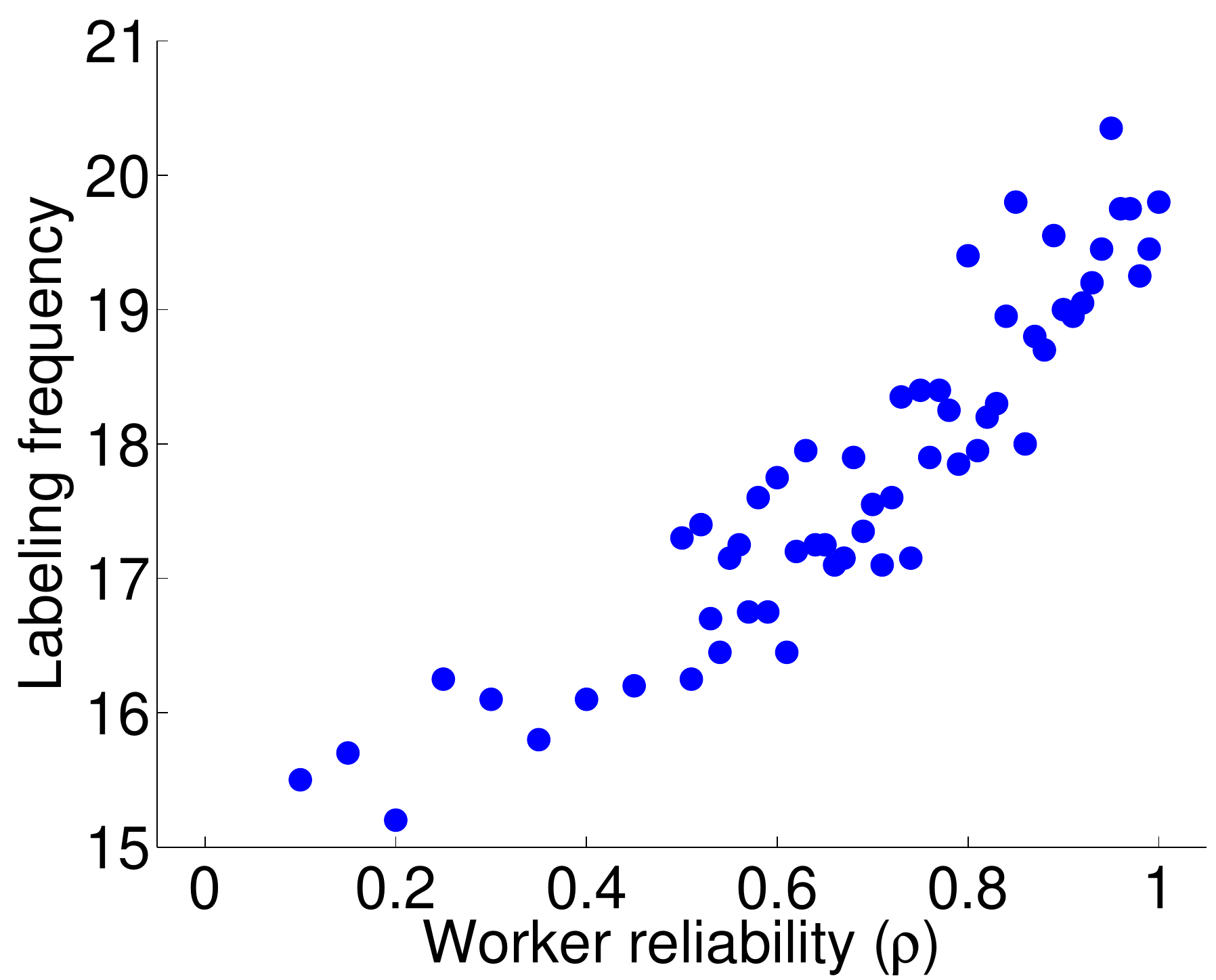}
	    \label{fig:worker_cnt_1050}
}
\caption{Labeling counts for workers with different levels of reliability.}
\label{fig:worker_cnt}
\end{figure}

We first investigate  that, in the homogeneous noiseless worker setting (i.e., workers are fully reliable), how the total budget is allocated among instances with different levels of ambiguity. In particular, we assume there are $K=21$ instances with soft-labels $\btheta=(\theta_1, \theta_2, \theta_3, \ldots, \theta_K)=(0, 0.05, 0.1, \ldots, 1)$. We vary the total budget $T=5K, 15K, 50K$ and report the number of  times that each instance is labeled on average over 20 independent runs. The results are presented in Figure \ref{fig:task_cnt}. It can be seen from Figure \ref{fig:task_cnt} that, more ambiguous instances with $\theta$ close to 0.5 in general receive more labels than those simple instances with $\theta$ close to 0 or 1.   A more interesting observation is that when the budget level is low (e.g., $T=5K$ in Figure \ref{fig:task_cnt_105}), the policy spends less budget on those very ambiguous instances (e.g., $\theta=0.45$ or $0.5$ ), but more budget on exploring less ambiguous instances (e.g., $\theta=0.35$, $0.4$ or $0.6$). When the budget goes higher (e.g., $T=15K$ in Figure \ref{fig:task_cnt_315}), those very ambiguous instances receive more labels but the most ambiguous instance ($\theta=0.5$) not necessarily receives the most labels. In fact, the instances with $\theta=0.45$ and $\theta=0.55$ receive more labels than that of the most ambiguous instance. When the total budget is sufficiently large (e.g., $T=50K$ in Figure \ref{fig:task_cnt_1050}), the most ambiguous  instance receives the most labels since all the other instances have received enough labels to infer their true labels.

Next, we investigate that, in the heterogeneous worker setting, how many instances each worker is assigned. We simulate $K=21$ instances' soft-labels as before and further simulate workers' reliability $\brho=(\rho_1, \rho_2, \ldots, \rho_M)=(0.1, 0.15, \ldots, 0.5, 0.505, 0.515, \ldots, 0.995)$ for $M=59$ workers. Such a simulation ensures that there are more reliable workers, which is in line with actual situation.  We vary the total budget $T=5K, 15K, 50K$ and report the number of instances that each worker is assigned on average over 20 independent runs in Figure \ref{fig:worker_cnt}. As one can see, when the budget level goes up, there is clear trend that more reliable workers receive more instances.

\subsubsection{Prior for instances}

\label{sec:exp_prior_instance}

We investigate how robust Opt-KG is when using the uniform prior for each $\theta_i$. We first simulate $K=50$ instances with each $\theta_i \sim \B(0.5, 0.5)$, $\theta_i \sim \B(2,2)$, $\theta_i \sim \B(2,1)$ or $\theta_i \sim \B(4,1)$. The density functions of these four different Beta distributions are plotted in Figure \ref{fig:beta_task_pdf}.  For each generating distribution of $\theta_i$, we compare Opt-KG using the uniform prior ($\B(1,1)$) (in red line) to Opt-KG with the true generating distribution as the prior (in blue line). The comparison in accuracy with different levels of budget ($T=2K, \ldots, 20K$) is shown in Figure \ref{fig:prior_task}. As we can see, the performance of Opt-KG using two different priors are quite similar for most generating distributions except for $\theta_i \sim \B(4,1)$ (i.e., the highly imbalanced class distribution). When $\theta_i \sim \B(4,1)$, the Opt-KG with uniform prior needs at least $T=16K$ units of budget to match the performance of Opt-KG with true generating distribution as the prior. This result indicates that for balanced class distributions, the uniform prior is a good choice and robust to the underlying distribution of $\theta_i$. For highly imbalanced class distributions, if a uniform prior is adopted, one needs more budget to recover from the inaccurate prior belief.

\begin{figure}[!t]
\centering
\subfigure[$\B(0.5,0.5)$]{
  \includegraphics[width=0.22\textwidth]{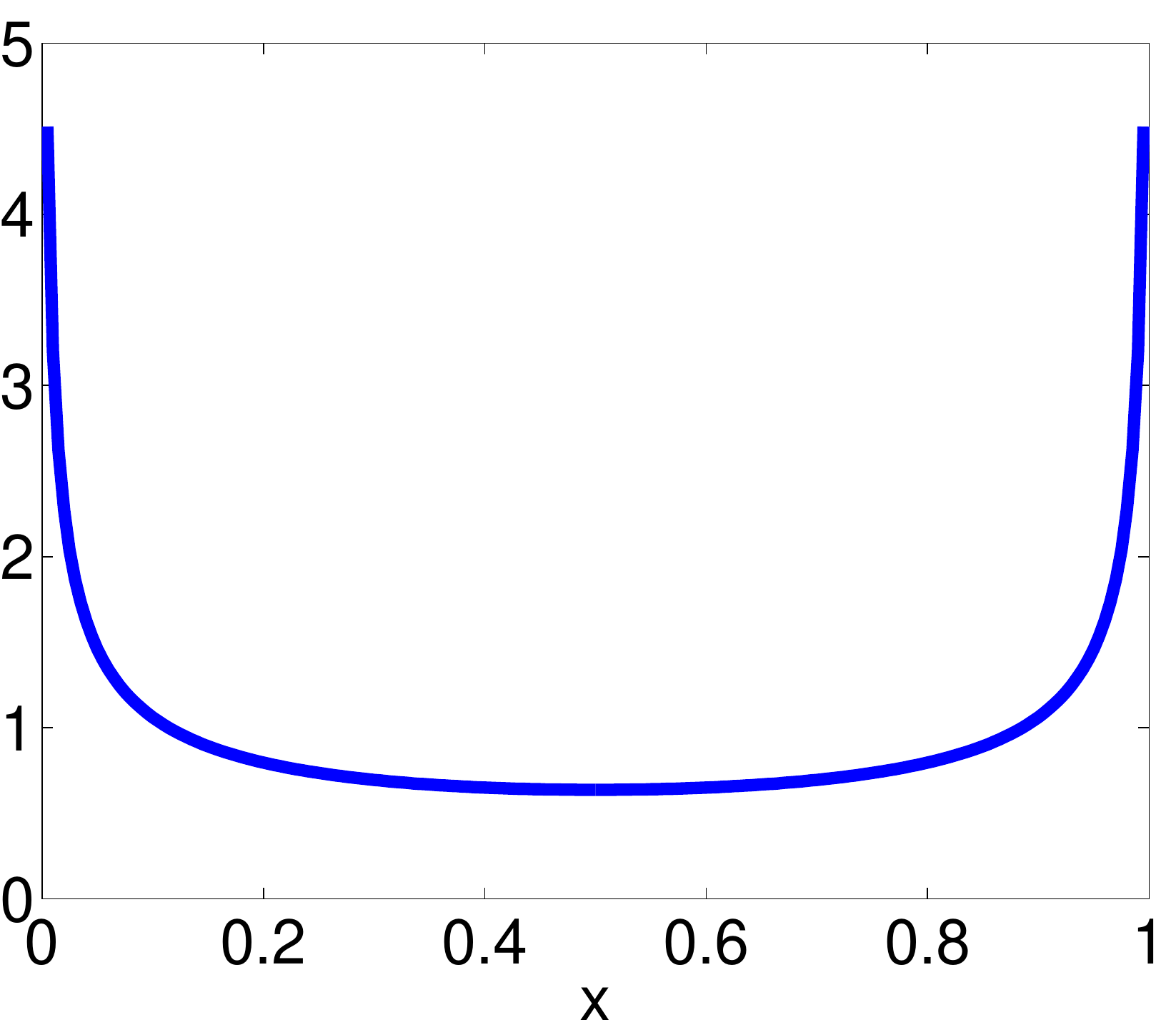}
	    \label{fig:beta_05_05}
}\subfigure[$\B(2,2)$]{
  \includegraphics[width=0.22\textwidth]{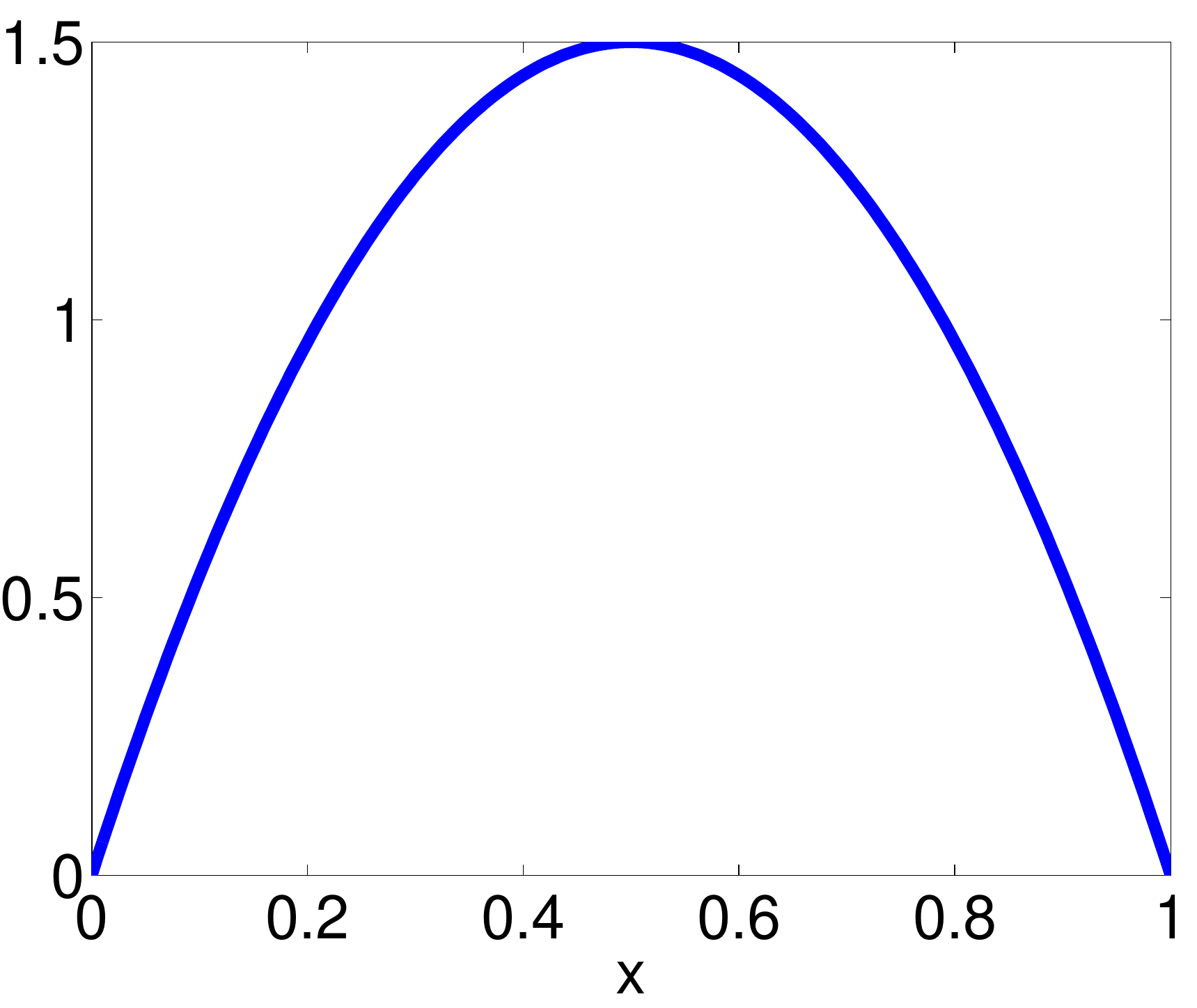}
	    \label{fig:beta_2_2}
}
\subfigure[$\B(2,1)$]{
  \includegraphics[width=0.22\textwidth]{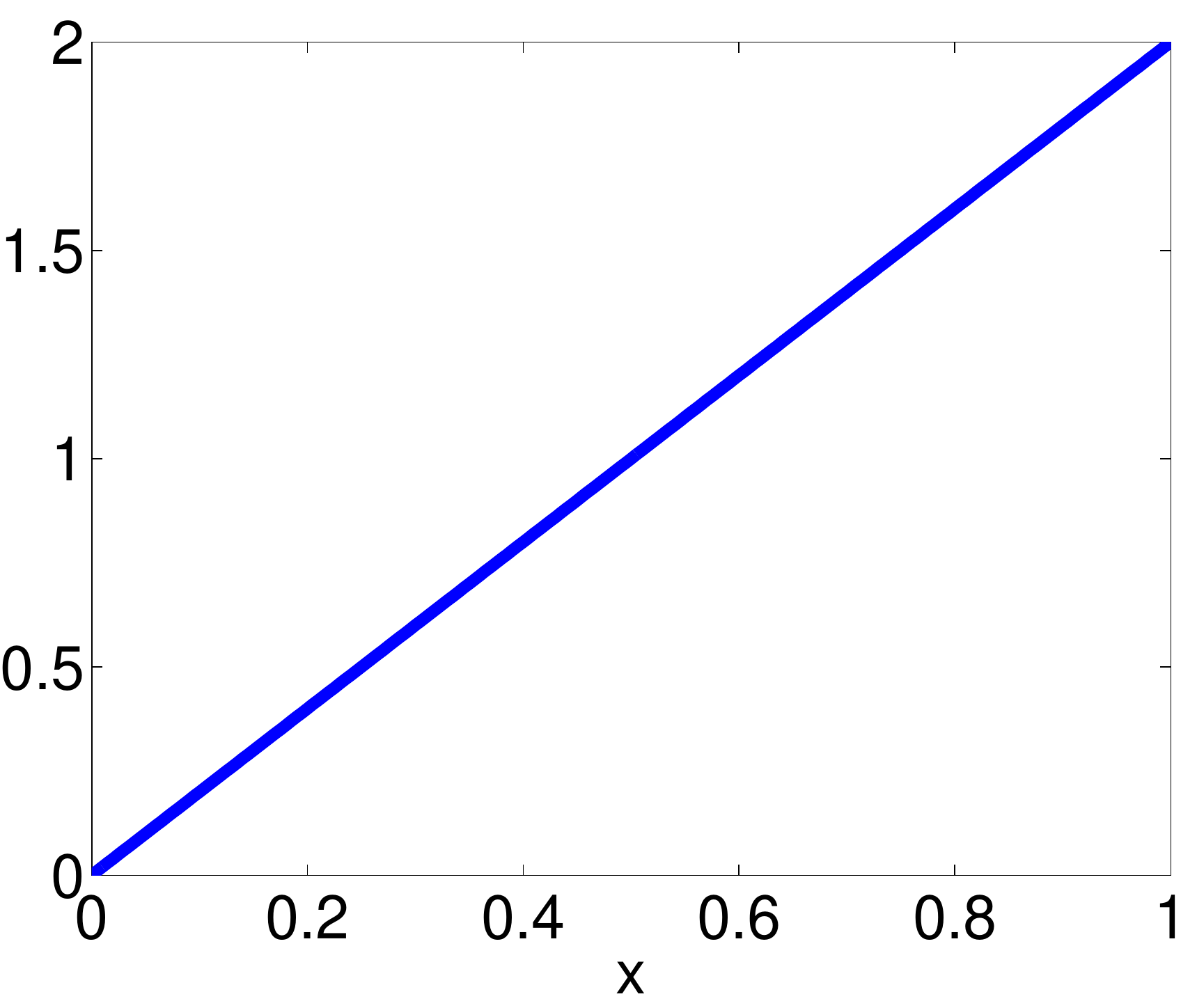}
	    \label{fig:beta_2_1}
}
\subfigure[$\B(4,1)$]{
  \includegraphics[width=0.22\textwidth]{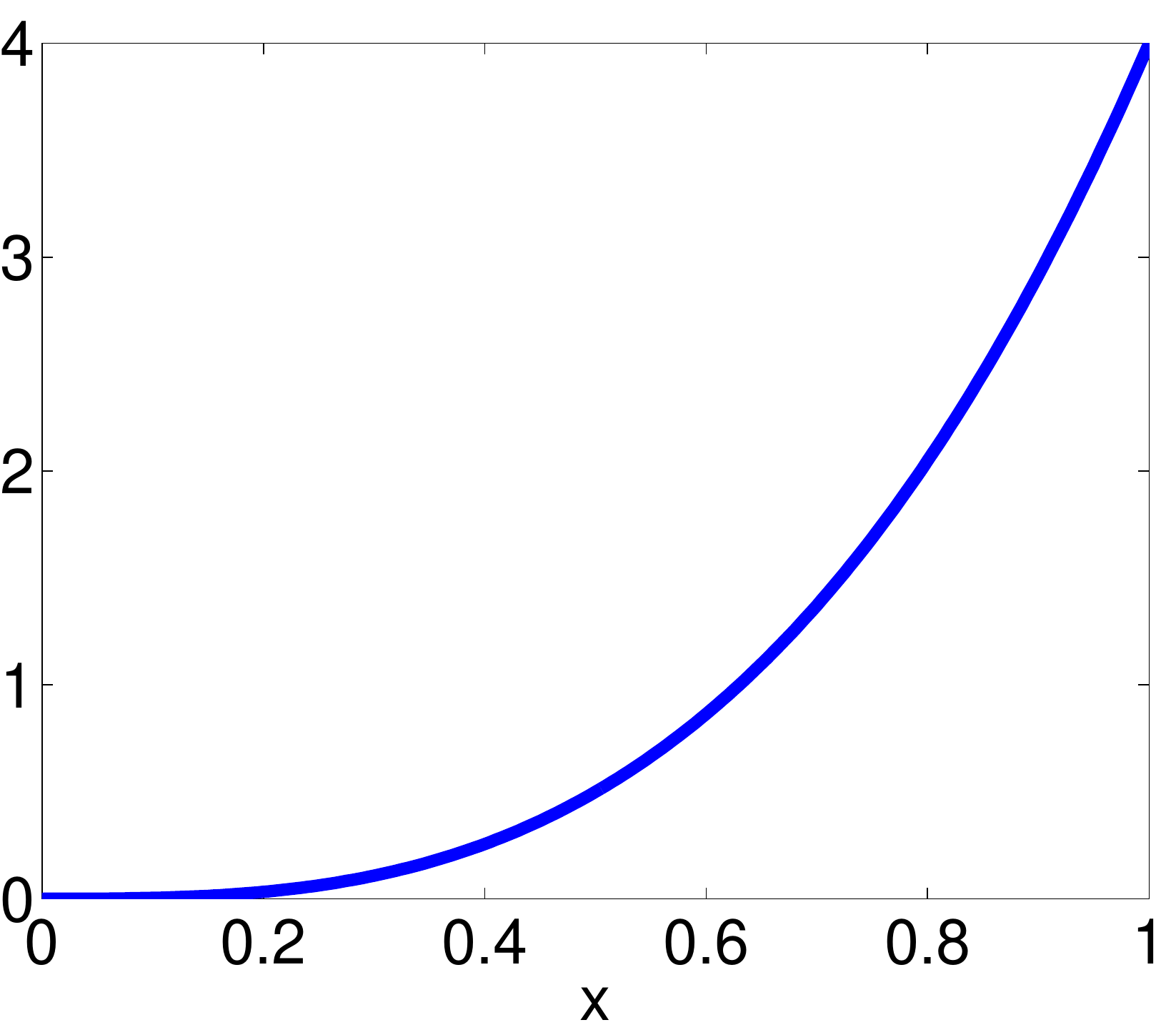}
	    \label{fig:beta_4_1}
}
\caption{Density plot for different Beta distributions for generating each $\theta_i$. Here, (a) represents there are more easier instances; (b) more ambiguous instances; (c)  \& (d) imbalanced class distributions with more positive instances.}
\label{fig:beta_task_pdf}
\end{figure}
\begin{figure}[!h]
\centering \hspace{-2mm}
\subfigure[$\theta_i \sim \B(0.5,0.5)$]{
  \includegraphics[width=0.44\textwidth]{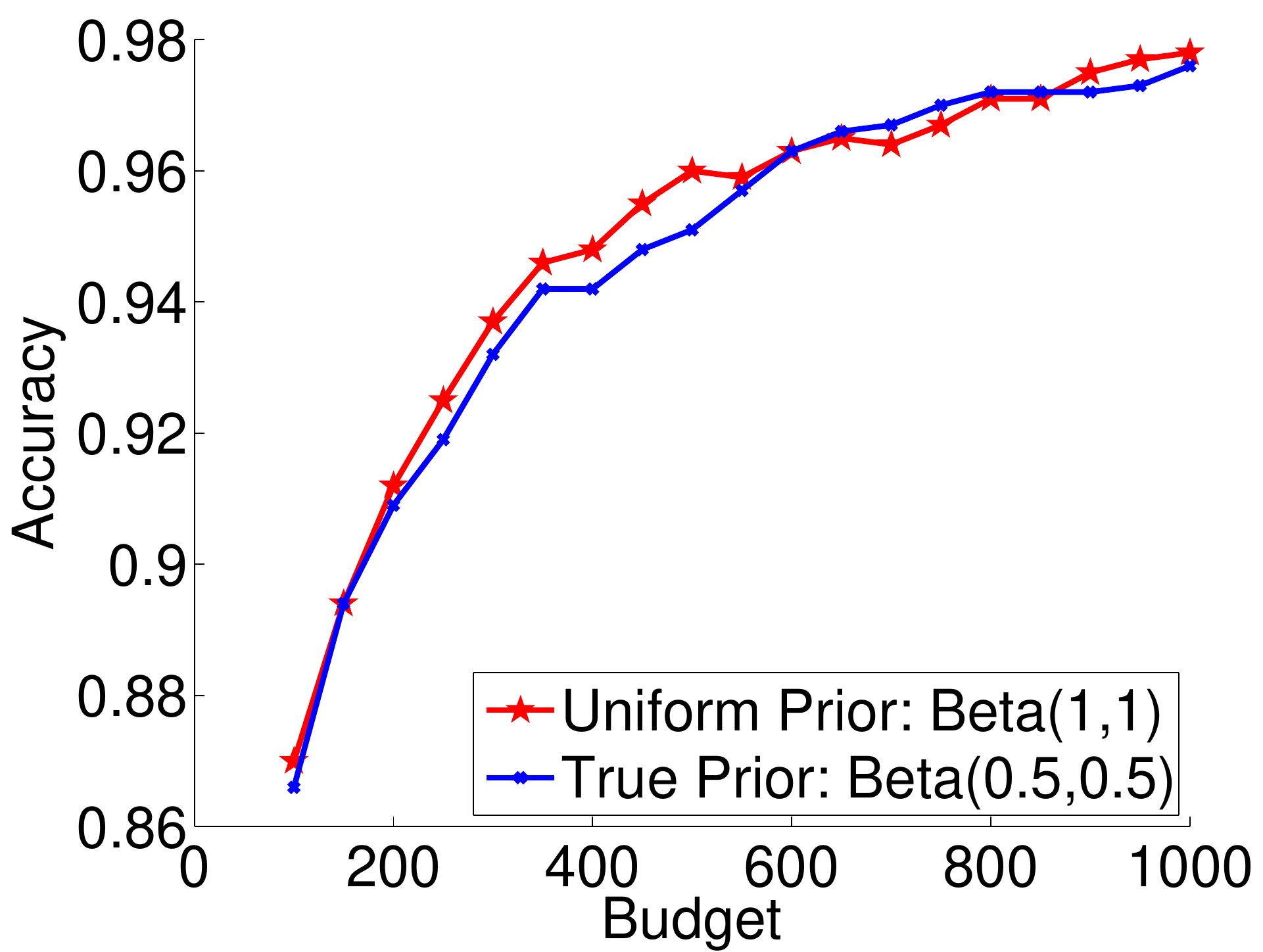}
	    \label{fig:prior_task_05_05}
}\hspace{-2mm}\subfigure[$\theta_i \sim \B(2,2)$]{
  \includegraphics[width=0.44\textwidth]{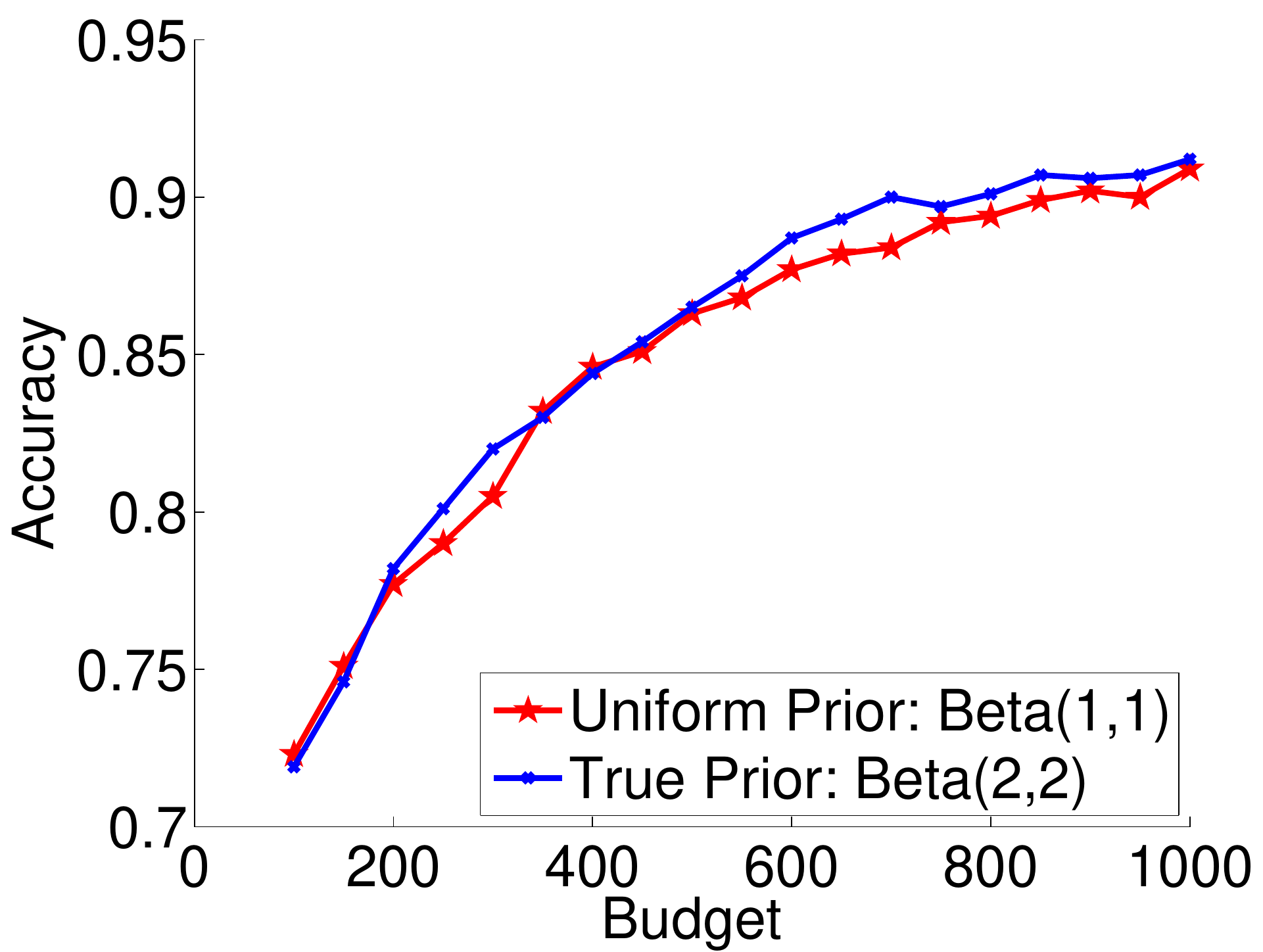}
	    \label{fig:prior_task_2_2}
}\hspace{-2mm}
\subfigure[$\theta_i \sim \B(2,1)$]{
  \includegraphics[width=0.44\textwidth]{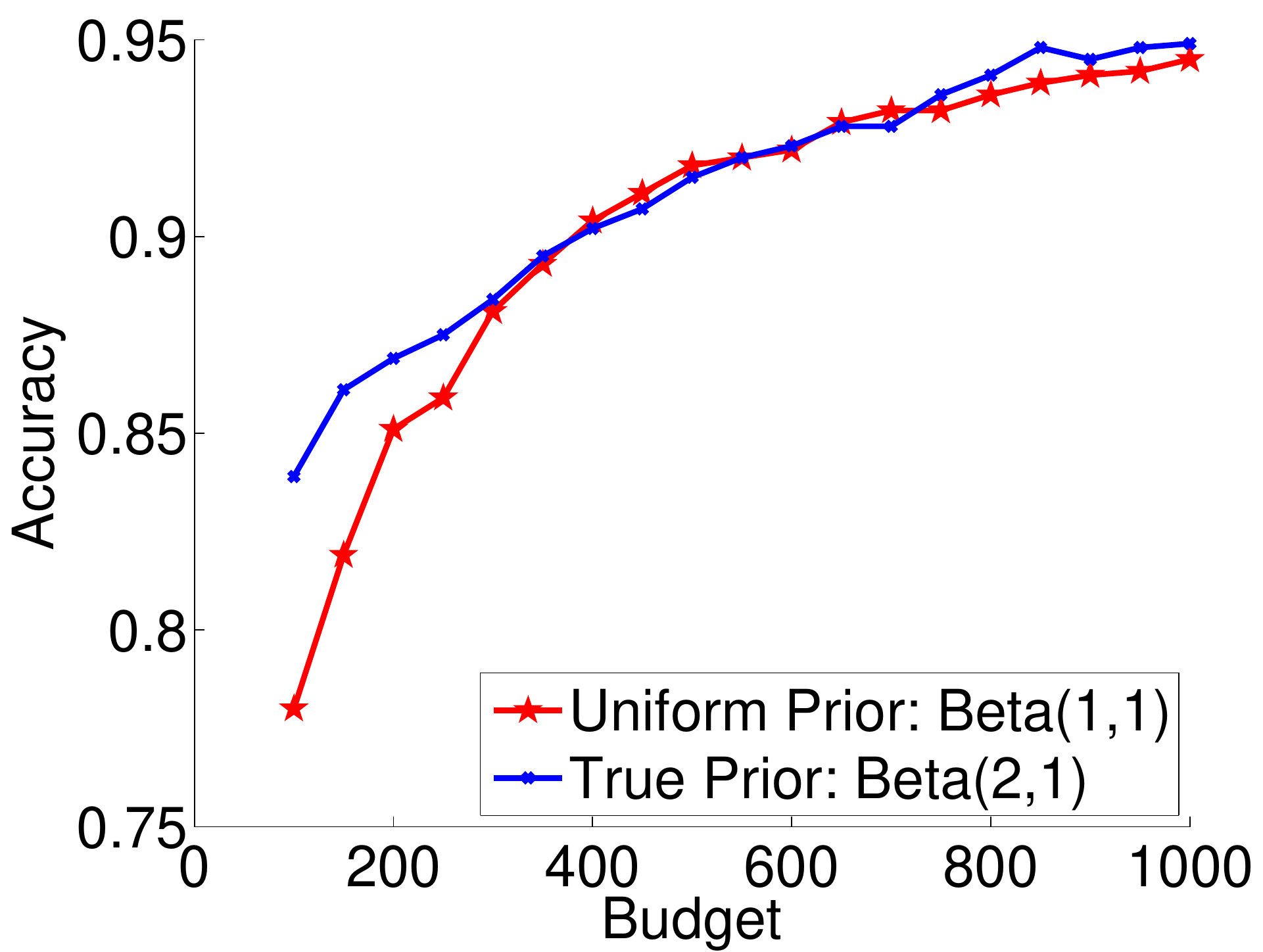}
	    \label{fig:prior_task_2_1}
}\hspace{-2mm}
\subfigure[$\theta_i \sim \B(4,1)$]{
  \includegraphics[width=0.44\textwidth]{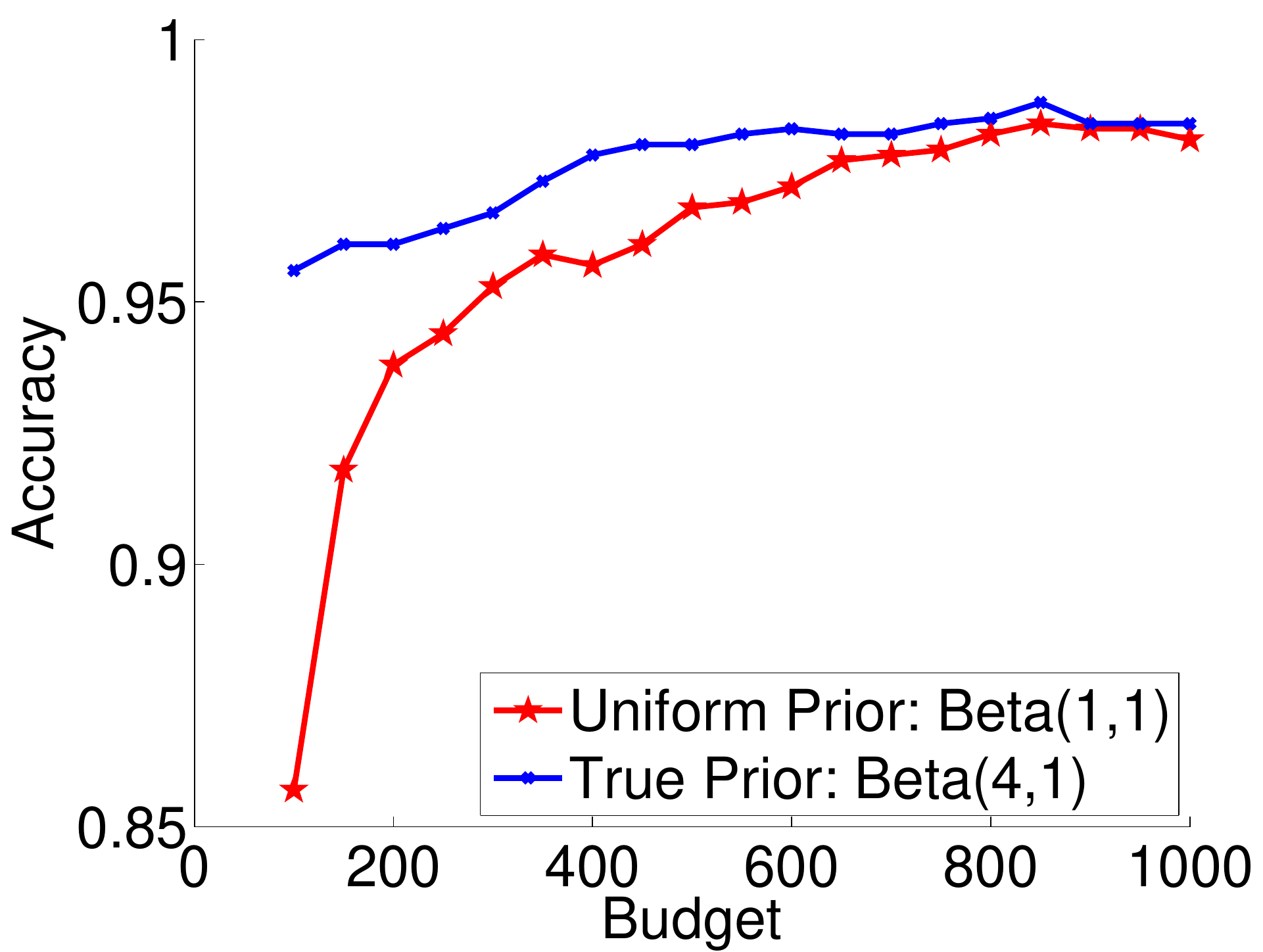}
	    \label{fig:prior_task_4_1_1}
}
\caption{Comparison between Opt-KG using the uniform distribution and true generating distribution as the prior.}
\label{fig:prior_task}
\end{figure}

\subsubsection{Prior on workers}

We investigate how sensitive the prior for the workers' reliability $\rho_j$ is. In particular, we simulate $K=50$ instances with each
$\theta_i \sim \B(1,1)$ and $M=100$ workers with $\rho_j \sim \B(3,1)$,  $\rho_j \sim \B(8,1)$ or  $\rho_j \sim \B(5,2)$. We ensure that there are more reliable workers than spammers or poorly informed workers, which is in line with the actual situation. We use the prior $\B(4,1)$, which indicates that we have the prior belief that most workers preform reasonably well and the averaged accuracy is $4/5=80\%$. In Figure \ref{fig:beta_worker_pdf}, we show different density functions for generating $\rho_j$ and the prior that we use (in Figure \ref{fig:beta_worker_pdf} (d)). For each generating distribution of $\theta_i$, we compare the Opt-KG policy using the prior ($\B(4,1)$) (in red line) to the Opt-KG with the true generating distribution as the prior (in blue line). The comparison in accuracy with different levels of budget ($T=2K, \ldots, 20K$) is shown in Figure \ref{fig:prior_worker}. From Figure \ref{fig:prior_worker}, we observe that the performance of Opt-KG using two different priors are quite similar in all different settings. Hence, we will use $\B(4,1)$ as the prior when the true prior of workers is unavailable.

\begin{figure}[!t]
\centering
\subfigure[$\B(3,1)$]{
  \includegraphics[width=0.22\textwidth]{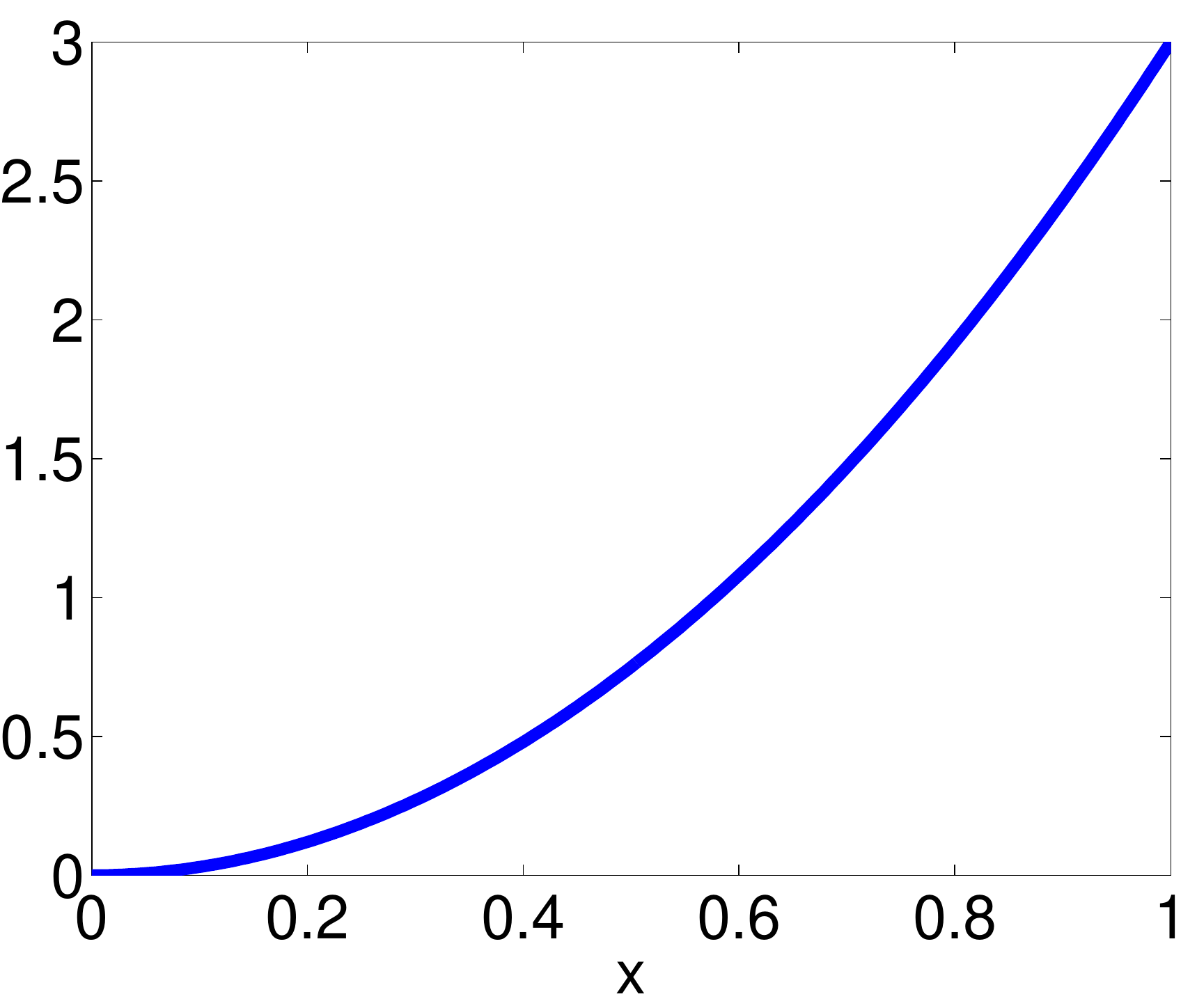}
	    \label{fig:beta_3_1}
}\subfigure[$\B(8,1)$]{
  \includegraphics[width=0.22\textwidth]{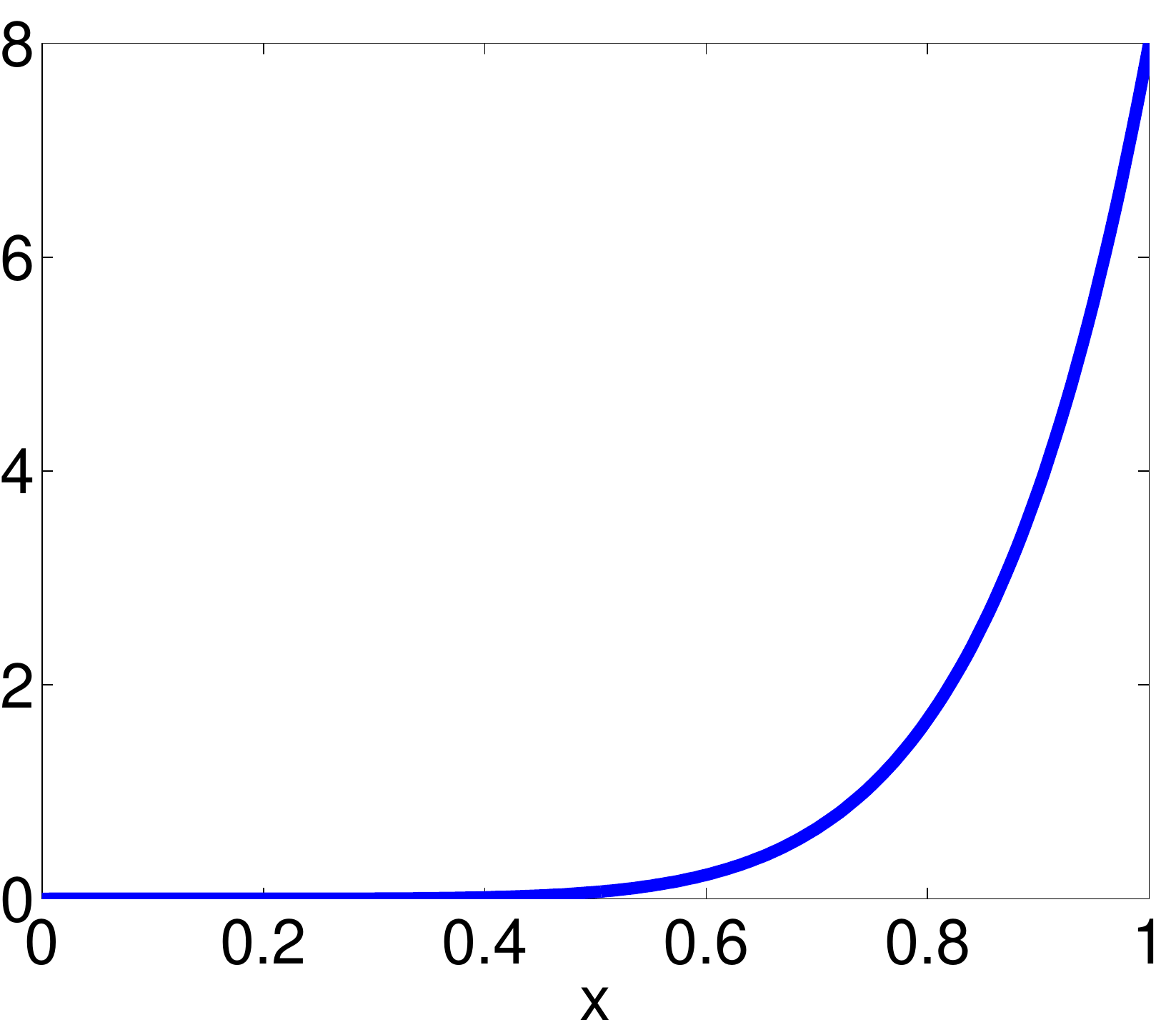}
	    \label{fig:beta_8_1}
}
\subfigure[$\B(5,2)$]{
  \includegraphics[width=0.22\textwidth]{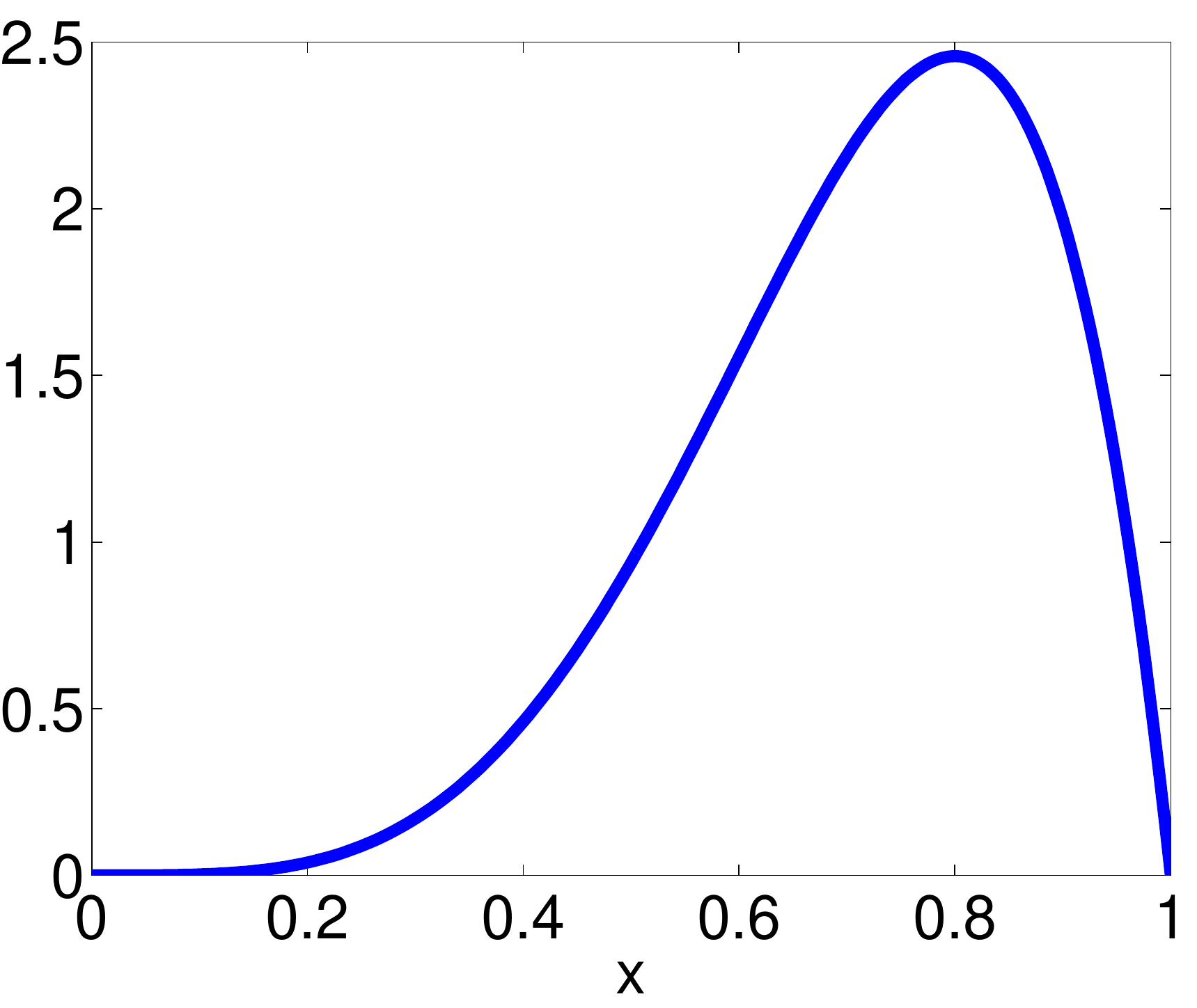}
	    \label{fig:beta_5_2}
}
\subfigure[$\B(4,1)$]{
  \includegraphics[width=0.22\textwidth]{./beta_4_1}
	    \label{fig:beta_4_1_1}
}
\caption{Density plot for different Beta distributions for generating $\rho_j$. The plot in (d) is the one that we use as the prior.}
\label{fig:beta_worker_pdf}
\end{figure}
\begin{figure}[!t]
\centering
\subfigure[$\theta_i \sim \B(3,1)$]{
  \includegraphics[width=0.31\textwidth]{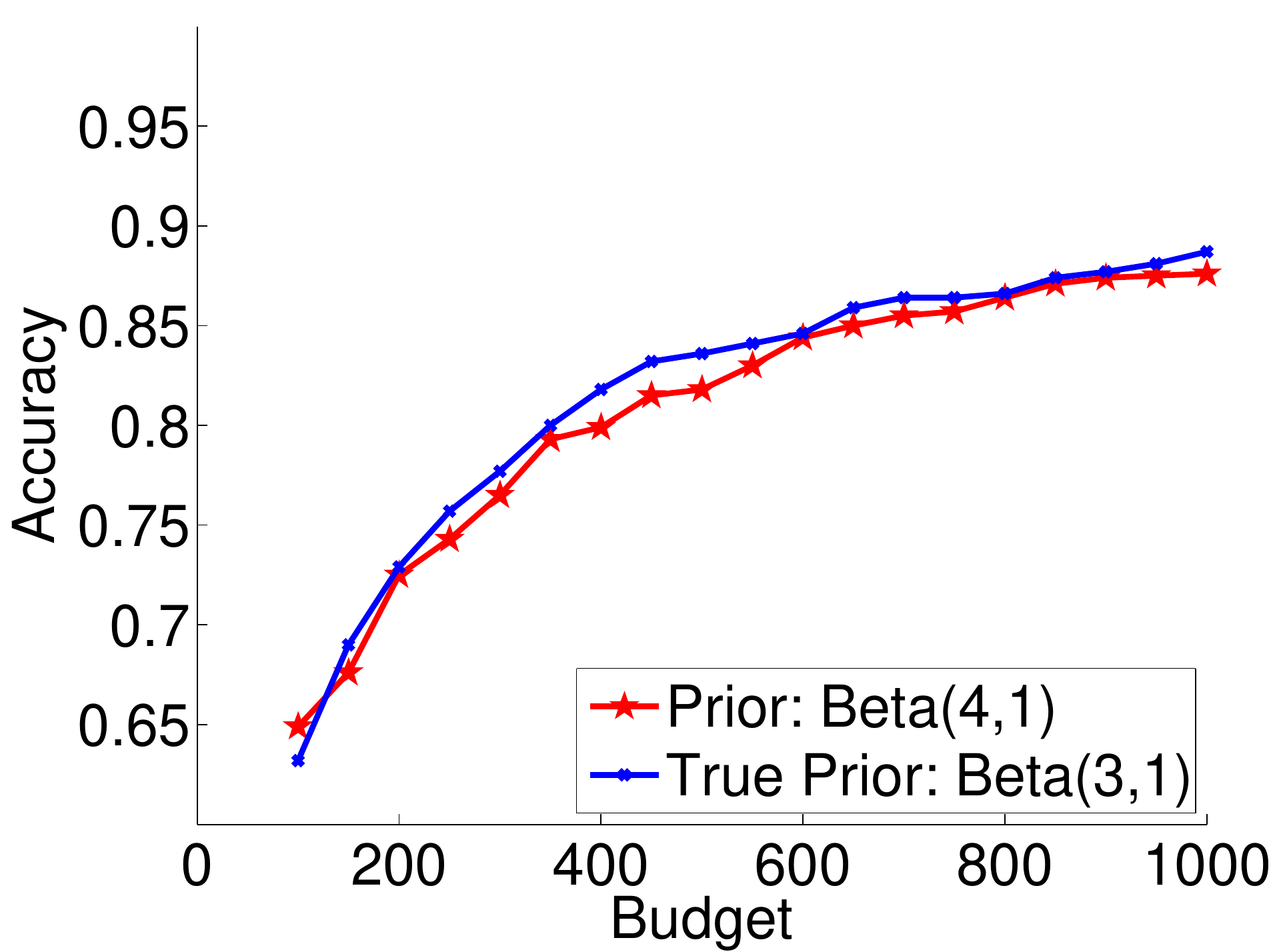}
	    \label{fig:prior_worker_3_1}
}\subfigure[$\theta_i \sim \B(8,1)$]{
  \includegraphics[width=0.31\textwidth]{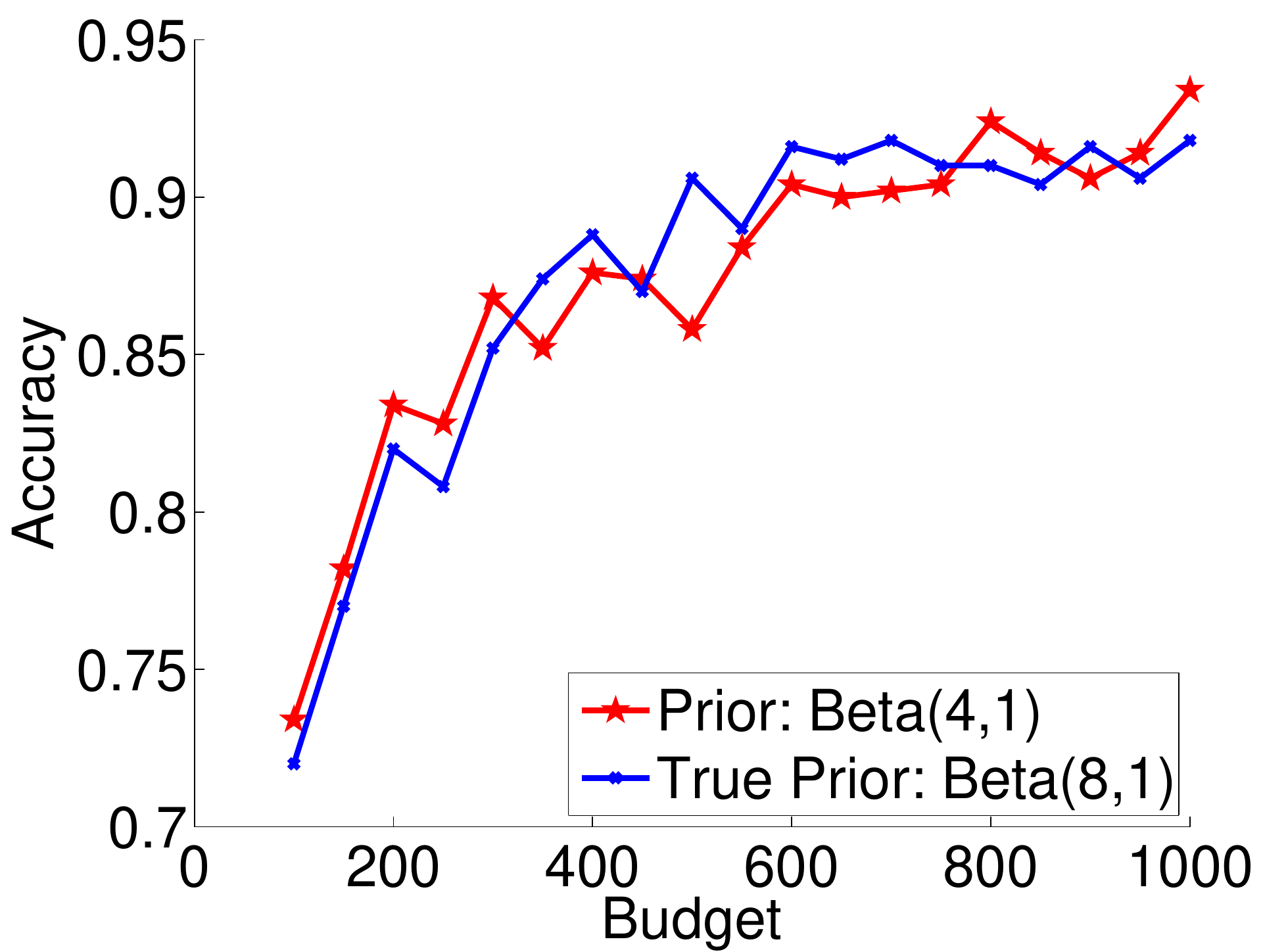}
	    \label{fig:prior_worker_8_1}
}
\subfigure[$\theta_i \sim \B(5,2)$]{
  \includegraphics[width=0.31\textwidth]{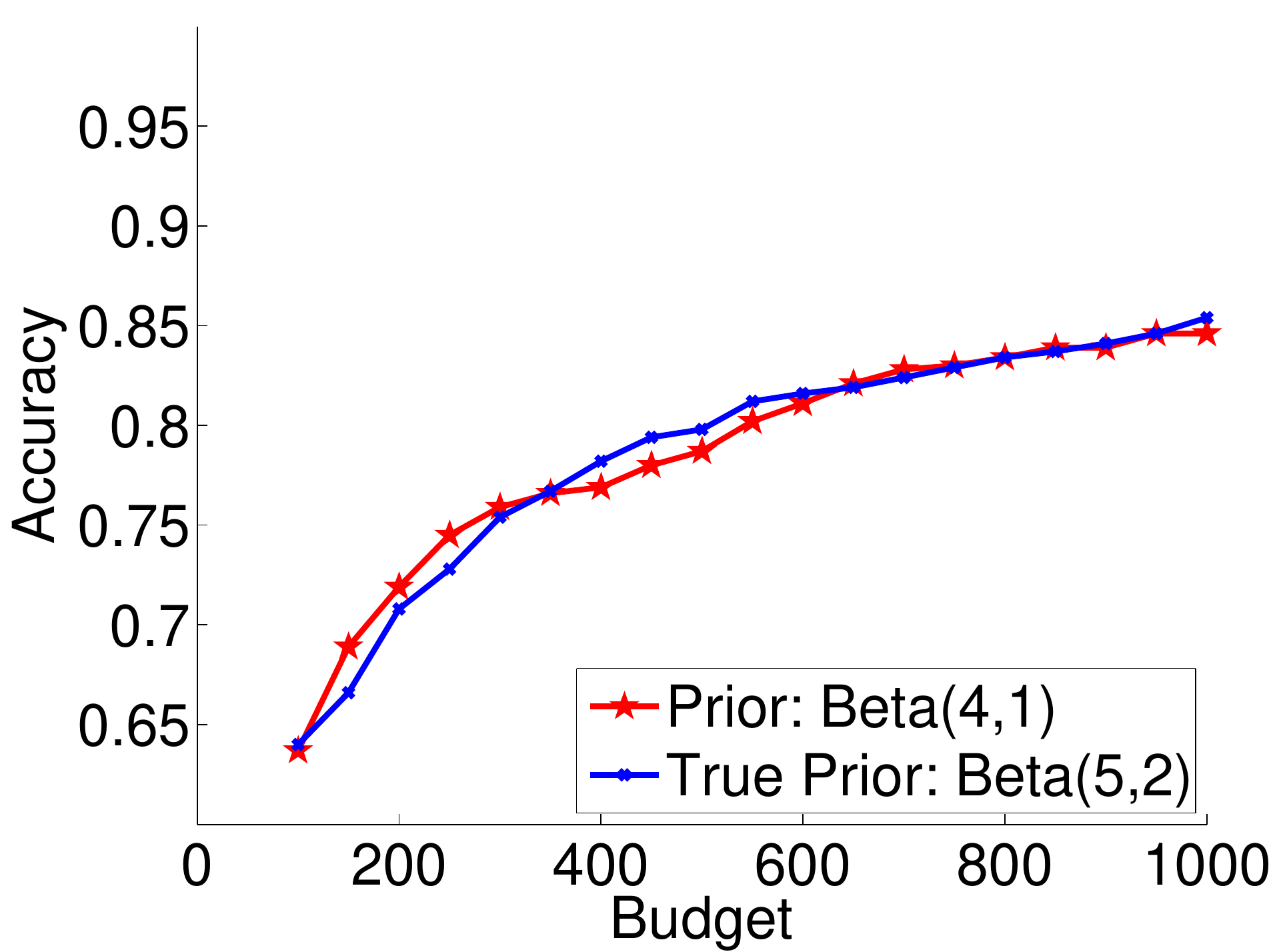}
	    \label{fig:prior_worker_5_2}
}
\caption{Comparison between Opt-KG using $\B(4,1)$ and true generating distribution prior as the prior.}
\label{fig:prior_worker}
\end{figure}

\subsubsection{Performance comparison under the homogeneous noiseless worker setting}

\begin{figure}[!t]
\centering
\subfigure[$\theta_i \sim \B(1,1)$ (True Prior)]{
  \includegraphics[width=0.3\textwidth]{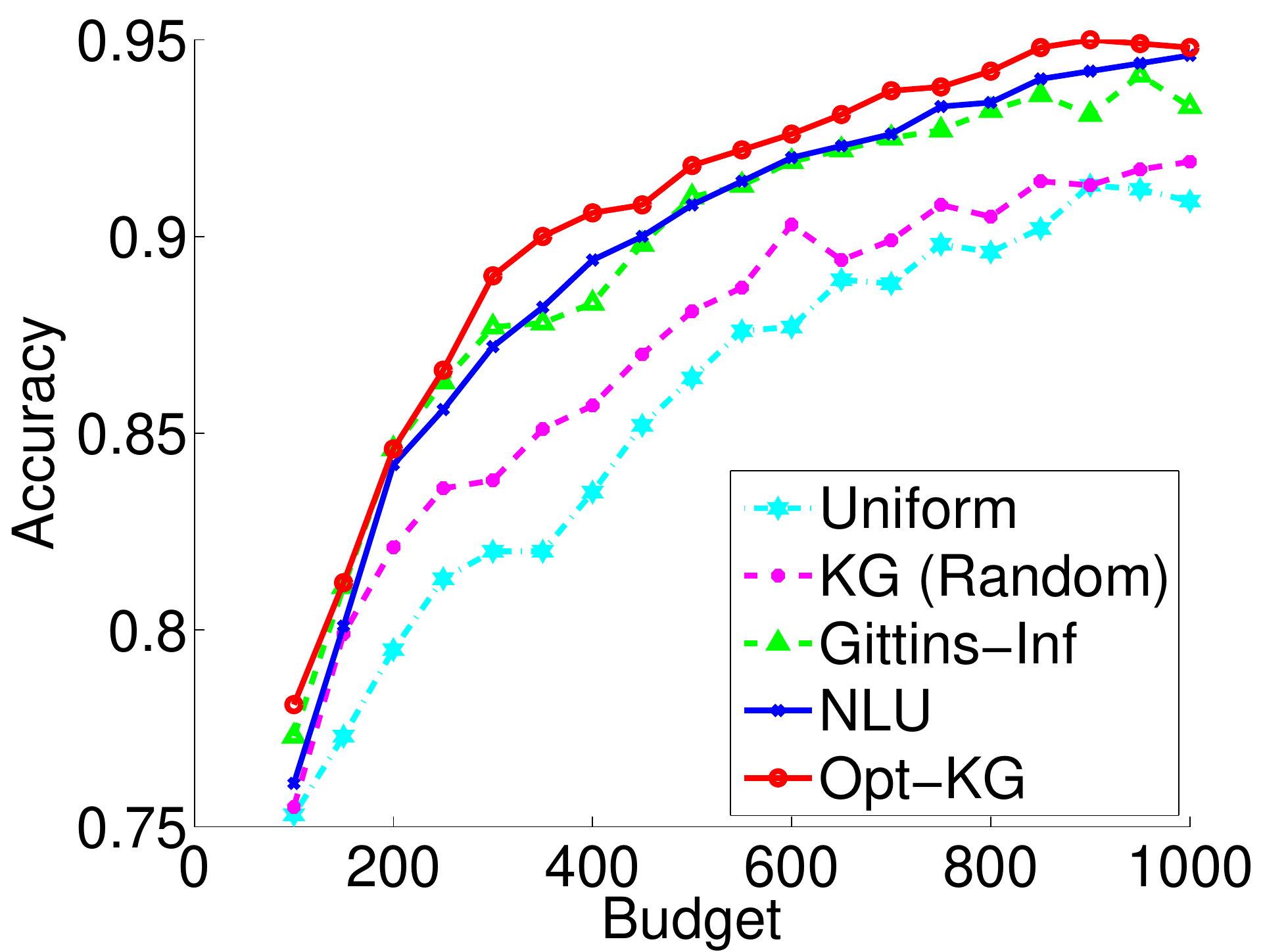}
}\subfigure[$\theta_i \sim \B(0.5,0.5)$ (True Prior)]{
  \includegraphics[width=0.3\textwidth]{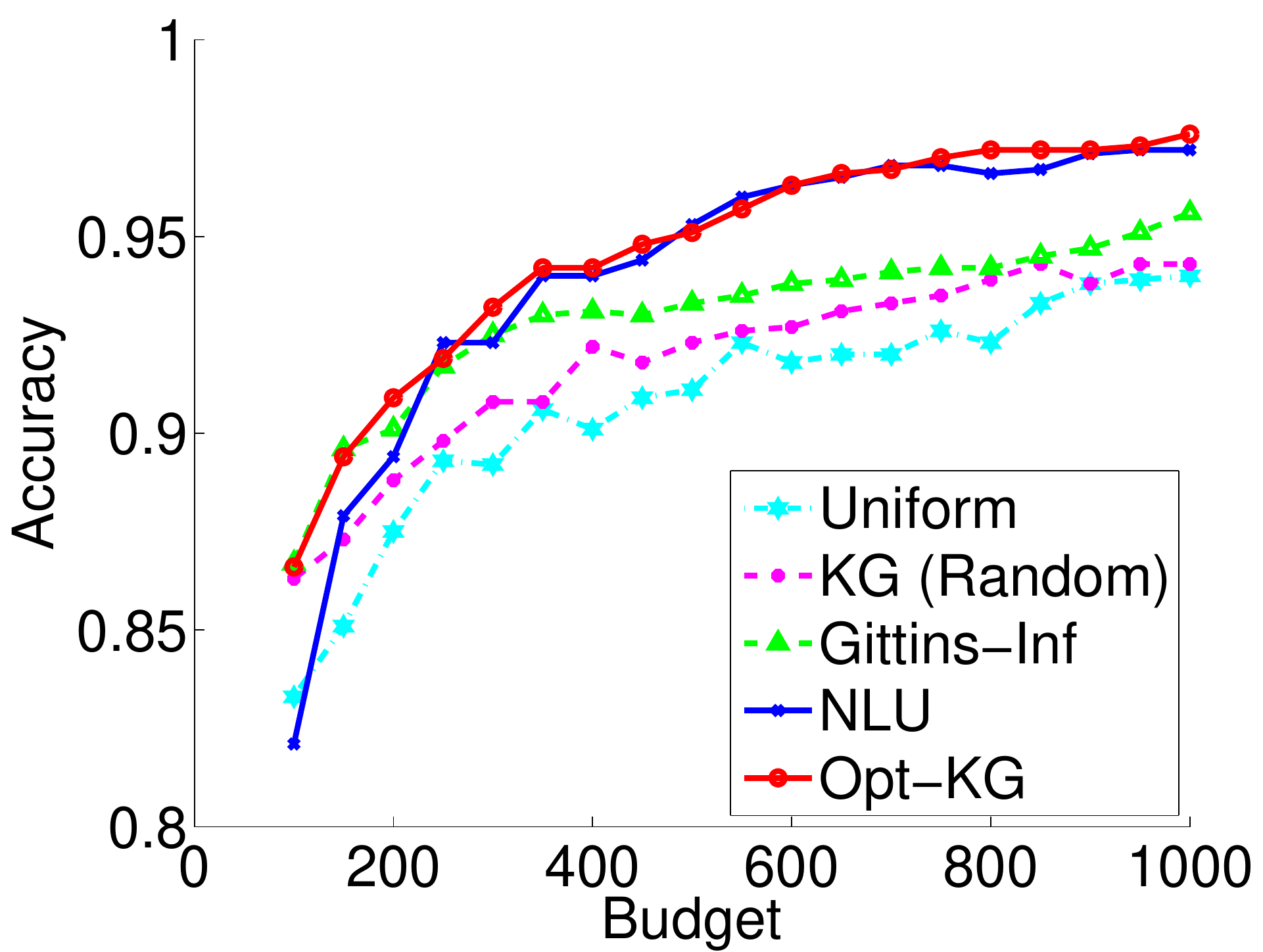}
}\subfigure[$\theta_i \sim \B(0.5,0.5)$ (Uni Prior)]{
  \includegraphics[width=0.3\textwidth]{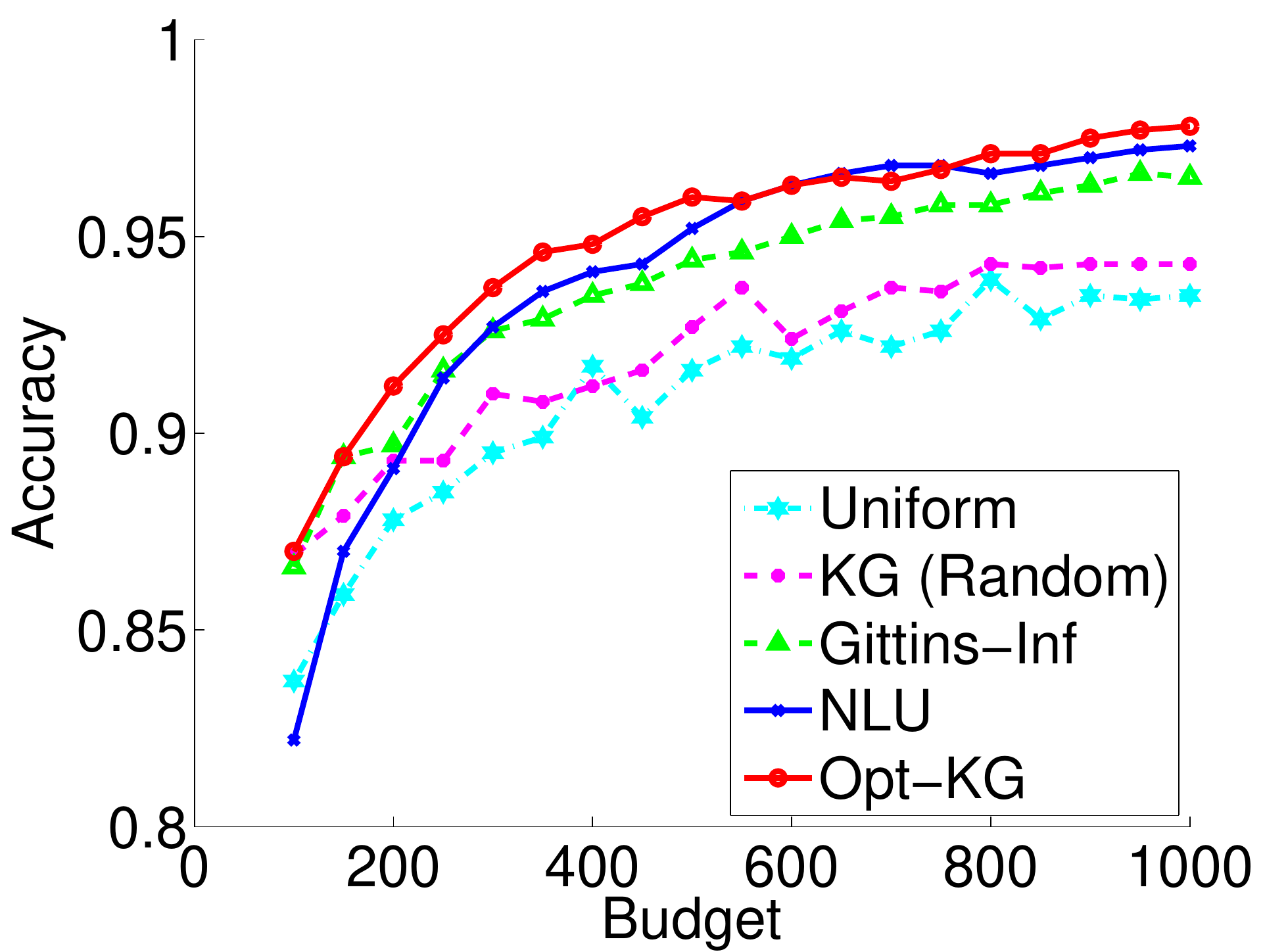}
}\\
\subfigure[$\theta_i \sim \B(2,2)$ (True Prior)]{
  \includegraphics[width=0.32\textwidth]{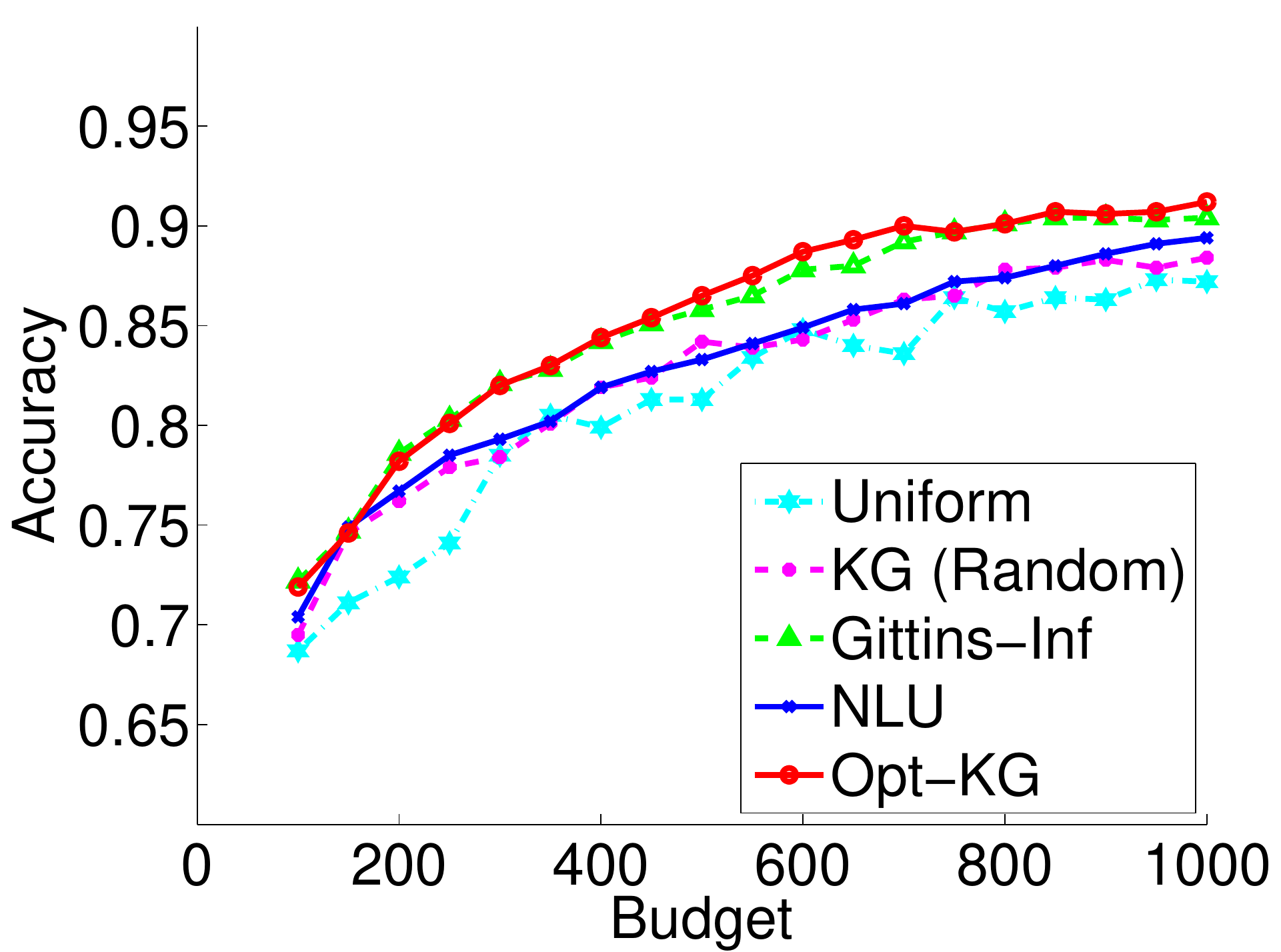}
}\subfigure[$\theta_i \sim \B(2,2)$ (Uni Prior)]{
  \includegraphics[width=0.32\textwidth]{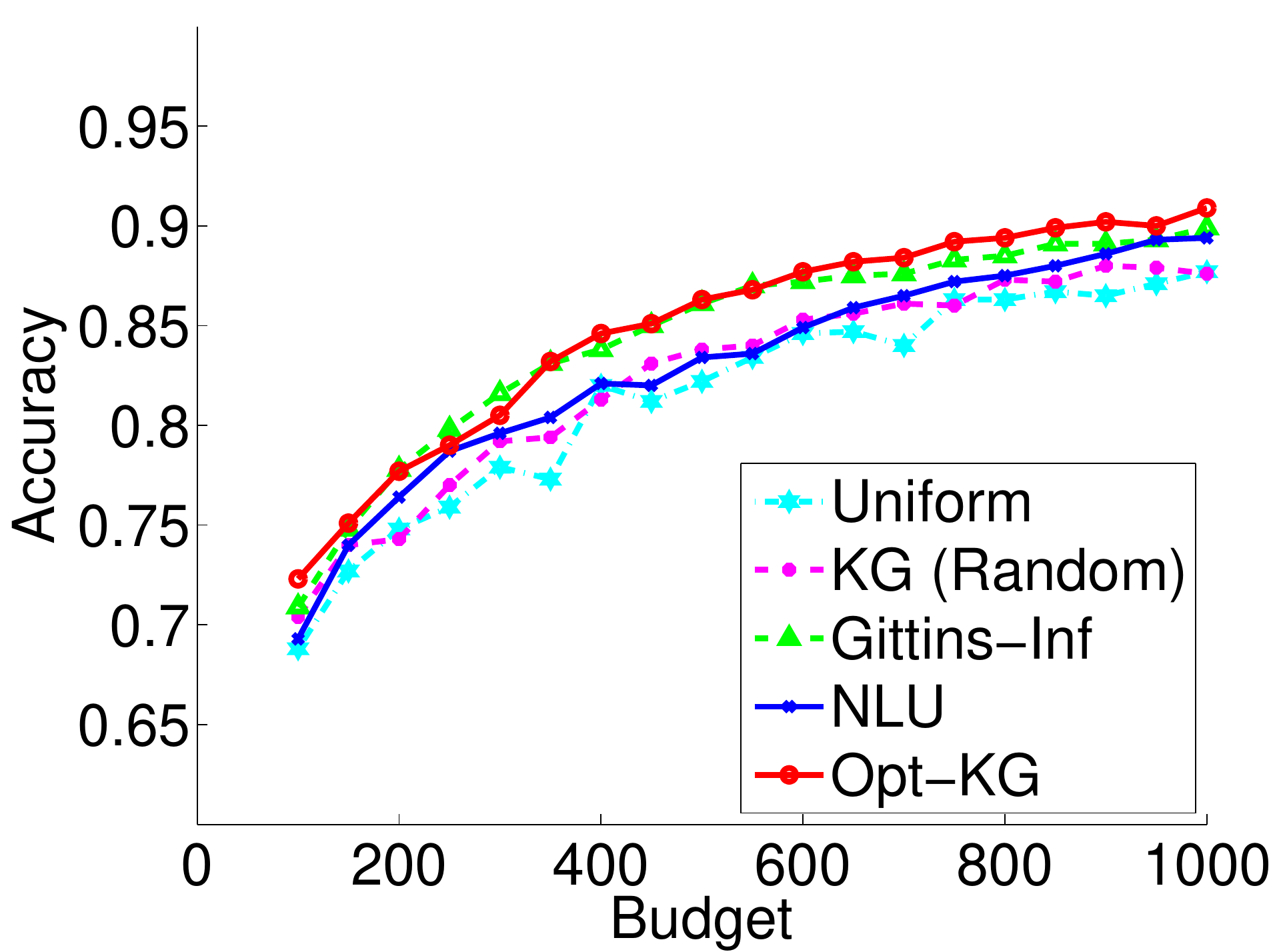}
} \\
\subfigure[$\theta_i \sim \B(2,1)$ (True Prior)]{
  \includegraphics[width=0.32\textwidth]{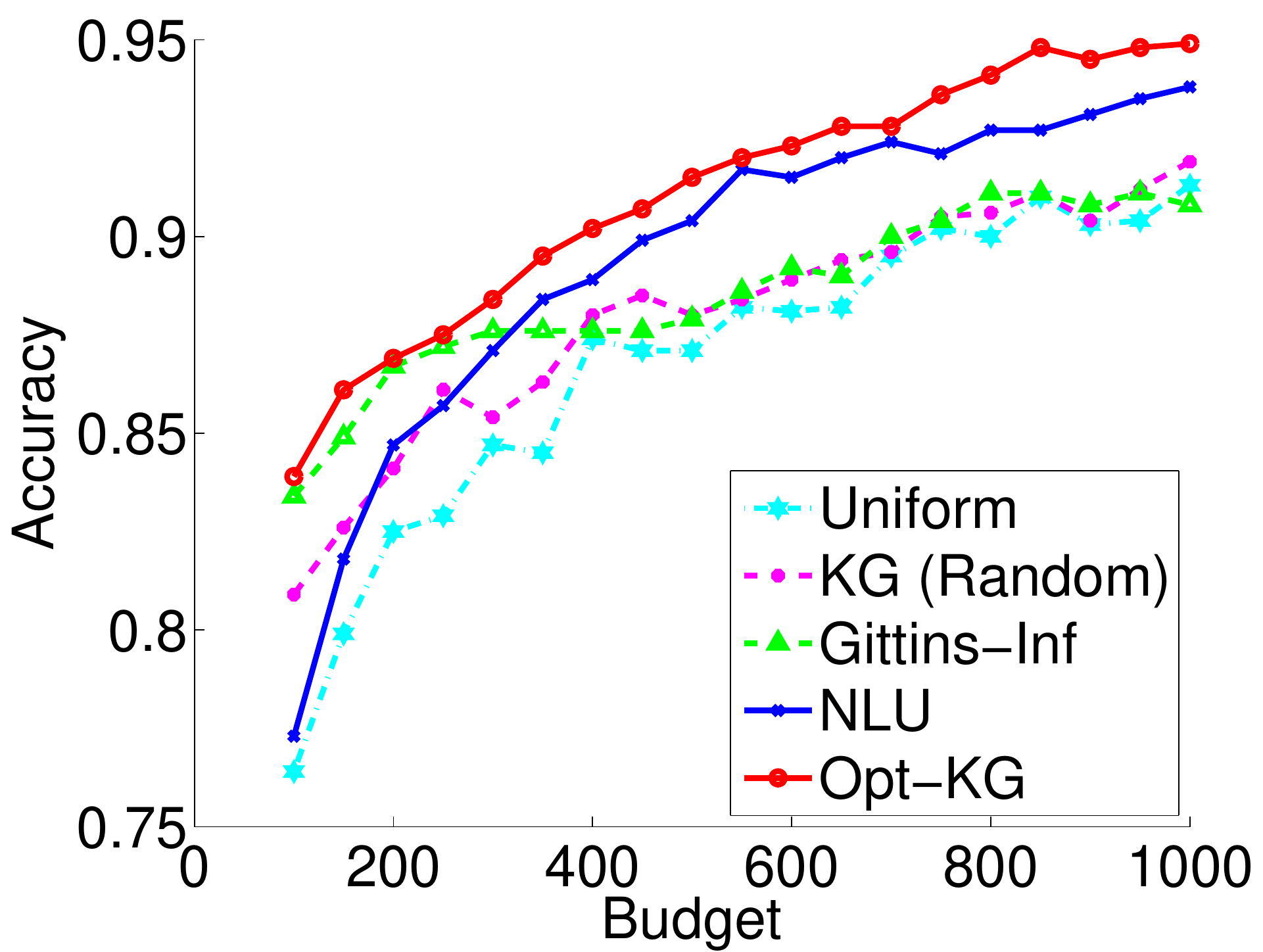}
}
\subfigure[$\theta_i \sim \B(2,1)$ (Uni Prior)]{
  \includegraphics[width=0.32\textwidth]{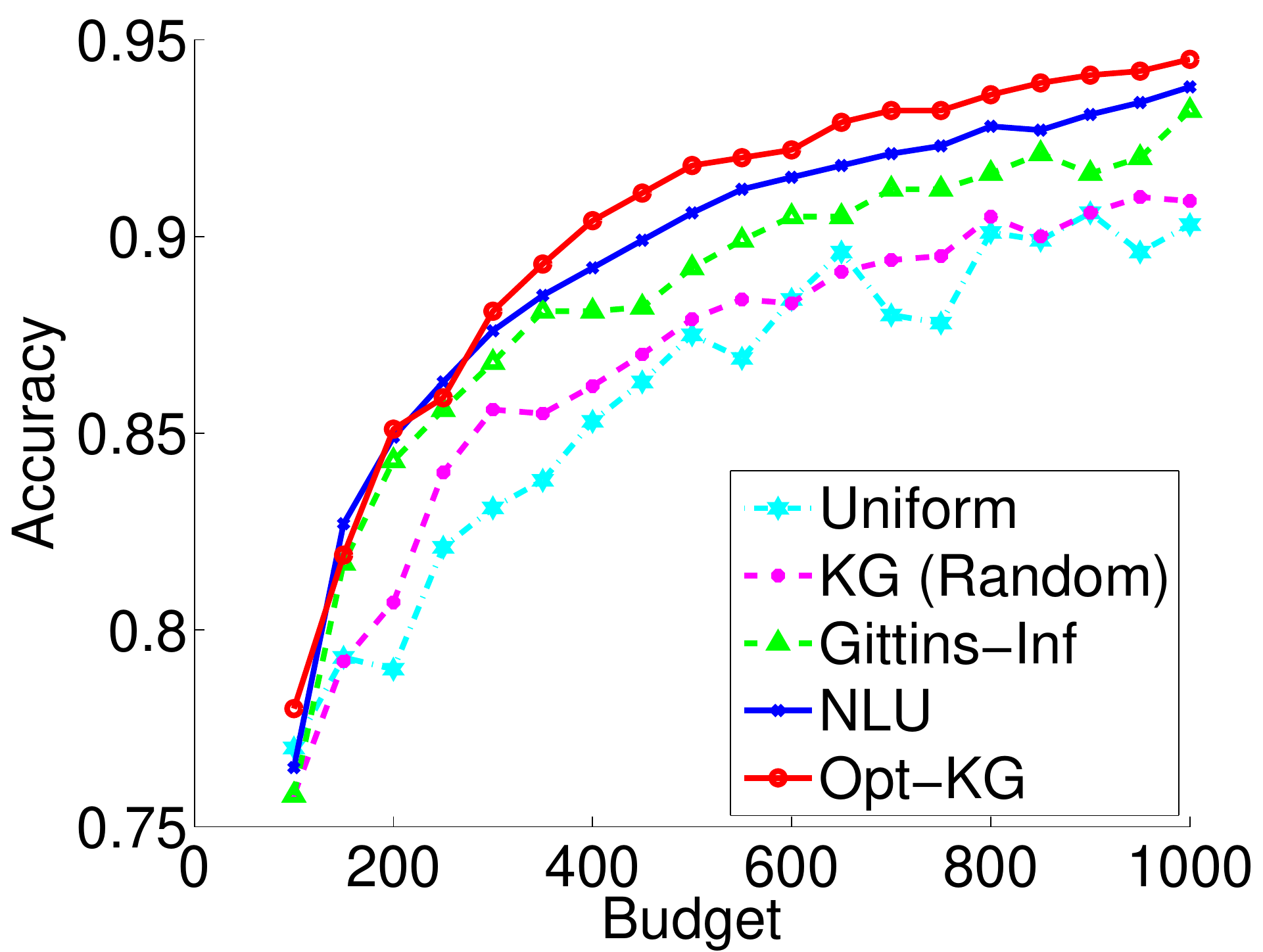}
} \\
\subfigure[$\theta_i \sim \B(4,1)$ (True Prior)]{
  \includegraphics[width=0.32\textwidth]{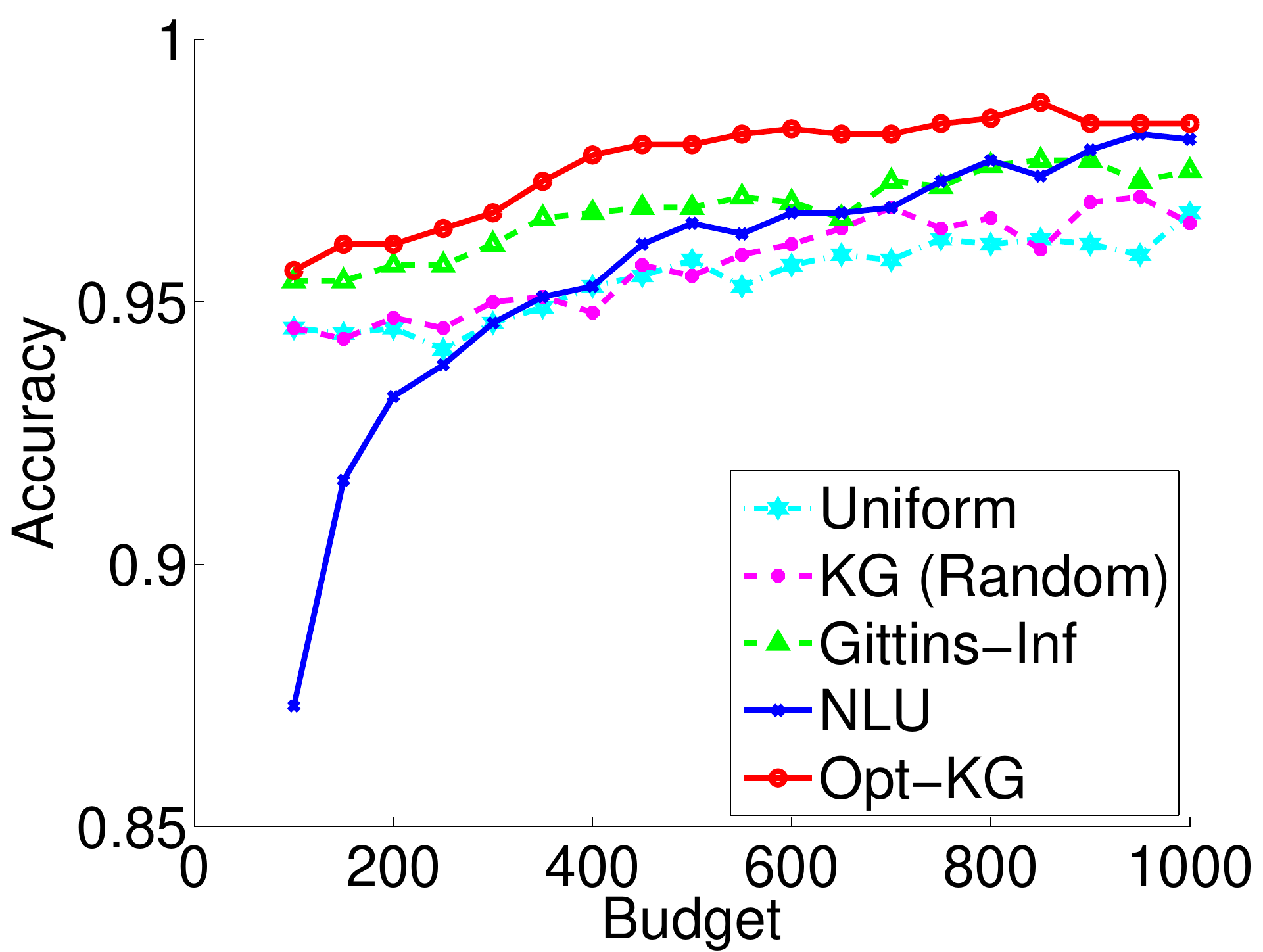}
}\subfigure[$\theta_i \sim \B(4,1)$ (Uni Prior)]{
  \includegraphics[width=0.32\textwidth]{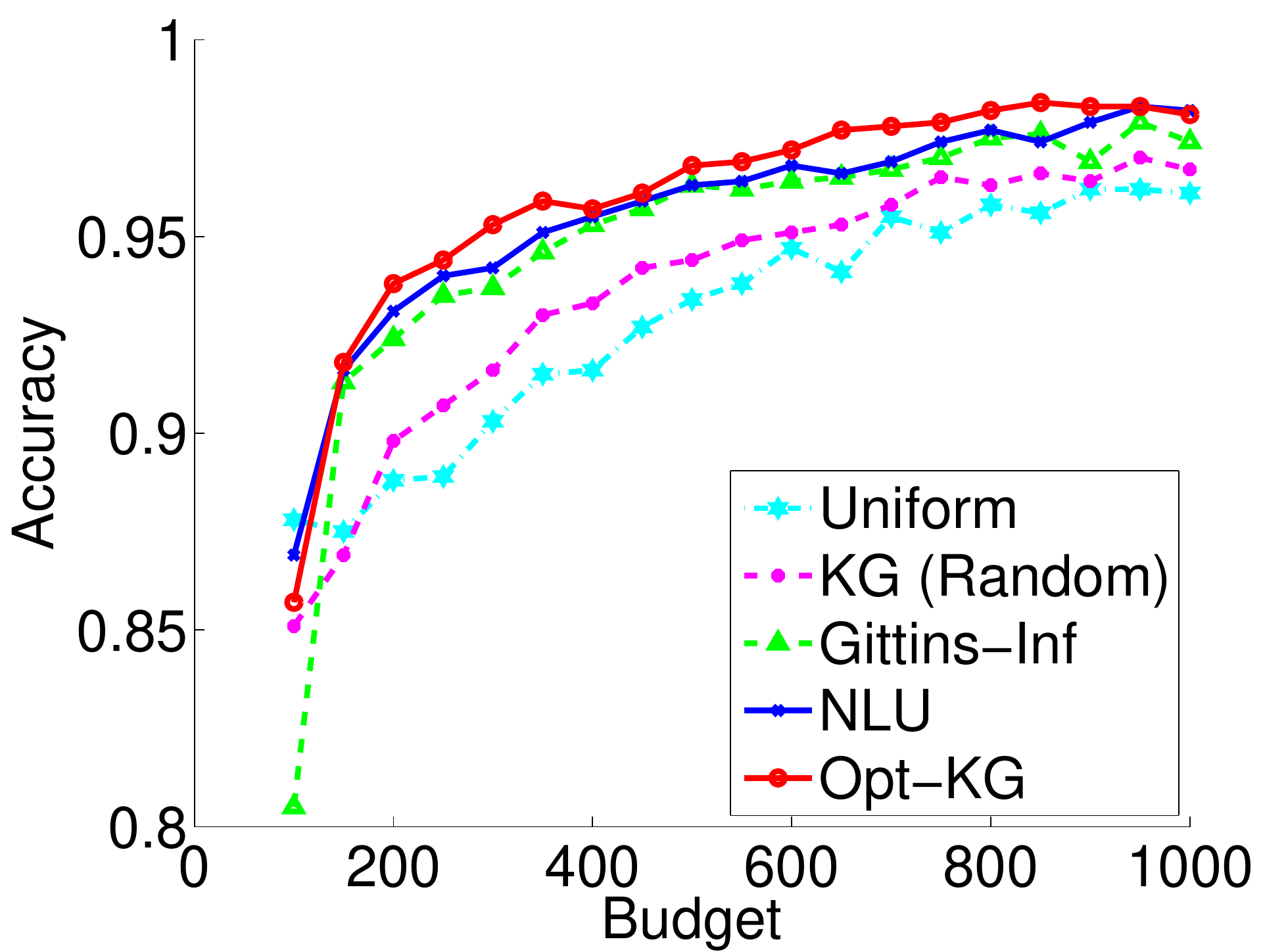}
}
\caption{Performance comparison under the homogeneous noiseless worker setting.}
\label{fig:comp_task}
\end{figure}

We compare the performance of Opt-KG under the homogeneous noiseless worker setting to several other competitors, including
\begin{enumerate}
  \item \textsf{Uniform:} Uniform sampling.
  \item \textsf{KG(Random):} Randomized knowledge gradient \cite{Frazier:08}.
  \item \textsf{Gittins-Inf:}  A Gittins-indexed based policy proposed in \cite{Xie:12} for solving an infinite-horizon Bayesian MAB problem where the reward is discounted by $\delta$. Although it solves a different problem, we apply it as a heuristic by choosing the discount factor $\delta$  such  that $T=1/(1-\delta)$.
  \item \textsf{NLU:} The ``new labeling uncertainty'' method proposed in \cite{Panos:13}.
\end{enumerate}
We note that we do not compare to the finite-horizon Gittins index rule \cite{Mora:11} since its computation is  very expensive. On some small-scale problems, we observe that the finite-horizon Gittins index rule  \cite{Mora:11} has the similar performance as \textsf{Gittins-Inf} in \cite{Xie:12}.

We simulate $K=50$ instances with each $\theta_i \sim \B(1,1)$, $\theta_i \sim \B(0.5, 0.5)$, $\theta_i \sim \B(2,2)$, $\theta_i \sim \B(2,1)$ or $\theta_i \sim \B(4,1)$ (see Figure \ref{fig:beta_task_pdf}). For each of the five settings, we vary the total budget $T=2K, 3K, \ldots, 20K$ and report the mean of accuracy for 20 independently generated sets of $\{\theta_i\}_{i=1}^K$. For the last four settings, we report the comparison among different methods when either using the uniform prior (``uni prior" for short) or the true generating distribution as the prior.  From Figure \ref{fig:comp_task}, the proposed Opt-KG outperforms all the other competitors in most settings regardless the choice of the prior. For $\theta_i \sim  \B(0.5,0.5)$, NLU matches the performance of Opt-KG; and for $\theta_i \sim \B(2,2)$, Gittins-inf matches the performance of Opt-KG.
We also observe that the performance of randomized KG only slightly improves that of uniform sampling.

\subsubsection{Performance comparison under the heterogeneous worker setting}

We compare the proposed Opt-KG under the heterogeneous worker setting to several other competitors:

\makeatletter
\newcommand\mynobreakpar{\par\nobreak\@afterheading}
\makeatother
\newenvironment{myenumerate}{\mynobreakpar\begin{enumerate}}{\end{enumerate}}

\begin{enumerate}
  \item \textsf{Uniform:} Uniform sampling.
  \item \textsf{KG(Random):} Randomized knowledge gradient \cite{Frazier:08}.
  \item \textsf{KOS:} The randomized budget allocation algorithm in \cite{Oh:12}.
\end{enumerate}

We note that several competitors for the homogeneous worker setting (e.g., Gittins-inf and NLU) cannot be directly applied to the heterogeneous worker setting since they fail to model each worker's reliability.

\begin{figure}[!t]
\centering
\subfigure[$\rho_j \sim \B(4,1)$ (True Prior)]{
  \includegraphics[width=0.31\textwidth]{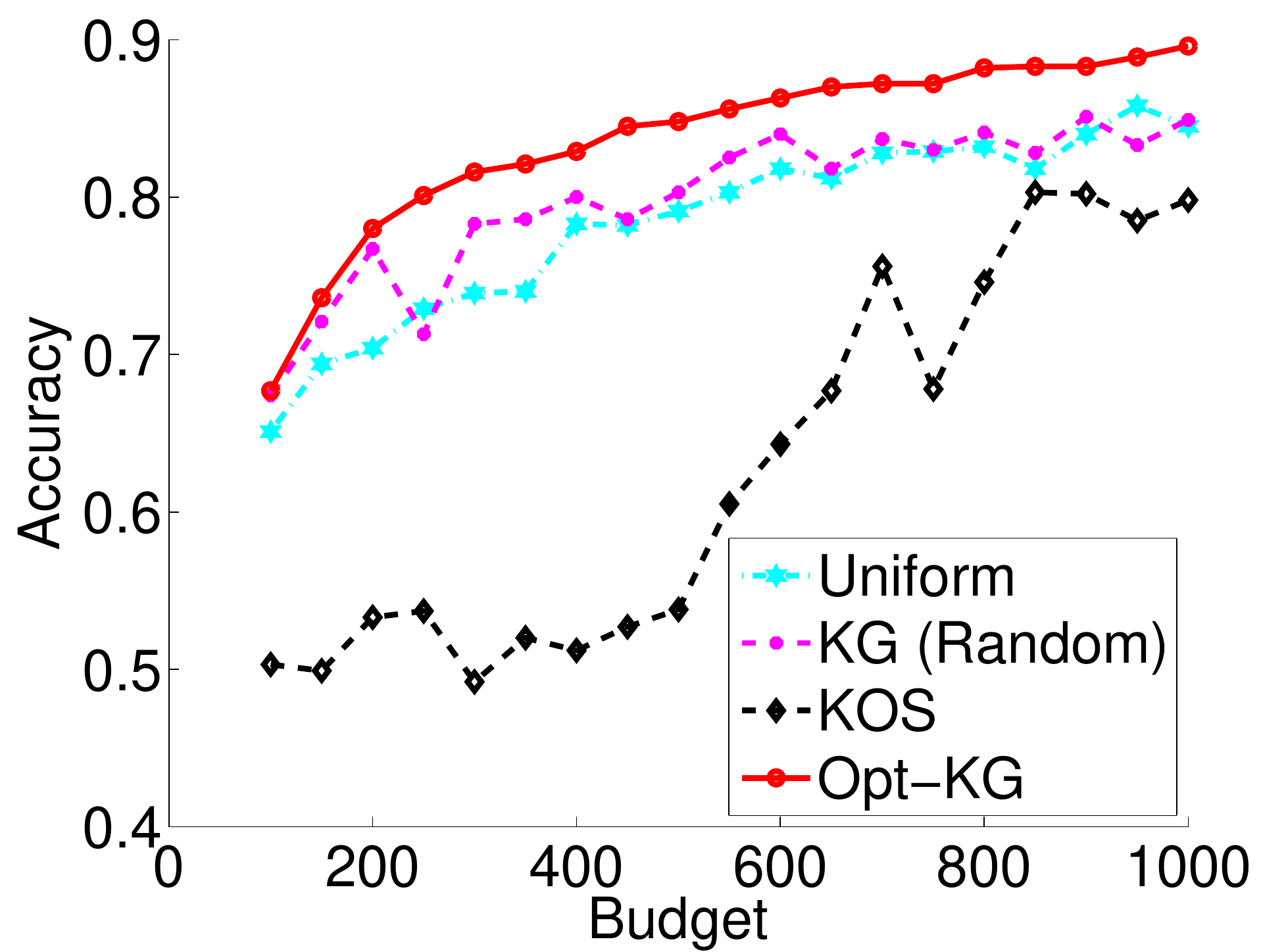}
}\subfigure[$\rho_j \sim \B(3,1)$ (True Prior)]{
  \includegraphics[width=0.31\textwidth]{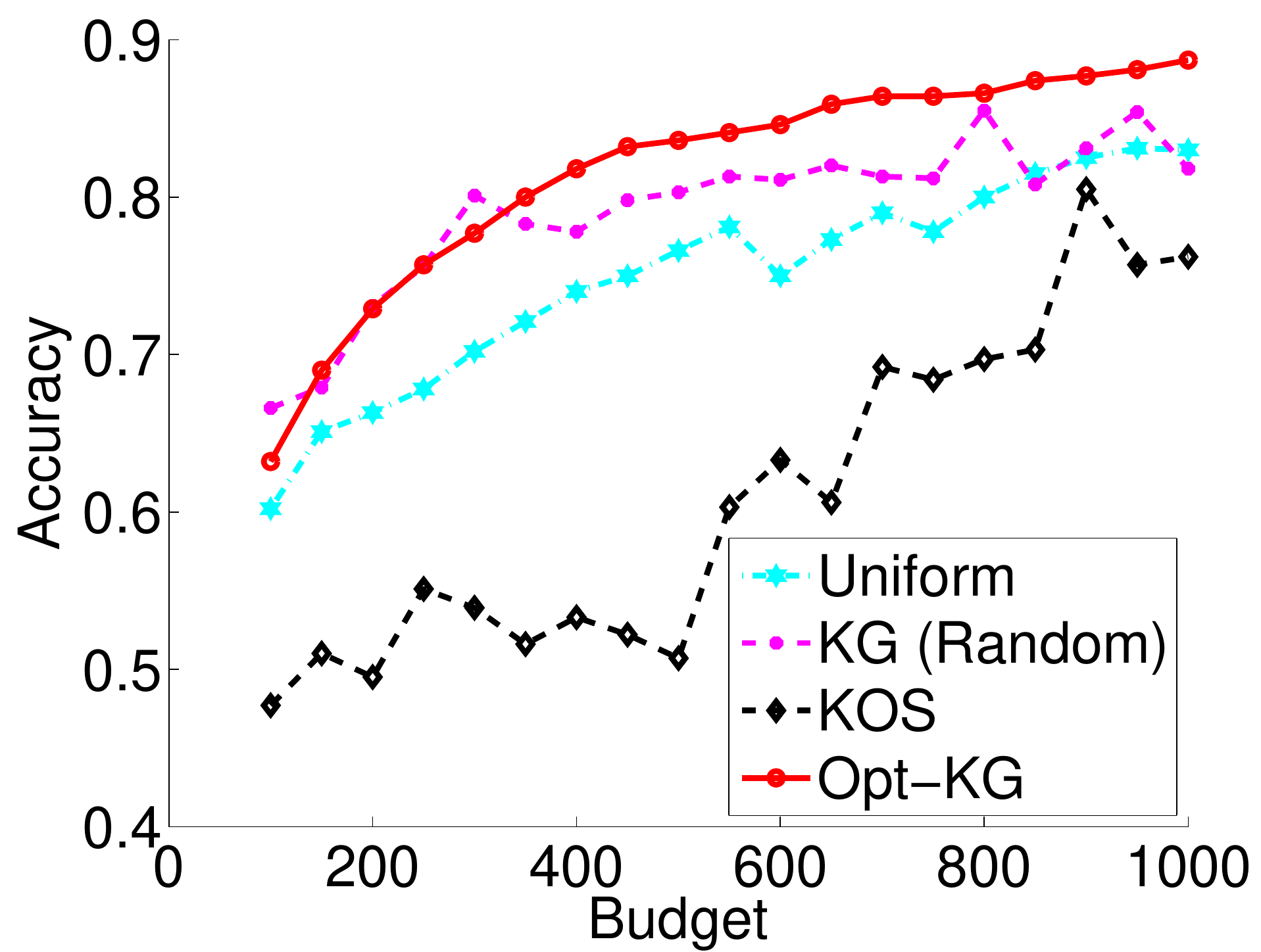}
}\subfigure[$\rho_j \sim \B(3,1)$ ($\B(4,1)$ Prior)]{
  \includegraphics[width=0.31\textwidth]{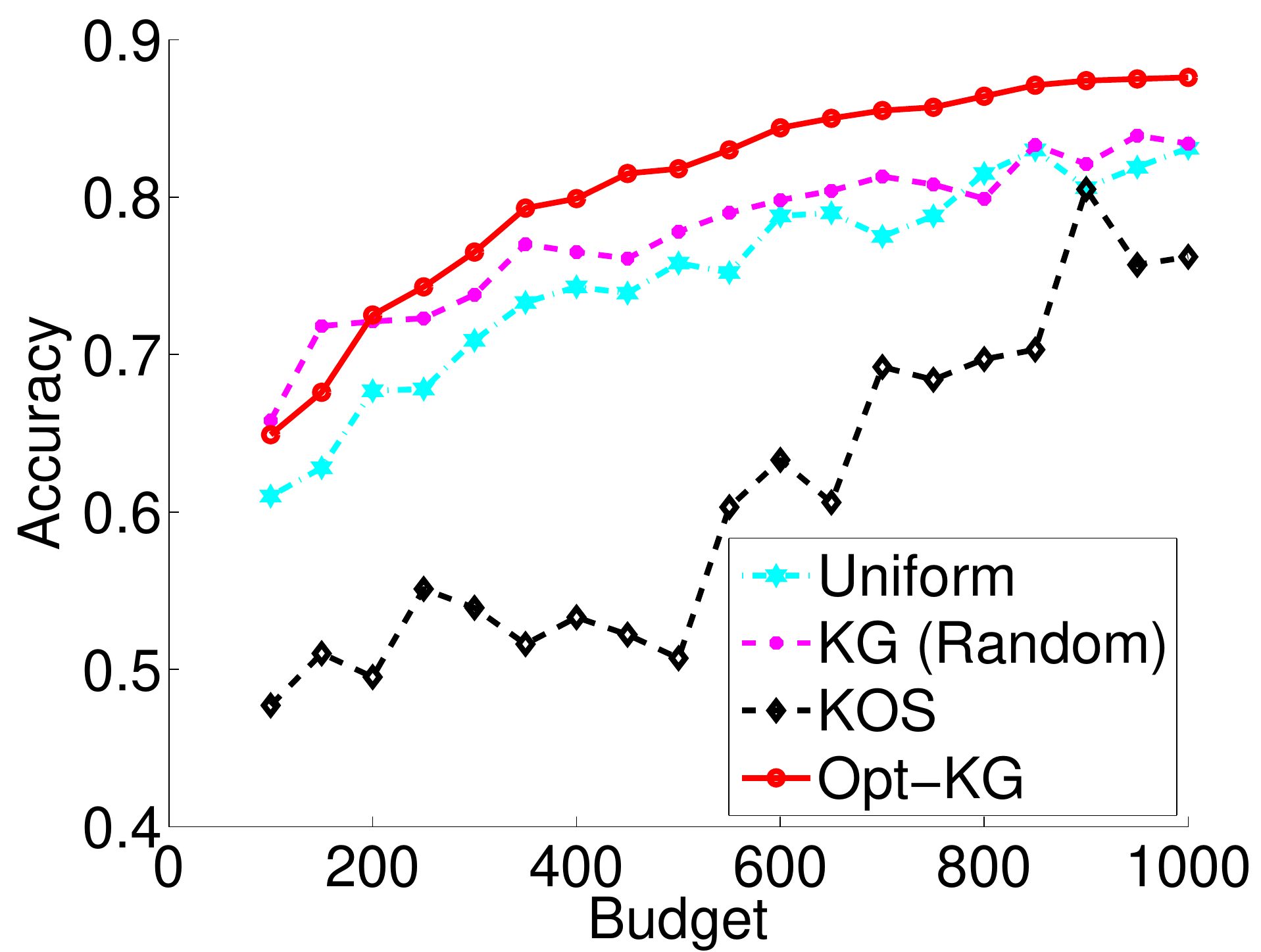}
}\\
\subfigure[$\rho_j \sim \B(8,1)$ (True Prior)]{
  \includegraphics[width=0.35\textwidth]{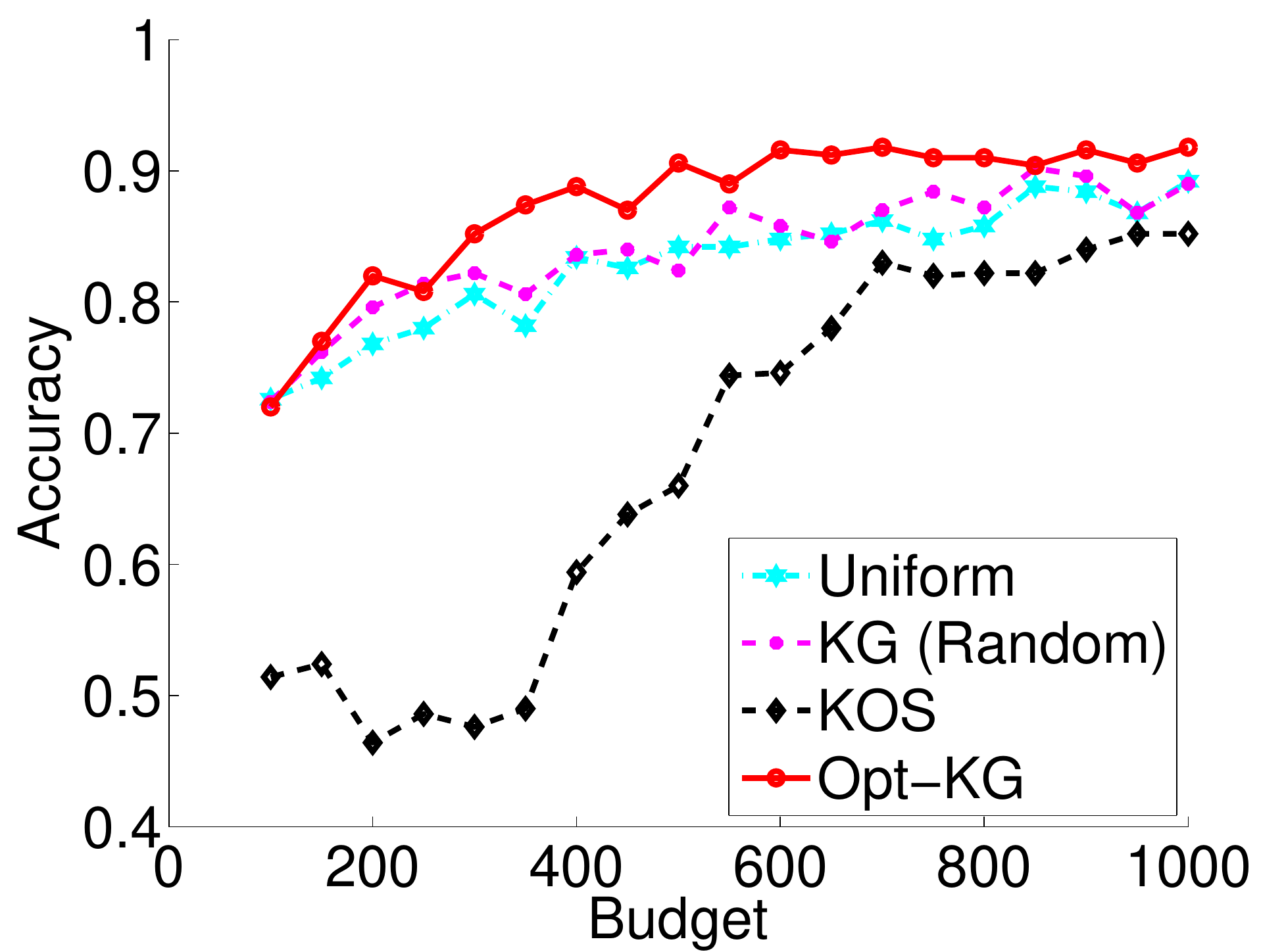}
}\subfigure[$\rho_j \sim \B(8,1)$ ($\B(4,1)$ Prior)]{
  \includegraphics[width=0.35\textwidth]{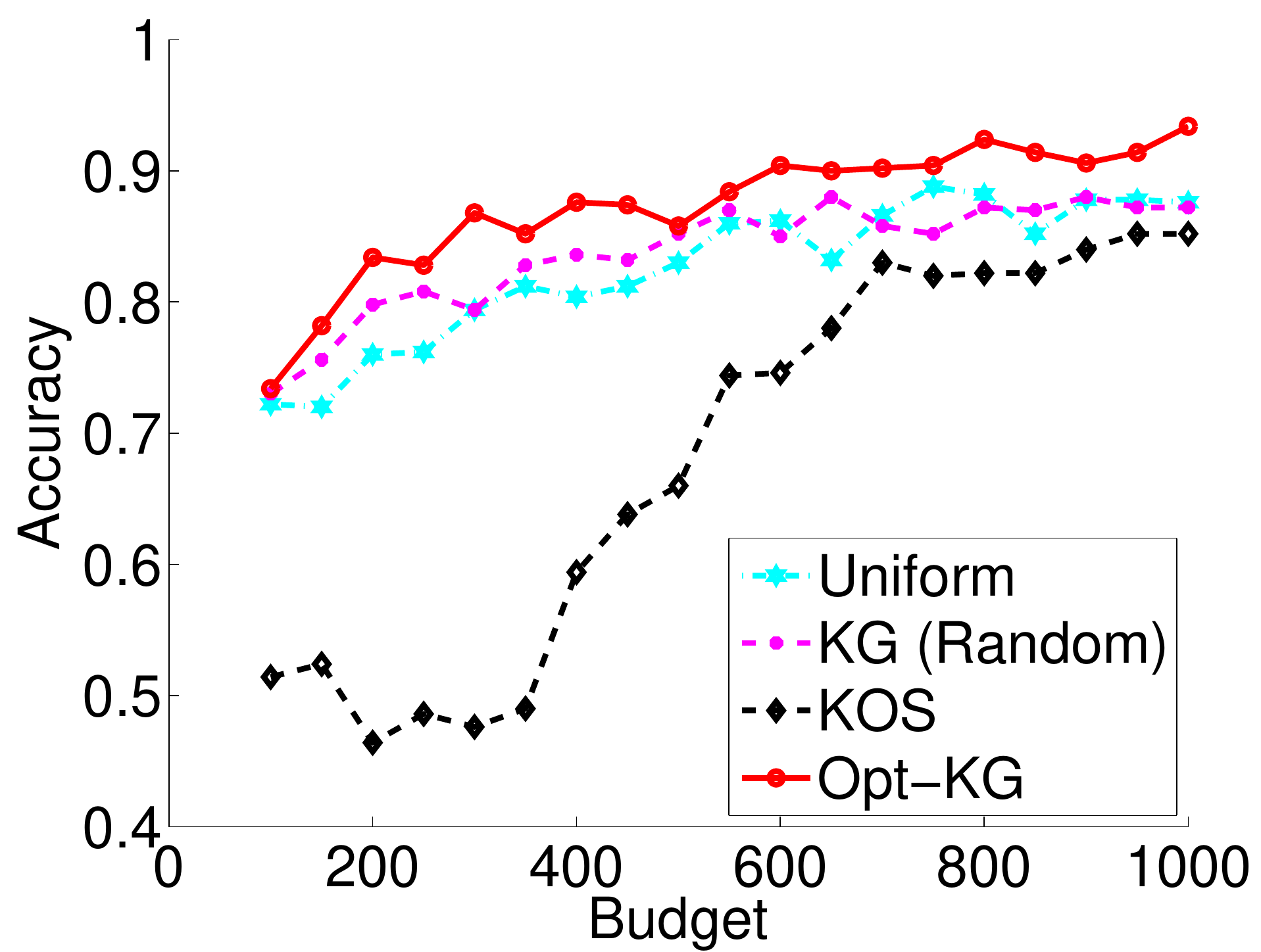}
}\\
\subfigure[$\rho_j \sim \B(5,2)$ (True Prior)]{
  \includegraphics[width=0.35\textwidth]{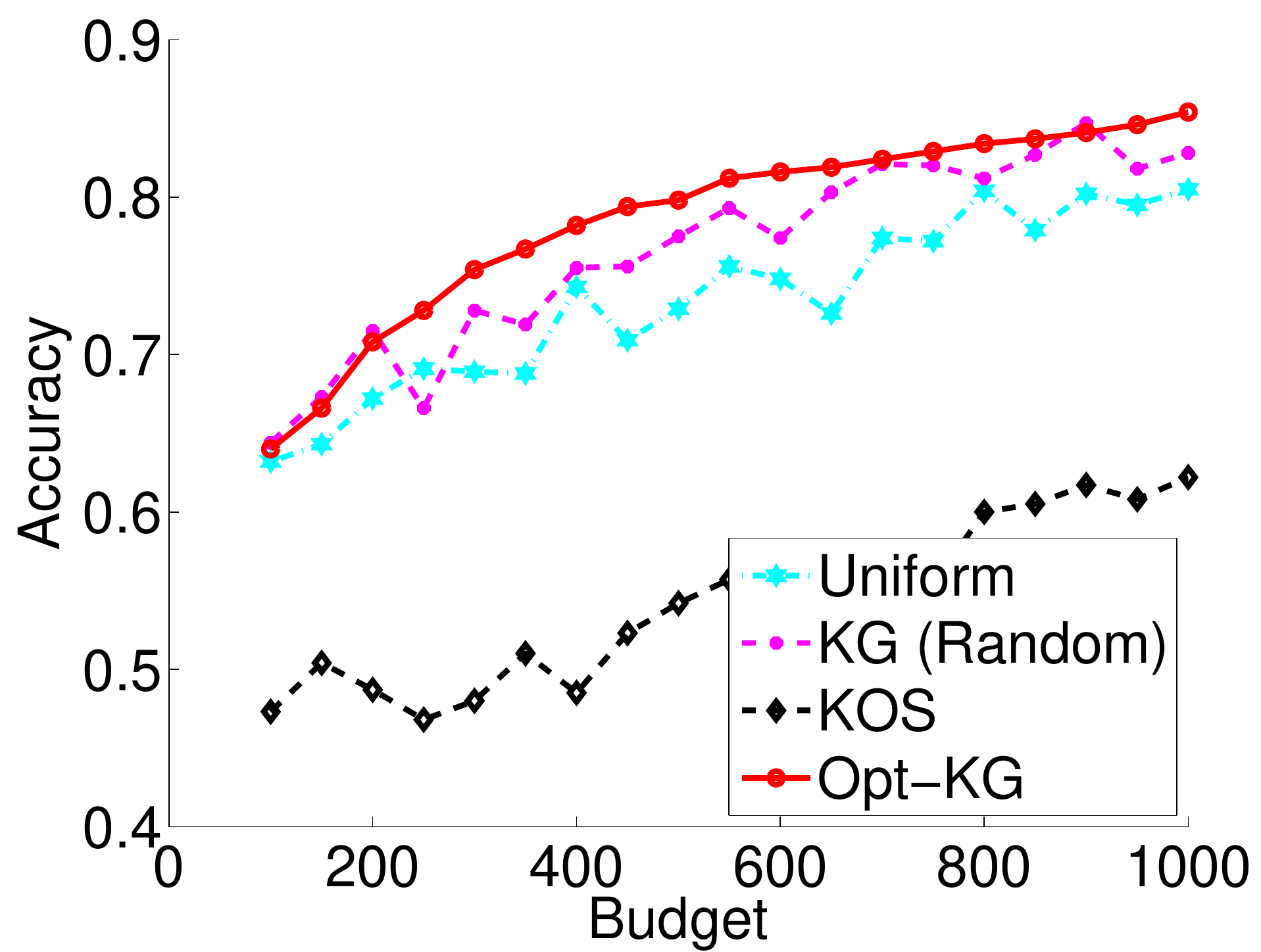}
}\subfigure[$\rho_j \sim \B(5,2)$ ($\B(4,1)$ Prior)]{
  \includegraphics[width=0.35\textwidth]{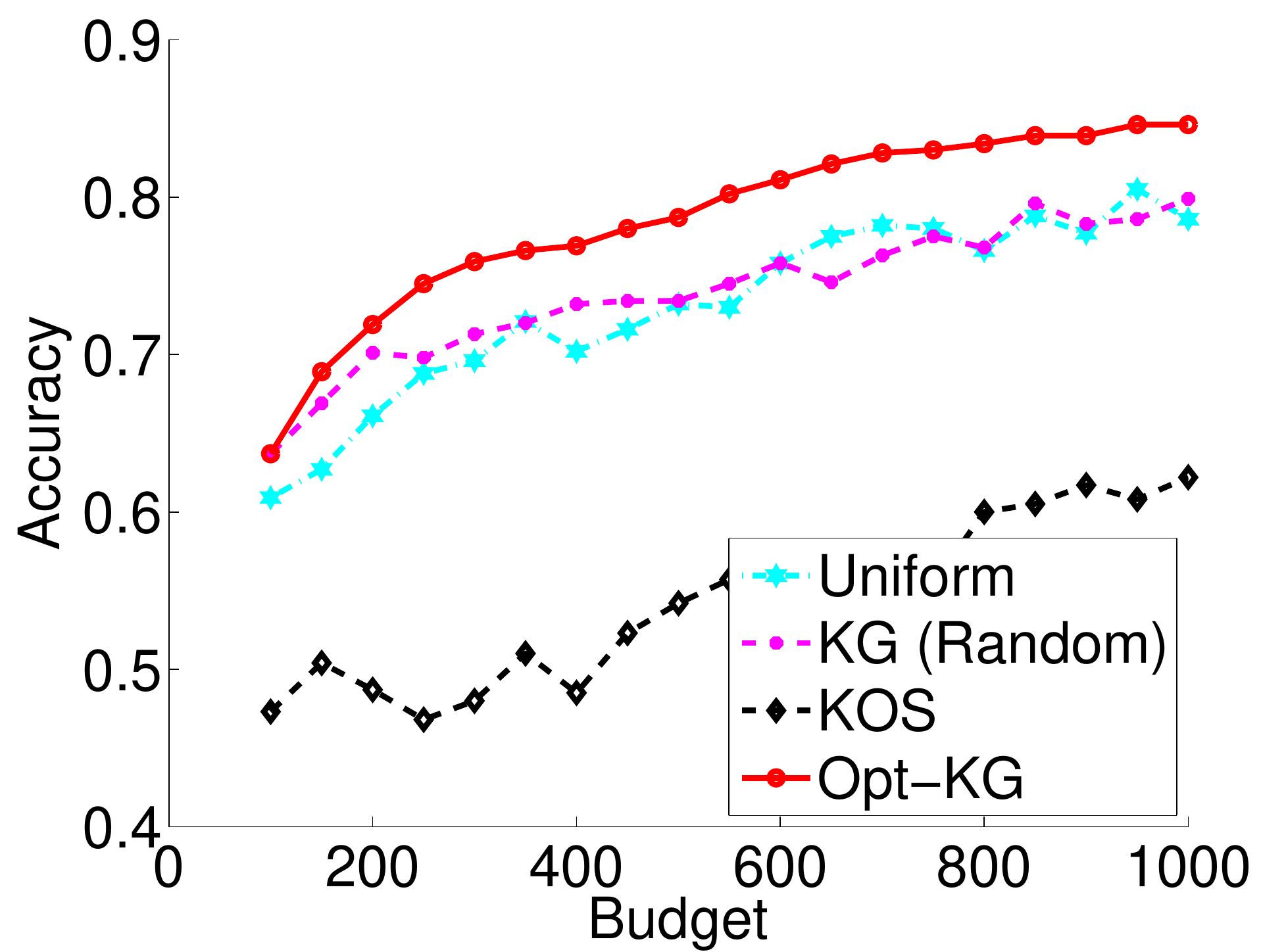}
}
\caption{Performance comparison under the heterogeneous worker setting.}
\label{fig:comp_worker}
\end{figure}

We simulate $K=50$ instances with each $\theta_i \sim \B(1,1)$ and $M=100$ workers with $\rho_j \sim \B(4,1)$, $\rho_j \sim \B(3,1)$,  $\rho_j \sim \B(8,1)$ or  $\rho_j \sim \B(5,2)$ (see Figure \ref{fig:beta_worker_pdf}). For each of the four settings, we vary the total budget $T=2K, 3K, \ldots, 20K$ and report the mean of accuracy for 20 independently generated sets of parameters. For the last three settings, we report the comparison among different methods when either using $\B(4,1)$ prior or the true generating distribution for $\rho_j$ as the prior. From Figure \ref{fig:comp_worker}, the proposed Opt-KG outperforms all the other competitors regardless the choice of the prior.
\subsection{Real Data}

We compare different policies on a standard real dataset for recognizing textual entailment (RTE) (Section 4.3 in \cite{Snow:08}). There are 800 instances and each instance is a sentence pair. Each sentence pair is presented to 10 different workers to acquire binary choices of whether the second hypothesis sentence can be inferred from the first one.   There are in total 164 different workers. We first consider the homogeneous noiseless setting without incorporating the diversity of workers and  use  the uniform prior ($\B(1,1)$) for each $\theta_i$. In such a setting, once we decide to label an instance, we randomly choose a worker (who provides the label in the full dataset) to acquire the label. Due to this randomness, we run each policy 20 times and report the mean of the accuracy  in Figure \ref{fig:rte}.
As we can see, Opt-KG, Gittins-inf and NLU all perform quite well.  We also note that although Gittins-inf performs slightly better than our method on this data,   it requires solving a linear system with $O(T^2)$ variables at each stage, which could be too expensive for large-scale applications. While our Opt-KG policy has a time complexity linear in $KT$ and space complexity linear in $K$, which is much more efficient when a quick
online decision is required. In particular, we present the comparison between Opt-KG and Gittins-inf on the averaged CPU time under different budget levels in Table \ref{tab:comp_time}. As one can see,  Gittins-inf is computationally more expensive than Opt-KG.

\begin{figure}[!t]
\centering
\subfigure[RTE: Homogeneous Noiseless Worker]{
  \includegraphics[width=0.45\textwidth]{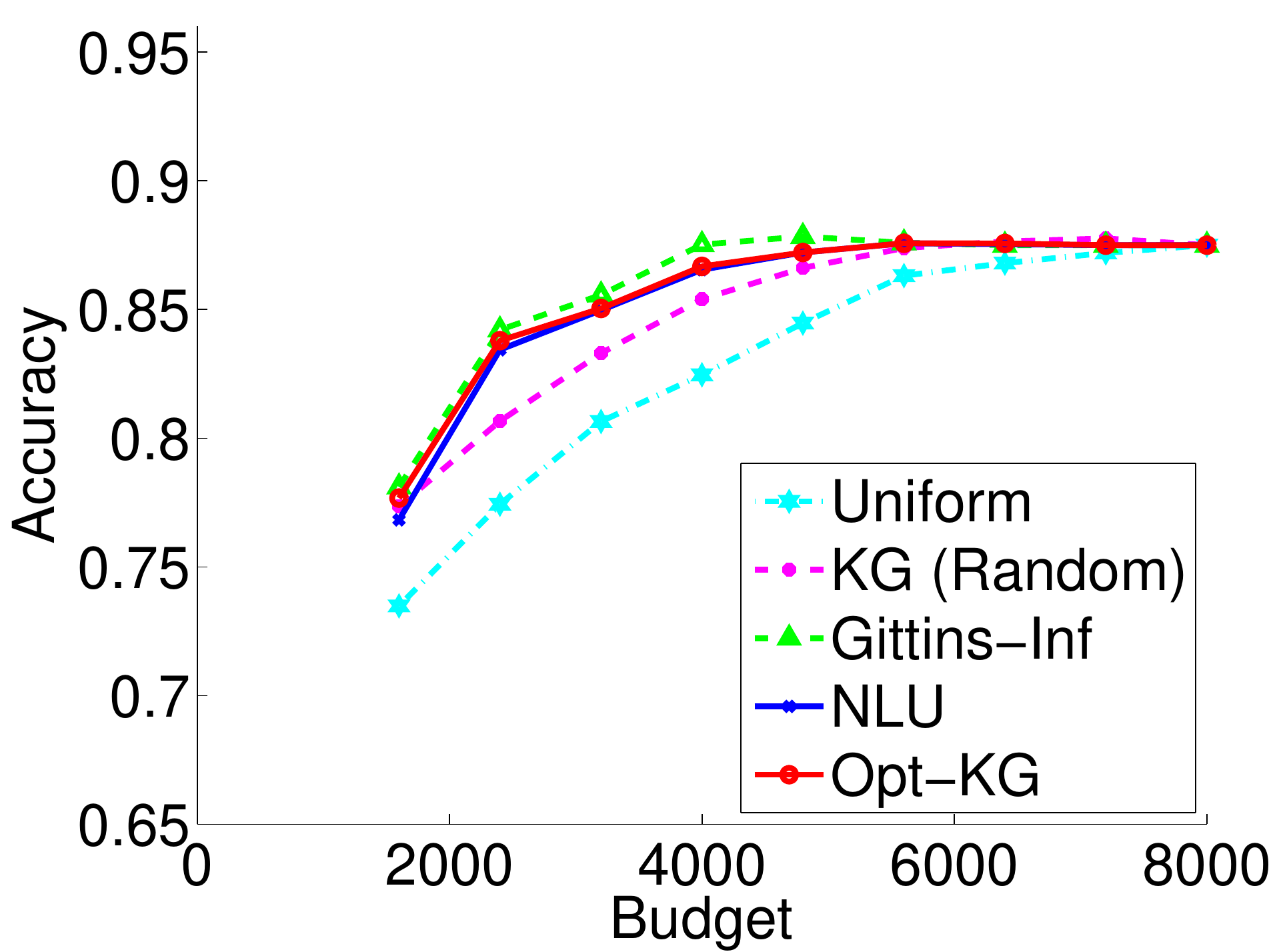}
	    \label{fig:rte}
} \hspace{0.2cm}
\subfigure[RTE: Heterogeneous Worker]{
  \includegraphics[width=0.45\textwidth]{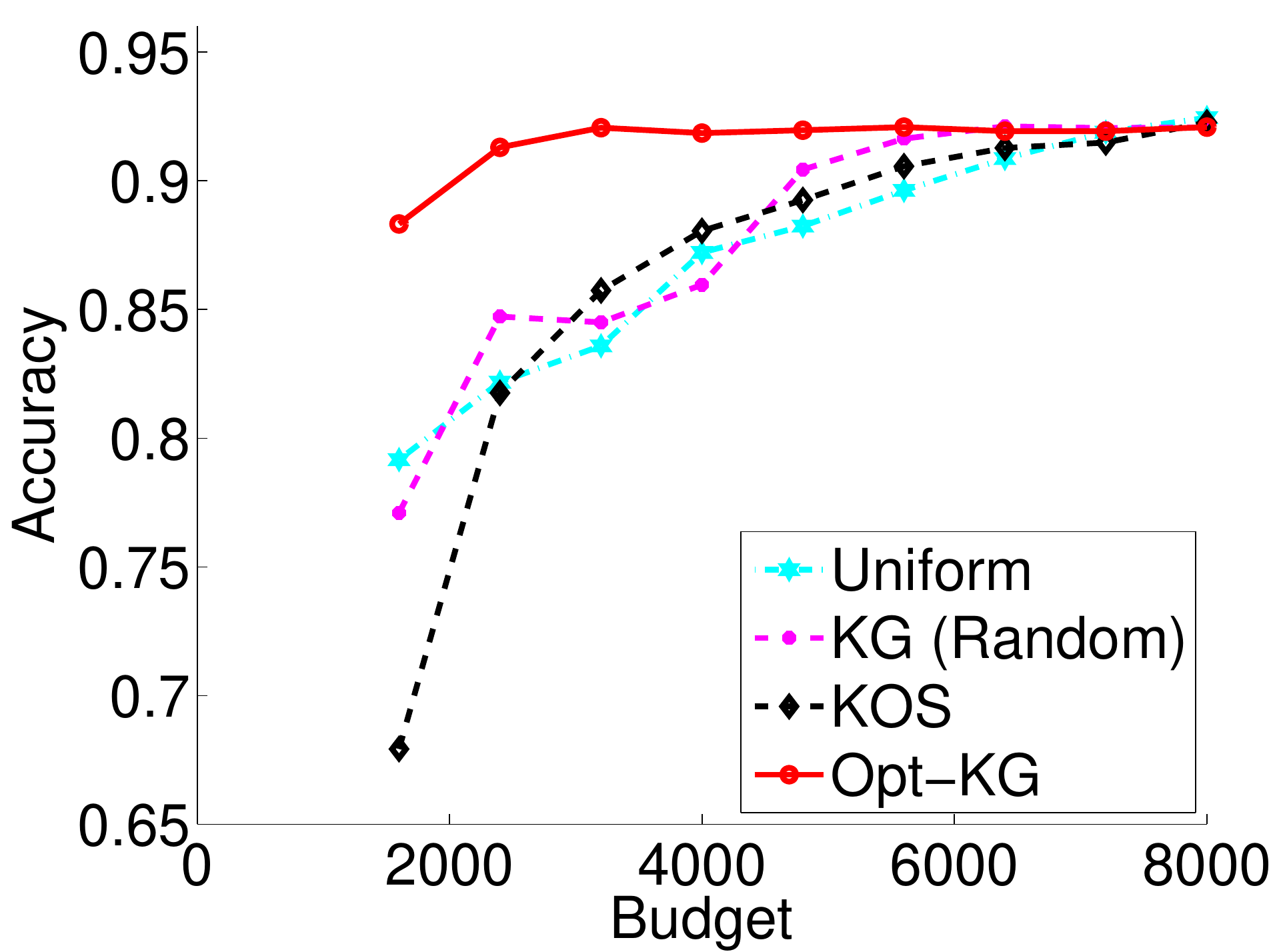}
	    \label{fig:rte_worker}
}
\caption{ Performance comparison on the real dataset.}
\label{fig:rte_full}
\end{figure}

\begin{table}[!t]
\centering
 \caption{Comparison in CPU time (seconds)}
  \begin{tabular}{|c|c|c|c|c|c|} \hline
    Budget $T$ & $2K=1,600$ & $4K=3,200$ & $6K=4,800$ & $10K=8,000$ \\ \hline
    Opt-KG & 1.09 &  2.19 & 3.29 &   5.48 \\ \hline
    Gittins-inf &  25.87  & 35.70 & 45.59 & 130.68 \\ \hline
 \end{tabular}
\label{tab:comp_time}
\end{table}

When the worker reliability is incorporated, we compare different policies in Figure \ref{fig:rte_worker}. We put $\B(4,1)$ prior distribution for each $\rho_j$ which indicates that we have the prior belief that most workers perform reasonably well. Other priors in Figure  \ref{fig:beta_worker_pdf} lead to similar results and thus omitted here. As one can see, the accuracy of Opt-KG is much higher than that of other policies when $T$ is small. It achieves the highest accuracy of $92.05\%$ only using 40\% of the total budget (i.e., on average, each instance is labeled 4 times).  One may also observe that when $T > 4K=3,200$, the performance of Opt-KG does not improve and in fact, slightly downgrades a little bit.   This is mainly due to the restrictiveness of the experimental setting. In particular, since the experiment is conducted on a fixed dataset with partially observed labels, the Opt-KG cannot freely choose instance-worker pairs especially when the budget goes up (i.e., the action set is greatly restricted). According to our experience, such a phenomenon will not happen on experiments when labels can be obtained from any instance-worker pair. Comparing Figure \ref{fig:rte_worker} to \ref{fig:rte}, we also observe that Opt-KG under the heterogeneous worker setting performs much better than  Opt-KG under the homogeneous worker setting, which indicates that it is beneficial to incorporate workers' reliability.

\section{Conclusions and Future Works}
\label{sec:conclusion}

In this paper, we propose to address the problem of budget allocation in crowd labeling. We model the problem using the Bayesian Markov decision process and  characterize the optimal policy using the dynamic programming. We further propose a computationally more attractive approximate policy: optimistic knowledge gradient. Our MDP formulation is a general framework, which can be applied to  binary or multi-class,  contextual or non-contextual crowd labeling problems in either pull or push crowdsourcing marketplaces.

There are several possible future directions for this work. First, it is of great interest to show the consistency of Opt-KG in heterogonous worker setting and further provide the theoretical results on the performance of Opt-KG under finite budget. Second, in this work, we assume that both instances and  workers are equally priced. Although this assumption is standard in many crowd labeling applications, a dynamic pricing strategy as the allocation process proceeds will better motivate those more reliable workers to label more challenge instances. A recent work in \cite{Wang:13} provides some quality-based pricing algorithms for crowd workers and it will be interesting to incorporate their strategies into our dynamic allocation framework.  Third, we assume that the labels provided by the same worker to different instances are independent. It is more interesting to consider that the workers' reliability will be improved during the labeling process when some useful feedback can be provided. Further, since the proposed Opt-KG is a fairly general approximate policy for MDP, it is also interesting to apply it to other statistical decision problems.
\section{Acknowledgement}
We would like to thank Qiang Liu for sharing the code for KOS method; Jing Xie and Peter Frazier for sharing their code for computing infinite-horizon Gittins index; John Platt, Chris J.C. Burges and Kevin P. Murphy for helpful discussions; and anonymous reviewers and the associate editor for their constructive comments on improving the quality of the paper.

\newpage

\appendix

\section*{Appendix}

\section*{Proof of  Proposition \ref{prop:H}}

The final positive set $H_T$ is chosen to maximize the expected accuracy conditioned on $\calF_T$:
\begin{equation}
H_T=\argmax_H  \E\left( \sum_{i=1}^K \left( \mathbf{1}(i \in H) \mathbf{1}(i \in H^*) +  \mathbf{1}(i \not \in H) \mathbf{1}(i \not \in H^*) \right) \Bigg| \calF_T \right)
  \label{eq:exp_acc}
\end{equation}


According to the definition \eqref{Pti} of $P^T_i$, we can re-write \eqref{eq:exp_acc} using the linearity of the expectation:

\begin{eqnarray}
 & & \sum_{i=1}^K \left( \mathbf{1}(i \in H) \Pr(i \in H^* | \calF_T ) + \mathbf{1}(i \not \in H) \Pr(i \not \in H^* | \calF_T) \right)  \nonumber \\
 =&&\sum_{i=1}^K \left( \mathbf{1}(i \in H) P_i^T +  \mathbf{1}(i \not \in H) (1-P_i^T) \right)
\label{eq:exp_acc_1}
\end{eqnarray}

To maximize \eqref{eq:exp_acc_1} over $H$, it easy to see that we should set $i \in H$ if and only if $P_i^T \geq 0.5$. Therefore, we have the  positive set $$H_T = \{i: P_i^T \geq 0.5\}.$$

\section*{Proof of  Corollary \ref{cor:majority_vote}}

Recall that
\begin{eqnarray}
  I(a,b)  =  \Pr(\theta \geq 0.5 | \theta \sim \mathrm{Beta}(a,b))  =  \frac{1}{B(a,b)}\int_{0.5}^{1} t^{a-1}(1-t)^{b-1} \mathrm{d} t,
  \label{eq:I}
\end{eqnarray}
where $B(a,b)$ is the beta function.

It is easy to see that $I(a,b) > 0.5 \Longleftrightarrow I(a,b) > 1- I(a,b)$.  We re-write $1-I(a,b)$  as follows
\begin{eqnarray*}
1-I(a,b) =  \frac{1}{B(a,b)} \int_{0}^{0.5} t^{a-1}(1-t)^{b-1} \mathrm{d} t =   \frac{1}{B(a,b)} \int_{0.5}^{1} t^{b-1}(1-t)^{a-1} \mathrm{d} t,
\end{eqnarray*}
where the second equality is obtained by setting $t:=1-t$. Then we have:
\begin{eqnarray*}
    I(a,b)-(1-I(a,b))  &= & \frac{1}{B(a,b)}  \int_{0.5}^{1} (t^{a-1}(1-t)^{b-1}- t^{b-1}(1-t)^{a-1}) \mathrm{d} t \\
  & = & \frac{1}{B(a,b)}   \int_{0.5}^{1} t^{a-1}(1-t)^{b-1} \left( \left(\frac{t}{1-t} \right)^{a-b}-1 \right ) \mathrm{d} t
\end{eqnarray*}
Since $t > 0.5$, $\frac{t}{1-t}  > 1$. When $a>b$,  $\left(\frac{t}{1-t} \right)^{a-b}>1$ and hence $I(a,b)-(1-I(a,b)) >0$, i.e, $I(a,b)>0.5$. When $a=b$,  $\left(\frac{t}{1-t} \right)^{a-b} \equiv 1$ and  $I(a,b)=0.5$. When $a<b$, $\left(\frac{t}{1-t} \right)^{a-b}<1$ and   $I(a,b)<0.5$.


\section*{Proof of Proposition \ref{prop:reward}}

We use the proof technique in \cite{Xie:12} to prove Proposition \ref{prop:reward}. According to \eqref{eq:value_func}, the value function takes the following form,
\begin{align}
  V(S^0)  = \sup_{\pi } \E^{\pi} \left( \sum_{i=1}^K h(P_i^T) \right).
\end{align}

To decompose the final accuracy $\sum_{i=1}^K h(P_i^T)$ into the incremental reward at each stage, we define $G_0= \sum_{i=1}^K h(P_i^0) $ and $G_{t+1}= \sum_{i=1}^K h(P_i^{t+1}) -  \sum_{i=1}^K h(P_i^{t})$. Then, $ \sum_{i=1}^K h(P_i^T)$ can be decomposed as:  $ \sum_{i=1}^K h(P_i^T) \equiv G_0 + \sum_{t=0}^{T-1} G_{t+1}$. The value function can now be re-written as follows:

\begin{eqnarray*}
   V(S^0) & = &   G_0(S^0) + \sup_{\pi} \sum_{t=0}^{T-1} \E^{\pi} ( G_{t+1} ) \\
       &=  &  G_0(S^0) + \sup_{\pi} \sum_{t=0}^{T-1} \E^{\pi} \left( \E( G_{t+1} |\calF_t) \right) \\
         & = &  G_0(S^0) + \sup_{\pi} \sum_{t=0}^{T-1} \E^{\pi} \left( \E( G_{t+1} |S^{t}, i_t) \right).
\end{eqnarray*}

Here, the first inequality is true because $G_0$ is determinant and independent of $\pi$; the second inequality is due to the tower property of conditional expectation and the third one holds because $G_{t+1}$, which is a function of $P_i^{t+1}$ and $P_i^t$, depends on $\calF_t$ only through $S^t$ and $i_t$. We define incremental expected reward gained by labeling the $i_t$-th instance at the state $S^t$ as follows:

\begin{eqnarray}
R(S^t, i_t)&   = & \E( G_{t+1} |S^{t}, i_t)  = \E\left(\sum_{i=1}^K h(P_i^{t+1}) -  \sum_{i=1}^K h(P_i^{t}) |S^{t}, i_t \right)    \nonumber \\
            &  = & \E \left( h(P_{i_t}^{t+1}) -   h(P_{i_t}^{t}) |S^{t}, i_t \right).
           \label{eq:exp_reward}
\end{eqnarray}

The last equation is due to the fact that only $P_{i_t}^t$ will be changed if the $i_t$-th instance is labeled next.  With the expected reward function in place, the value function in \eqref{eq:value_func} can be re-formulated as:

\begin{equation}
  V(S^0) = G_0(\bs) + \sup_{\pi} \E^{\pi} \left( \sum_{t=0}^{T-1}  R(S^t, i_t ) \Big| S^0 \right).
  \label{eq:value_function}
\end{equation}

\section*{Proof of Proposition \ref{prop:det_KG}}

To prove the failure of deterministic KG, we first show a key property for the expected reward function:
\begin{eqnarray}
\label{eq:reward_R}
  R(a,b) = \frac{a}{a+b} \left(h(I(a+1,b))-h(I(a,b)) \right)  +\frac{b}{a+b} \left(h(I(a,b+1))-h(I(a,b)) \right).
\end{eqnarray}

\begin{lemma}
When $a, b$ are positive integers, if $a= b$, $R(a,b)=\frac{0.5^{2a}}{aB(a,a)}$ and if $a \neq b$, $R(a,b)= 0$.
\label{lem:reward}
\end{lemma}

To prove lemma \ref{lem:reward}, we first present several basic properties for $B(a,b)$ and $I(a,b)$, which will be used in all the following theorems and proofs.
\begin{enumerate}
  \item  Properties for $B(a,b)$:
  \begin{align}
    & B(a,b)=B(b,a)     \label{eq:sym_B}\\
    & B(a+1,b)=\frac{a}{a+b} B(a,b)  \label{eq:B_a_1} \\
    & B(a, b+1)=\frac{b}{a+b} B(a,b) \label{eq:B_b_1}
  \end{align}
  \item Properties for $B(a,b)$:
  \begin{align}
    & I(a,b)=1-I(b,a)     \label{eq:sym_I}\\
    & I(a+1,b)=I(a,b)+ \frac{0.5^{a+b}}{a B(a,b)}   \label{eq:I_a_1} \\
    & I(a,b+1)=I(a,b)- \frac{0.5^{a+b}}{b B(a,b)}  \label{eq:I_b_1}
  \end{align}
  The properties for $I(a,b)$ are derived from the basic property of regularized incomplete beta function \footnote{\url{http://dlmf.nist.gov/8.17}}.
\end{enumerate}

\begin{proof}[Proof of Lemma \ref{lem:reward}]

When $a=b$, by Corollary \ref{cor:majority_vote}, we have $I(a+1,b)>0.5$, $I(a,b)=0.5$ and $I(a,b+1) < 0.5$. Therefore, the expected reward \eqref{eq:reward_R} takes the following form:
\begin{eqnarray*}
  R(a,b) & = & 0.5(I(a+1,a) -I(a,a))+   0.5((1-I(a,a+1)) -I(a,a))\\
         & = &  I(a+1,a) -I(a,a)  =  \frac{0.5^{2a}}{aB(a,a)}\\
\end{eqnarray*}

When $a > b$, since $a,b$ are integers,  we have $a \geq b+1$ and hence $I(a+1,b)>0.5, I(a,b)>0.5, I(a,b+1)\geq 0.5$ according to Corollary \ref{cor:majority_vote}. The expected reward \eqref{eq:reward_R} now becomes:

\begin{align*}
  R(a,b) =&  \frac{a}{a+b} I(a+1,b) + \frac{b}{a+b} I(a,b+1) -I(a,b)  \\
         =&   \frac{a}{a+b}  \frac{1}{B(a+1,b)} \int_{0.5}^{1} t\cdot t^{a-1}(1-t)^{b-1} \mathrm{d} t \\
          & +\frac{b}{a+b}  \frac{1}{B(a,b+1)} \int_{0.5}^{1} \ t^{a-1} (1-t) (1-t)^{b-1} \mathrm{d} t   -I(a,b)\\
        = & \frac{1}{B(a,b)} \int_{0.5}^{1} (t+(1-t)) \cdot t^{a-1}(1-t)^{b-1} \mathrm{d} t-I(a,b)\\
        =& I(a,b)-I(a,b)=0.
\end{align*}

Here we use \eqref{eq:B_a_1} and \eqref{eq:B_b_1} to show that $\frac{a}{a+b}  \frac{1}{B(a+1,b)}=\frac{b}{a+b}  \frac{1}{B(a,b+1)}=\frac{1}{B(a,b)}$.

When $a\leq b-1$, we can prove $R(a,b)=0$ in a similar way.
\end{proof}

With Lemma \ref{lem:reward} in place, the proof for Proposition \ref{prop:det_KG} is straightforward. 
Recall that the deterministic KG policy chooses the next instance according to $$i_t =\argmax_i R(S^t, i) =  \argmax_i R(a_i^t, b_i^t),$$and breaks the tie by selecting the one with the smallest index. Since $R(a,b)>0$ if and only if $a= b$, at the initial stage $t=0$, $R(a_i^0, b_i^0)>0$ for those instances $i \in \calE=\{i: a_i^0=b_i^0\}$.  The policy will first select  $i_0 \in \calE$ with the largest $R(a_i^0, b_i^0)$. After obtaining the label $y_{i_0}$, either $a_{i_0}^0$  or $b_{i_0}^0$ will add one and hence $a_{i_0}^1 \neq b_{i_0}^1$ and $R(a_{i_0}^1, b_{i_0}^1)=0$. The policy will select another instance $i_1 \in \calE$ with the ``current'' largest expected reward and the expected reward for $i_1$ after obtaining the label $y_{i_1}$ will then become zero. As a consequence, the KG policy will label each instance in $\calE$ for the first $|\calE|$ stages and $R(a_i^{|\calE|}, b_i^{|\calE|})=0$ for all $i\in\{1,\ldots, K\}$. Then the deterministic policy will break the tie selecting the first instance to label. From now on, for any $t \geq |\calE|$, if $a_1^t  \neq b_1^t$, then the expected reward $R(a_1^t,b_1^t)=0$. Since the expected reward for other instances are all zero, the policy will still label the first instance. On the other hand, if $a_1^t = b_1^t$, and the first instance is the only one with the positive expected reward and the policy will label it. Thus Proposition \ref{prop:det_KG} is proved.

\begin{remark}
  For randomized KG, after getting one  label for each instance in $\calE$ for the first $|\calE|$ stages, the expected reward for each instance has become zero. Then randomized KG will uniformly select one instance to label. At any stage $t \geq |\calE|$, if there exists one instance $i$ (at most one instance) with $a_i^t=b_i^t$, the KG policy will provide the next label for $i$; otherwise, it will randomly select an instance to label.
\end{remark}

\section*{Proof of Theorem \ref{thm:opt_KG}}

To prove the consistency of the Opt-KG policy, we first show the exact values for $R^+_{\alpha}(a,b)=\max(R_1(a,b), R_2(a,b))$.
\begin{enumerate}
  \item When $a \geq b+1$:
  \begin{align*}
    R_1(a,b) &  = I(a+1,b) - I(a,b) =\frac{0.5^{a+b}}{a B(a,b)} > 0 ; \\
    R_2(a,b) & = I(a,b+1) -I(a,b)= - \frac{0.5^{a+b}}{b B(a,b)} < 0.
  \end{align*}
  Therefore,
  \begin{equation*}
      R^+(a,b) =R_1(a,b)=\frac{0.5^{a+b}}{a B(a,b)} > 0.
  \end{equation*}
  \item When $a=b$:
  \begin{align*}
        R_1(a,b) &= I(a+1,a) - I(a, a) = \frac{0.5^{2a}}{a B(a,a)} ; \\
        R_2(a,b) &= 1-I(a,a+1) - I(a, a) = \frac{0.5^{2a}}{a B(a,a)}.
  \end{align*}
  Therefore, we have $R_1=R_2$ and
  \begin{equation*}
      R^+(a,b) =R_1(a,b)=R_2(a,b)= \frac{0.5^{2a}}{a B(a,a)}> 0.
  \end{equation*}
  \item When $b-1 \geq a  $:
  \begin{align*}
      R_1(a,b) & = I(a,b) - I(a+1,b)=-\frac{0.5^{a+b}}{a B(a,b)}<0 ;\\
      R_2(a,b) & = I(a,b)-I(a,b+1)=\frac{0.5^{a+b}}{b B(a,b)}>0 .
  \end{align*}
  Therefore
  \begin{equation*}
      R^+(a,b) =R_2(a,b)= \frac{0.5^{a+b}}{b B(a,b)} > 0.
  \end{equation*}
\end{enumerate}
We note that the values of $R^+(a,b)$ for different $a,b$  are plotted in Figure \ref{fig:cvar_right} in main text.

As we can see $R^+(a,b)>0$ for any positive integers $(a,b)$, we first prove that
\begin{equation}
\lim_{a+b \rightarrow \infty} R^{+}(a,b)=0
\label{eq:lim_R+}
\end{equation}
in the following Lemma.

\begin{lemma}Properties for   $R^+(a,b)$:
\begin{enumerate}
  \item $R(a,b)$ is symmetric, i.e., $R^+(a,b)=R^+(b,a)$.
  \item $\lim_{a \rightarrow \infty } R^+(a,a)=0$.
  \item For any fixed $a \geq 1 $, $R^+(a+k,a-k)=R^+(a-k,a+k)$ is monotonically decreasing in $k$ for $k=0, \ldots, a-1$.
  \item When $a \geq b$,  for any fixed $b$, $R^+(a,b)$ is monotonically decreasing in $a$. By the symmetry of $R^+(a,b)$, when $b \geq a$,  for any fixed $a$,  $R^+(a,b)$ is monotonically decreasing in $b$.
\end{enumerate}
  By the above four properties, we have $\lim_{(a+b) \rightarrow \infty} R^{+}(a,b)=0$.
  \label{lem:R+}
\end{lemma}

\begin{proof}[Proof of Lemma \ref{lem:R+}]

We first prove these four properties.

\begin{itemize}
  \item Property 1: By the fact that $B(a,b)=B(b,a)$, the symmetry of $R^+(a,b)$ is straightforward.
  \item Property 2: For $a >1 $,  $\frac{R^+(a,a)}{R^+(a-1,a-1)}=\frac{2a-1}{2a}<1$ and hence $R^+(a,a)$ is monotonically decreasing in $a$. Moreover,
                    \begin{align*}
                      R^+(a,a)&=R^{+}(1,1) \prod_{i=2}^{a} \frac{2i-1}{2i}  = R^{+}(1,1) \prod_{i=2}^{a} (1-\frac{1}{2i}) \leq R^{+}(1,1) e^{ -\sum_{i=2}^{a} \frac{1}{2i} }
                    \end{align*}
                    Since $\lim_{a \rightarrow \infty} \sum_{i=2}^{a} \frac{1}{2i}  = \infty$ and $R^+(a,a) \geq 0$,   $\lim_{a \rightarrow \infty} R^+(a,a)=0$.
  \item Property 3: For any $k \geq 0$,
                    \begin{align*}
                         \frac{R^+(a+(k+1), a-(k+1))}{R^+(a+k, a-k)}  =  \frac{(a+k)B(a+k, a-k)}{(a+k+1)B(a+(k+1), a-(k+1))}
                         =  \frac{a-(k+1)}{a+(k+1)} <1.
                   \end{align*}
  \item Property 4: When $a \geq b$, for any fixed $b$:
                    \begin{align*}
                       \frac{R^+(a+1, b)}{R^+(a, b)}=  \frac{aB(a,b)}{2(a+1)B(a+1,b)}
                         =  \frac{a(a+b)}{2a(a+1)} <1.
                    \end{align*}
\end{itemize}

According  to the third property, when $a+b$ is an even number, we have $R^+(a,b) < R^+(\frac{a+b}{2},\frac{a+b}{2})$. According to the fourth property, when $a+b$ is an odd number and $a \geq b+1$, we have $R^+(a,b) < R^{+}(a-1,b) < R^+(\frac{a+b-1}{2},\frac{a+b-1}{2})$; while when $a+b$ is an odd number and $a \leq b-1$,  we have $R^+(a,b) < R^{+}(a,b-1) < R^+(\frac{a+b-1}{2},\frac{a+b-1}{2})$. Therefore,
\begin{equation*}
  R^+(a,b) < R^+\left( \lfloor \frac{a+b}{2} \rfloor, \lfloor\frac{a+b}{2} \rfloor\right).
\end{equation*}
According to the second property such that $\lim_{a \rightarrow \infty } R^+(a,a)=0$, we obtain \eqref{eq:lim_R+}.
\end{proof}

Using Lemma \ref{lem:R+}, we first show that, in any sample path, the Opt-KG will label each instance infinitely many times  as $T$ goes to infinity. Let $\eta_i(T)$ be a random variable representing the number of times that the $i$-th instance has been labeled until the stage $T$ using Opt-KG. Given a sample path $\omega$, let $\calI(\omega) = \{i: \lim_{T \rightarrow \infty} \eta_i(T)(\omega) < \infty\}$ be the set of instances that has been labeled only finite number of times as $T$ goes to infinity in this sample path. We need to prove that $\calI(\omega)$ is an empty set for any $\omega$. We prove it by contradiction. Assuming that $\calI(\omega)$  is not empty, then after a certain stage $\widehat{T}$, instances in $\calI(\omega)$ will never be labeled.  By Lemma \ref{lem:R+}, for any $j \in \calI^c$,  $\lim_{T \rightarrow \infty} R^+(a_j^T(\omega), b_j^T(\omega)) =0$. Therefore, there will exist $\bar{T} > \widehat{T}$ such that:
\begin{eqnarray*}
  \max_{j \in \calI^c} R^+(a_j^{\bar{T}}(\omega), b_j^{\bar{T}}(\omega))& <& \max_{i \in \calI} R^+(a_i^{\widehat{T}}(\omega), b_i^{\widehat{T}}(\omega)) =  \max_{i \in \calI} R^+(a_i^{\bar{T}}(\omega), b_i^{\bar{T}}(\omega)).
\end{eqnarray*}
Then according to the Opt-KG policy, the next instance to be labeled must be in $\calI(\omega)$, which leads to the contradiction.  Therefore, $\calI(\omega)$ will be an empty set for any  $\omega$. 

 Let $Y_i^s$ be the random variable which takes the value $1$ if the $s$-th label of the $i$-th instance is 1 and the value $-1$ if the $s$-th label is 0. It is easy to see that $\E(Y_i^s|\theta_i)=\Pr(Y_i^s=1|\theta_i)=\theta_i$. Hence, $Y_i^s$, $s=1,2,\dots$ are independent and identically distributed  random variables. By  the fact that $\lim_{T \rightarrow \infty} \eta_T(i)=\infty$ in all sample paths and using the strong law of large number, we conclude that, conditioning on $\theta_i$, $i=1,\dots,K$, the conditional probability of
\begin{equation*}
  \lim_{T \rightarrow \infty} \frac{a_i^T-b_i^T}{\eta_i(T)} =  \lim_{T \rightarrow \infty} \frac{\sum_{s=1}^{\eta_i(T)} Y_i^s}{\eta_i(T)} =\E(Y_i^s | \theta_i)=2\theta_i-1
\end{equation*}
for all $i=1,\dots,K$, is one.
According to Proposition \ref{prop:H}, we have $H_T=\{i: a_i^T \geq b_i^T\}$ and $H^*=\{i: \theta_i \geq 0.5\}$. The accuracy is $\text{Acc}(T)=\frac{1}{K} \left(|H_T \cap H^*|+|H_T^c \cap (H^*)^c| \right).$ We have:

\begin{align*}
     & \Pr(\lim_{T \rightarrow \infty}\text{Acc}(T)=1|\{\theta_i\}_{i=1}^K)
    =\Pr\left( \lim_{T \rightarrow \infty} (|H_T \cap H^*|+|H_T^c \cap (H^*)^c|) =K | \{\theta_i\}_{i=1}^K\right)  \\
                        \geq& \Pr\left( \lim_{T \rightarrow \infty} \frac{a_i^T-b_i^T}{\eta_i(T)} =2\theta_i-1 , \forall i=1,\ldots, K |\{\theta_i\}_{i=1}^K \right)
                        =1,
\end{align*}

whenever $\theta_i\neq 0.5$ for all $i$.
The last inequality is due to the fact that, as long as $\theta_i$ is not $0.5$ in any $i$, any sample path that gives the event $\lim_{T \rightarrow \infty} \frac{a_i^T-b_i^T}{\eta_i(T)} =2\theta_i-1 , \forall i=1,\ldots, K$ also gives the event $\lim_{T \rightarrow \infty}(a_i^T-b_i^T)=\text{sgn}(2\theta_i-1)(+\infty) $, which further implies $ \lim_{T \rightarrow \infty} (|H_T \cap H^*|+|H_T^c \cap (H^*)^c|) =K $.

Finally, we have:

\begin{align*}
     \Pr\left(\lim_{T \rightarrow \infty}\text{Acc}(T)=1 \right)
&=  \E_{\{\theta_i\}_{i=1}^K}\left[\Pr\left(\lim_{T \rightarrow \infty}\text{Acc}(T)=1|\{\theta_i\}_{i=1}^K\right)\right]  \\
& =  \E_{\{\theta_i:\theta_i\neq 0.5\}_{i=1}^K}\left[\Pr\left(\lim_{T \rightarrow \infty}\text{Acc}(T)=1|\{\theta_i\}_{i=1}^K \right)\right]\\
&=  \E_{\{\theta_i:\theta_i\neq 0.5\}_{i=1}^K}\left[1\right]=1,
\end{align*}

where the second equality is because $\{\theta_i:\exists i, \theta_i=0.5\}$ is a zero measure set.

\section*{Proof of Proposition \ref{prop:pessimistic_KG}}

Recall that our random reward is a two-point distribution with the probability $p_1= \frac{a}{a+b}$ of being $R_1(a,b)  =  h(I(a+1,b))-h(I(a,b))$ and $p_2=\frac{b}{a+b}$ of being $R_2(a,b)= h(I(a,b+1))-h(I(a,b))$. The pessimistic KG selects the next instance which maximizes $R^{-}(a,b)= \min(R_1(a,b), R_2(a,b))$. To show that the policy is inconsistent, we first compute the exact values for $R^{-}(a,b)$ for positive integers $(a,b)$.

Utilizing Corollary \ref{cor:majority_vote} and the basic properties of $I(a,b)$ in \eqref{eq:sym_I}, \eqref{eq:I_a_1}, \eqref{eq:I_b_1},  we have:
\begin{enumerate}
  \item When $a \geq b+1$:
  \begin{align*}
    R_1(a,b) &  = I(a+1,b) - I(a,b) =\frac{0.5^{a+b}}{a B(a,b)} > 0 ; \\
    R_2(a,b) & = I(a,b+1) -I(a,b)= - \frac{0.5^{a+b}}{b B(a,b)} < 0.
  \end{align*}
  Therefore,
  \begin{equation*}
      R^-(a,b) =R_2(a,b)=- \frac{0.5^{a+b}}{b B(a,b)} < 0.
  \end{equation*}
  \item When $a=b$:
  \begin{align*}
        R_1(a,b) &= I(a+1,a) - I(a, a) = \frac{0.5^{2a}}{a B(a,a)} ; \\
        R_2(a,b) &= 1-I(a,a+1) - I(a, a) = \frac{0.5^{2a}}{a B(a,a)}.
  \end{align*}
  Therefore, we have $x_1=x_2$ and
  \begin{equation*}
      R^-(a,b) =R_1(a,b)=R_2(a,b)= \frac{0.5^{2a}}{a B(a,a)}> 0.
  \end{equation*}
  \item When $b-1 \geq a  $:
  \begin{align*}
      R_1(a,b) & = I(a,b) - I(a+1,b)=-\frac{0.5^{a+b}}{a B(a,b)}<0 ;\\
      R_2(a,b) & = I(a,b)-I(a,b+1)=\frac{0.5^{a+b}}{b B(a,b)}>0 .
  \end{align*}
  Therefore
  \begin{equation*}
      R^-(a,b) =R_1(a,b)=- \frac{0.5^{a+b}}{a B(a,b)} < 0.
  \end{equation*}

\end{enumerate}

We summarize the properties of $R^-(a,b)$ in the next Lemma.
\begin{lemma} Properties for   $R^-(a,b)$:
\begin{enumerate}
  \item $R^-(a,b)>0 $ if and only if $a=b$.
  \item $R^-(a,b)$ is symmetric, i.e., $R^-(a,b)=R^-(b,a)$
  \item When $a=b+1$, then $R^-(a,b)=R^-(b+1,b)$ is monotonically increasing in $b$. By the symmetry of $R^-(a,b)$, when $b=a+1$, $R^-(a,b)=R^-(a,a+1)$ is monotonically increasing in $a$.
  \item When $a \geq b+1$, for any fixed $b$, $R^-(a,b)$ is monotonically increasing in $a$. By the symmetry of $R^-(a,b)$,  when $b \geq a+1$, for any fixed $a$, $R^-(a,b)$ is monotonically increasing in $b$.
\end{enumerate}
\label{lem:cvar_left}
\end{lemma}
For better visualization, we plot values of $R^-(a,b)$ for different $a,b$ in Figure \ref{fig:cvar_left}.  All the properties in Lemma \ref{lem:cvar_left} can be seen clearly from Figure \ref{fig:cvar_left}. The proof of these properties are based on simple algebra and thus omitted here.
\begin{figure}[!t]
\centering
  \includegraphics[width=0.4\textwidth]{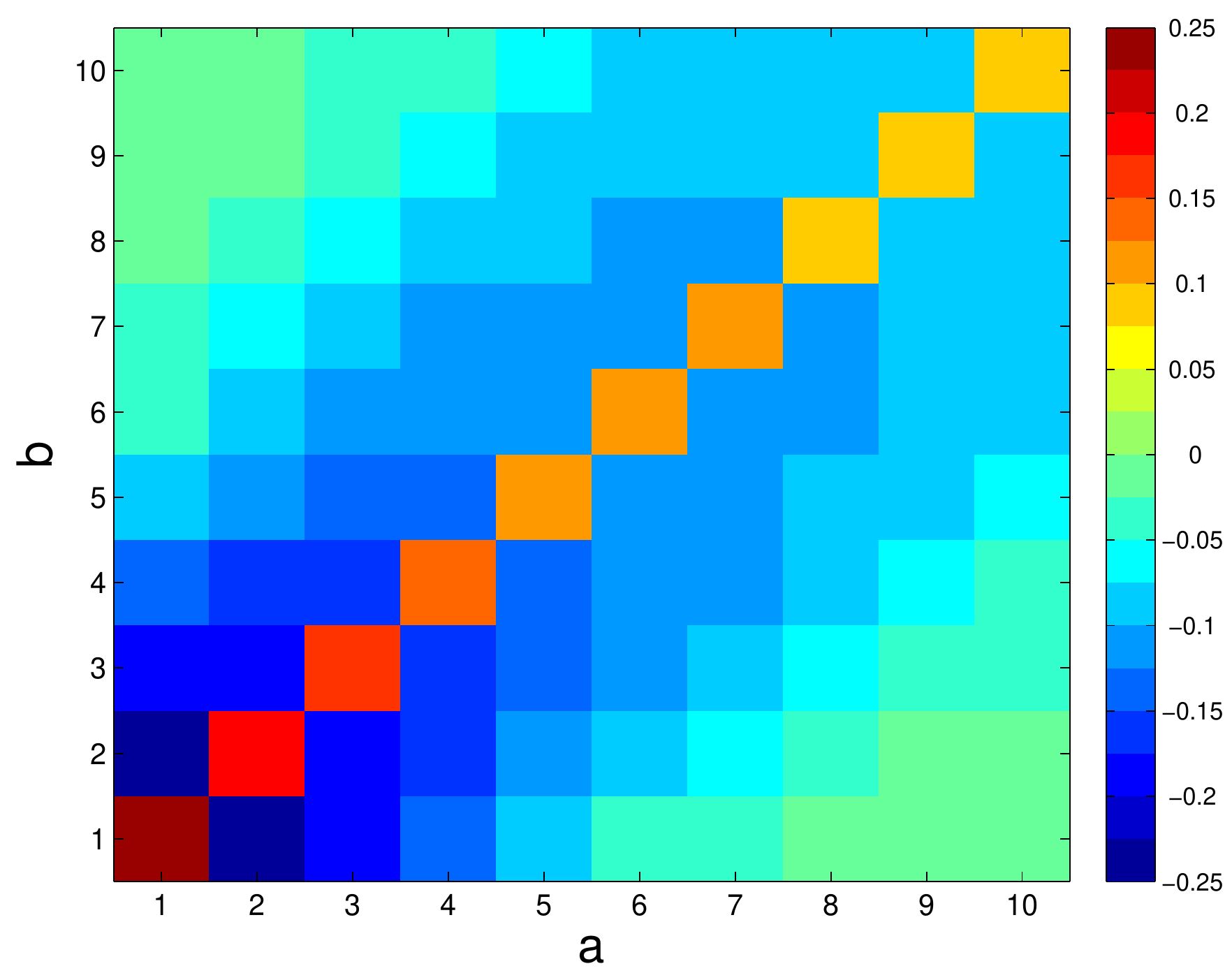}
  \vspace{-3mm}
  \caption{Illustration of $R^{-}(a,b)$.}
  \label{fig:cvar_left}
\end{figure}

From Lemma \ref{lem:cvar_left}, we can conclude that for any positive integers $a,b$ with $a+b \neq 3$:
\begin{eqnarray}
\label{eq:cvar_1_2}
R^-(1,2) = R^-(2,1) < R^-(a,b).
\end{eqnarray}
Recall that the pessimistic KG selects:
$$i_t =\argmax_{i\in \{1, \ldots, K\}} R^-(a_i^t, b_i^t).$$
When starting from the uniform prior with $a_i^0=b_i^0=1$ for all $i\in \{1\ldots, K\}$, the corresponding $R^-(a_i^0,b_i^0)=R^-(1,1)>0$. After obtaining a label for any instance $i$, the Beta parameters for $\theta_i$ will become either $(2,1)$ or $(1,2)$ with $R^-(1,2) = R^-(2,1) <0$. Therefore, for the first $K$ stages, the pessimistic KG policy will acquire the label for each instance once. For any instance $i$, we have either $a_i^{K}=2, b_i^{K}=1$ or $a_i^{K}=1, b_i^{K}=2$ at the stage $K$. Then the pessimistic KG policy will select the first instance to label. According to \eqref{eq:cvar_1_2}, for any $t \geq K$,  $R^-(a^t_1,b^t_1) > R^-(1,2)= R^-(2,1)$. Therefore, the pessimistic KG policy will consistently acquire the label for the first instance. Since the tie will only appear at the stage $K$, the randomized pessimistic KG  will also consistently select a single instance to label after $K$ stages.

\section*{Incorporate Reliability of Heterogeneous Workers}


As we discussed in Section \ref{sec:worker} in main text, we approximate the posterior so that at any stage for all $i, j$,  $\theta_i$ and $\rho_j$ will follow  Beta distributions. In particular, assuming at the current state $\theta_i \sim \B(a_i,b_i)$ and $\rho_j \sim \B(c_j,d_j)$, the posterior distribution conditioned on $z_{ij}$ takes the following form:
\begin{eqnarray*}
  p(\theta_i, \rho_j |z_{ij}=1) = \frac{\Pr(z_{ij}=1| \theta_i, \rho_j) \B(a_i,b_i)\B(c_j,d_j)}{\Pr(z_{ij}=1)} \\
   p(\theta_i, \rho_j |z_{ij}=-1) = \frac{\Pr(z_{ij}=-1| \theta_i, \rho_j) \B(a_i,b_i)\B(c_j,d_j)}{\Pr(z_{ij}=-1)}
\end{eqnarray*}
where the likelihood $\Pr(z_{ij}=z| \theta_i, \rho_j)$  for $z=1,-1$ is defined in \eqref{eq:Z_1} and \eqref{eq:Z_0}, i.e.,
\begin{eqnarray*}
  \Pr(z_{ij}= 1| \theta_i, \rho_j) & = &  \theta_i \rho_j + (1-\theta_i)(1-\rho_j) \\
  \Pr(z_{ij}= -1| \theta_i, \rho_j) & = & (1-\theta_i) \rho_j + \theta_i(1-\rho_j)
\end{eqnarray*}
Also,
\begin{align*}
  \Pr(z_{ij}=1)&=\E(\Pr(z_{ij}=1| \theta_i, \rho_j))               =\E(\theta_i)\E(\rho_j) + (1-\E(\theta_i)) (1-\E(\rho_j)) \\
               &=\frac{a_i}{a_i + b_i} \frac{c_j}{c_j+d_j}+\frac{b_i}{a_i + b_i} \frac{d_j}{c_j+d_j}.
\end{align*}
\begin{align*}
  \Pr(z_{ij}=-1)&=\E(\Pr(z_{ij}=-1| \theta_i, \rho_j))
               =(1-\E(\theta_i))\E(\rho_j) + \E(\theta_i) (1-\E(\rho_j)) \\
               &=\frac{b_i}{a_i + b_i} \frac{c_j}{c_j+d_j}+\frac{a_i}{a_i + b_i} \frac{d_j}{c_j+d_j}.
\end{align*}

The posterior distributions $p(\theta_i, p_j |z_{ij}=z )$ no longer takes the form of the product of Beta distributions on $\theta_i$ and $p_j$. Therefore, we use variational approximation by first assuming the conditional independence of $\theta_i$ and $\rho_j$:
\begin{eqnarray*}
p(\theta_i, \rho_j | z_{ij}=z) \approx  p(\theta_i| z_{ij}=z )p(\rho_j| z_{ij}=z )
\end{eqnarray*}
In fact, the exact form of marginal distributions can be calculated as follows:
\begin{align*}
  p(\theta_i| z_{ij}=1)&=  \frac{\theta_i\E(\rho_j) + (1-\theta_i) (1-\E(\rho_j))}{\Pr(z_{ij}=1)}\B(a_i,b_i) \\
  p(\rho_j| z_{ij}=1) &=  \frac{\E(\theta_i)\rho_j + (1-\E(\theta_i)) (1-\rho_j)}{\Pr(z_{ij}=1)}\B(c_j,d_j) \\
  p(\theta_i| z_{ij}=-1)&=  \frac{(1-\theta_i)\E(\rho_j) + \theta_i (1-\E(\rho_j))}{\Pr(z_{ij}=-1)}\B(a_i,b_i) \\
  p(\rho_j| z_{ij}=-1) &=  \frac{(1-\E(\theta_i))\rho_j + \E(\theta_i) (1-\rho_j)}{\Pr(z_{ij}=-1)}\B(c_j,d_j).
\end{align*}
To approximate the marginal distribution as Beta distribution, we use the moment matching technique. In particular, we approximate
$ p\left(\theta_i| z_{ij}=z\right) \approx  \B(\tilde{a}_i(z), \tilde{b}_i(z)) $  such that
\begin{align}
\tE_z(\theta_i) \doteq \E_{p(\theta_i| z_{ij}=z)} (\theta_i) &= \frac{\tilde{a}_i(z)}{\tilde{a}_i(z)+\tilde{b}_i(z)},  \label{eq:mom_theta_1}\\
\tE_z(\theta_i^2) \doteq \E_{p(\theta_i| z_{ij}=z)} (\theta_i^2) & = \frac{\tilde{a}_i(z)(\tilde{a}_i(z)+1)}{(\tilde{a}_i(z)+\tilde{b}_i(z))(\tilde{a}_i(z)+\tilde{b}_i(z)+1)},\label{eq:mom_theta_2}
\end{align}
where $\frac{\tilde{a}_i(z)}{\tilde{a}_i(z)+\tilde{b}_i(z)}$ and $\frac{\tilde{a}_i(z)(\tilde{a}_i(z)+1)}{(\tilde{a}_i(z)+\tilde{b}_i(z))(\tilde{a}_i(z)+\tilde{b}_i(z)+1)}$ are the first and second order moment of $\B(\tilde{a}_i(z), \tilde{b}_i(z))$.
To make \eqref{eq:mom_theta_1} and \eqref{eq:mom_theta_2} hold, we have:
\begin{align}
  \tilde{a}_i(z) &=\tE_z(\theta_i) \frac{\tE_z(\theta_i)-\tE_z(\theta_i^2)}{\tE_z(\theta_i^2)-\left(\tE_z(\theta_i)\right)^2}, \label{eq:a_worker} \\
  \tilde{b}_i(z) &=(1-\tE_z(\theta_i)) \frac{\tE_z(\theta_i)-\tE_z(\theta_i^2)}{\tE_z(\theta_i^2)-\left(\tE_z(\theta_i)\right)^2}. \label{eq:b_worker}
\end{align}
Similarly,  we approximate $p\left(\rho_j| z_{ij}=z \right) \approx  \B(\tilde{c}_j(z), \tilde{d}_j(z))$, such that
\begin{align}
\tE_z(\rho_j) \doteq \E_{p(\rho_j| z_{ij}=z)} (\rho_j) &= \frac{\tilde{c}_j(z)}{\tilde{c}_j(z)+\tilde{d}_j(z)},  \label{eq:mom_rho_1}\\
\tE_z(\rho_j^2) \doteq \E_{p(\rho_j| z_{ij}=z)} (\rho_j^2) & = \frac{\tilde{c}_j(z)(\tilde{c}_j(z)+1)}{(\tilde{c}_j(z)+\tilde{d}_j(z))(\tilde{c}_j(z)+\tilde{d}_j(z)+1)},\label{eq:mom_rho_2}
\end{align}
where $\frac{\tilde{c}_j(z)}{\tilde{c}_j(z)+\tilde{d}_j(z)}$ and $\frac{\tilde{c}_j(z)(\tilde{c}_j(z)+1)}{(\tilde{c}_j(z)+\tilde{d}_j(z))(\tilde{c}_j(z)+\tilde{d}_j(z)+1)}$ are the first and second order moment of $\B(\tilde{c}_j(z), \tilde{d}_j(z))$.
To make \eqref{eq:mom_theta_1} and \eqref{eq:mom_theta_2} hold, we have:
\begin{align}
  \tilde{c}_j(z) &=\tE_z(\rho_j) \frac{\tE_z(\rho_j)-\tE_z(\rho_j^2)}{\tE_z(\rho_j^2)-\left(\tE_z(\rho_j)\right)^2}, \label{eq:c_worker}\\
  \tilde{d}_j(z) &=(1-\tE_z(\rho_j)) \frac{\tE_z(\rho_j)-\tE_z(\rho_j^2)}{\tE_z(\rho_j^2)-\left(\tE_z(\rho_j)\right)^2}. \label{eq:d_worker}
\end{align}

Furthermore, we can compute the exact values for $\tE_z(\theta_i)$, $\tE_z(\theta_i^2)$, $\tE_z(\rho_j)$ and $\tE_z(\rho_j^2)$ as follows.
\begin{align*}
  \tE_1(\theta_i) & =  \frac{\E(\theta_i^2)\E(\rho_j)+ (\E(\theta_i)-\E(\theta_i^2))(1-\E(\rho_j))}{ p(z_{ij}=1)}
                  = \frac{a_i((a_i+1)c_j+b_id_j)}{(a_i+b_i+1)(a_ic_j+b_id_j)} ; \\
  \tE_1(\theta_i^2) & =  \frac{\E(\theta_i^3)\E(\rho_j)+ (\E(\theta_i^2)-\E(\theta_i^3))(1-\E(\rho_j))}{ p(z_{ij}=1)}
                      = \frac{a_i(a_i+1)((a_i+2)c_j+b_id_j)}{(a_i+b_i+1)(a_i+b_i+2)(a_ic_j+b_id_j)}; \\
  \tE_{-1}(\theta_i) & =  \frac{(\E(\theta_i)-\E(\theta_i^2))\E(\rho_j) + \E(\theta_i^2)(1-\E(\rho_j))}{ p(z_{ij}=-1)}
                   =  \frac{a_i(b_ic_j+(a_i+1)d_j)}{(a_i+b_i+1)(b_ic_j+a_id_j)}; \\
  \tE_{-1}(\theta_i^2)& = \frac{(\E(\theta_i^2)-\E(\theta_i^3))\E(\rho_j) + \E(\theta_i^3)(1-\E(\rho_j))}{ p(z_{ij}=-1)}
                   =  \frac{a_i(a_i+1)(b_ic_j+(a_i+2)d_j)}{(a_i+b_i+1)(a_i+b_i+2)(b_ic_j+a_id_j)}; \\
  \tE_1(\rho_j)  & = \frac{\E(\theta_i)\E(\rho_j^2)+ (1-\E(\theta_i))(\E(\rho_j)-\E(\rho_j^2))}{ p(z_{ij}=1)}
                  = \frac{c_j(a_i(c_j+1)+b_id_j)}{(c_j+d_j+1)(a_ic_j+b_id_j)} ; \\
  \tE_1(\rho_j^2) &   =  \frac{\E(\theta_i)\E(\rho_j^3)+ (1-\E(\theta_i))(\E(\rho_j^2)-\E(\rho_j^3))}{ p(z_{ij}=1)}
                   = \frac{c_j(c_j+1)(a_i(c_j+2)+b_id_j)}{(c_j+d_j+1)(c_j+d_j+2)(a_ic_j+b_id_j)}; \\
 \tE_{{-1}}(\rho_j)& =  \frac{(1-\E(\theta_i))\E(\rho_j^2) + \E(\theta_i)(\E(\rho_j)-\E(\rho_j^2))}{ p(z_{ij}=-1)}
                = \frac{c_j(b_i(c_j+1)+a_id_j)}{(c_j+d_j+1)(b_ic_j+a_id_j)}; \\
 \tE_{-1}(\rho_j^2)  & =  \frac{(1-\E(\theta_i))\E(\rho_j^3) + \E(\theta_i)(\E(\rho_j^2)-\E(\rho_j^3))}{ p(z_{ij}=-1)}
                   = \frac{c_j(c_j+1)(b_i(c_j+2)+a_id_j)}{(c_j+d_j+1)(c_j+d_j+2)(b_ic_j+a_id_j)}.
\end{align*}

Assuming that at a certain stage, $\theta_i$ follows a Beta posterior $\B(a_i,b_i)$ and $\rho_j$ follows a Beta posterior $\B(c_j,d_j)$, the reward of getting  positive and negative labels for the $i$-th instance from the $j$-th worker are:
\begin{align}
  R_1(a_i,b_i,c_j,d_j)&=h(I(\tilde{a}_i(z=1), \tilde{b}_i(z=1)))-h(I(a_i, b_i)) \\
  R_2(a_i,b_i,c_j,d_j)&=h(I(\tilde{a}_i(z=-1), \tilde{b}_i(z=-1)))-h(I(a_i, b_i)),
\end{align}
where $\tilde{a}_i(z=\pm 1)$ and  $\tilde{b}_i(z=\pm 1)$ are defined in \eqref{eq:a_worker} and \eqref{eq:b_worker}, which further depend on $c_j$ and $d_j$ through $\tE_z(\theta_i)$ and $\tE_z(\theta_i^2)$. With the reward in place,  we can directly apply the Opt-KG policy in the heterogeneous worker setting.

\section*{Extensions}

\subsection*{Utilizing Contextual Information}

When each instance is associated with a $p$-dimensional feature vector $\bx_i \in \mathbb{R}^p$, we incorporate the feature information in our budget allocation problem by assuming:
\begin{equation}
  \theta_i = \sigma(\langle \bw,  \bx_i \rangle) \doteq \frac{1}{1+\exp\{-\langle \bw,  \bx_i \rangle \}},
\end{equation}
where $\sigma(x)=\frac{1}{1+\exp\{-x\}}$ is the sigmoid function and $\bw$ is assumed to be drawn from a Gaussian prior  $N(\bmu_0, \bSigma_0)$. At the $t$-th stage with the state $S^t=(\bmu_t, \bSigma_t)$ and  $\bw \sim (\bmu_t, \bSigma_t)$, the decision maker chooses the $i_t$-th instance to be labeled and observes the label $y_{i_t} \in \{-1,1\}$. The posterior distribution $p(\bw | y_{i_t}, S^t) \propto p( y_{i_t} |\bw)p(\bw |S^t)$  has the following log-likelihood:
 \begin{align*}
& \ln  p(\bw | y_{i_t}, S^t)  = \ln p( y_{i_t} |\bw) + \ln p(\bw |S^t) + \text{const} \\
= & \mathbf{1}(y_{i_t}=1) \ln  \sigma(\langle \bw, \bx_{i_{t}} \rangle) + \mathbf{1}(y_{i_t}=-1)  \ln  \left( 1- \sigma(\langle \bw, \bx_{i_{t}} \rangle) \right)
  -\frac{1}{2} (\bw - \bmu_t)' \bOmega_t (\bw- \bmu_t)   +\text{const},
\end{align*}
where $\bOmega_t = (\bSigma_t)^{-1}$ is the precision matrix. To approximate  $p(\bw | y_{i_t}, \bmu_t, \bSigma_t)$ by a Gaussian distribution $N(\bmu_{t+1}, \bSigma_{t+1})$, we use the Laplace method (see Chapter 4.4 in \cite{Bishop:PRML}). In particular, the mean of the posterior Gaussian is the MAP (maximum a posteriori) estimator of $\bw$:
\begin{equation}
  \mu_{t+1} = \argmax_{\bw}  \ln  p(\bw | y_{i_t},S^t),
  \label{eq:post_mean}
\end{equation}
which can be computed by any numerical optimization method (e.g., Newton's method). The precision matrix takes the following form,
\begin{align*}
  \bOmega_{t+1}  & =  - \nabla^2 \ln  p(\bw | y_{i_t}, S^t) \big|_{\bw=\bmu_{t+1}}
               = \bOmega_t + \sigma(\bmu_{t+1}' \bx_{i_{t+1}})(1-\sigma(\bmu_{t+1}' \bx_{i_{t+1}})) \bx_{i_{t+1}}  \bx_{i_{t+1}}'.
\end{align*}
By Sherman-Morrison formula, the covariance matrix can be computed as,
\begin{align*}
   & \bSigma_{t+1} = (\bOmega_{t+1})^{-1}
  =  \bSigma_{t} -  \frac{ \sigma(\bmu_{t+1}'\bx_{i_{t}})(1-\sigma(\bmu_{t+1}\bx_{i_{t}}))}{1+ \sigma(\bmu_{t+1}'\bx_{i_{t}})(1-\sigma(\bmu_{t+1}'\bx_{i_{t}}))\bx_{i_{t}}'  \bSigma_{t} \bx_{i_{t}} }\bSigma_{t}  \bx_{i_{t+1}}\bx_{i_{t}}'\bSigma_{t}.
\end{align*}

We also calculate the transition probability of $y_{i_t}=1$ and $y_{i_t}=-1$  using the technique from Bayesian logistic regression (see Chapter 4.5 in \cite{Bishop:PRML}):
\begin{align*}
  \Pr(y_{i_t}=1 | S^t, i_t)  = \int p(y_{i_t}=1 | \bw) p(\bw | S^t ) \mathrm{d} \bw
                            = \int \sigma(\bw' \bx_i )  p(\bw | S^t) \mathrm{d} \bw
                             \approx \sigma(\mu_i\kappa(s_i^2)),
\end{align*}
where $\kappa(s_i^2)  = (1+\pi s_i^2/8)^{-1/2}$ and $\mu_i = \langle \bmu_t,  \bx_i \rangle$  and $s_i^2= \bx_i' \bSigma_t \bx_i$.

To calculate the reward function, in addition to the transition probability, we also need to compute:
\begin{align*}
  P_i^t &=\Pr( \theta_i \geq 0.5 | \calF_t) \\
        & = \Pr\left( \frac{1}{1+\exp\{-\bw'_t\bx_i\}} \geq 0.5 \Big| \bw_t \sim N(\bmu_t, \bSigma_t)\right)\\
        & = \Pr( \bw'_t \bx_i \geq 0 | \bw_t \sim N(\bmu_t, \bSigma_t)) \\
        & = \int_{0}^{\infty} \left( \int_{\bw} \delta(c- \langle \bw, \bx_i \rangle) N(\bw | \bmu_t, \bSigma_t) \mathrm{d} \bw \right)  \mathrm{d} c,
\end{align*}
where $\delta(\cdot)$ is the Dirac delta function. Let
\begin{eqnarray*}
   p(c)= \int_{\bw} \delta(c- \langle \bw, \bx_i \rangle) N(\bw | \bmu_t, \bSigma_t) \mathrm{d} \bw.
\end{eqnarray*}
Since the marginal of a Gaussian distribution is still a Gaussian, $p(c)$ is a univariate-Gaussian distribution with the mean and variance:
\begin{eqnarray*}
\mu_i & = & \E(c)=\langle  \E(\bw), \bx_i \rangle = \langle \bmu_t, \bx_i \rangle  \\
s_i^2 & = & \V(c)= (\bx_i)' \text{Cov} (\bw, \bw)  \bx_i = (\bx_i)' \bSigma_t   \bx_i.
\end{eqnarray*}
Therefore, we have:
\begin{eqnarray}
  P_i^t = \int_{0}^{\infty} p(c)  \mathrm{d} c  = 1  - \Phi \left(-\frac{\mu_i}{s_i} \right),
  \label{eq:P_i_t_feature}
\end{eqnarray}
where $\Phi(\cdot)$ is the CDF of the standard Gaussian distribution.

With $P_i^t$ and transition probability in place, the expected reward in value function takes the following form :
\begin{eqnarray}
 R(S^t, i_t) =  \E\left(\sum_{i=1}^K h(P_i^{t+1}) -  \sum_{i=1}^K h(P_i^{t}) \Big|S^{t}, i_t \right).
 \label{eq:reward_feature}
\end{eqnarray}
We note that since  $\bw$ will affect all $P_i^t$, the summation from 1 to $K$ in \eqref{eq:reward_feature} can not be omitted and hence \eqref{eq:reward_feature} cannot be written as $\E \left( h(P_{i_t}^{t+1}) -   h(P_{i_t}^{t}) |S^{t}, i_t \right)$ in \eqref{eq:exp_reward}. In this problem,  KG or Opt-KG need to solve $O(2TK)$ optimization problems to compute the mean of the posterior as in \eqref{eq:post_mean}, which could be computationally quite expensive. One possibility to address this problem is to use the variational Bayesian logistic regression \cite{Jaakkola:00}, which could lead to a faster optimization procedure.

\subsection*{Multi-Class Categorization}

Given the model and notations introduced in Section \ref{sec:multi}, at the final stage $T$ when all budget is used up, we construct the set $H^T_c$ for each class $c$ to maximize the conditional expected classification accuracy:

\begin{align}
  \{H_c^T\}_{c=1}^C & =  \argmax_{H_c\subseteq\{1,\dots,C\}, H_c \cap H_{\tilde{c}}=\emptyset} \E\left( \sum_{i=1}^K \sum_{c=1}^C  I(i \in H_c)  I(i \in H^*_c) \Bigg| \calF_T \right)  \nonumber \\
   &= \argmax_{H_c\subseteq\{1,\dots,C\}, H_c \cap H_{\tilde{c}}=\emptyset}   \sum_{i=1}^K \sum_{c=1}^C  I(i \in H_c)  \Pr \left(i \in H^*_c | \calF_T \right).
  \label{eq:multi_H}
\end{align}

Here, $H^*_c=\{i: \theta_{ic} \geq \theta_{ic'}, \forall c' \neq c  \}$ is the true set of instances in the  class $c$. The set $H_c^T$ consists of instances that belong to class $c$. Therefore, $\{H_c^T\}_{c=1}^C$ should form a partition of all instances $\{1,\ldots, K\}$. Let

\begin{eqnarray}
P^T_{ic}= \Pr(i \in H^*_c | \calF_T) = \Pr (  \theta_{ic} \geq \theta_{i\tilde{c}}, \;\; \forall \;\; \tilde{c}\neq c| \calF_T).
\label{eq:P_T_c}
\end{eqnarray}
To maximize the right hand side of \eqref{eq:multi_H}, we have
\begin{equation}
  H_c^T= \{i:  P^T_{ic}\geq P^T_{i\tilde{c}}, \forall \tilde{c} \neq c\}.
\end{equation}
If there is $i$ belongs to more than one $H_c^T$, we only assign it to the one with the smallest index $c$.
The maximum conditional expected accuracy takes the form: $\sum_{i=1}^K \left( \max_{c \in \{1\ldots, C\}} P_{ic}^T \right).$

Then the value function can be defined as:
\begin{align}
  V(S^0) & \doteq \sup_{\pi }\E^{\pi} \left( \E \Biggl(\sum_{i=1}^K \sum_{c=1}^C I(i\in H_{c}^T) I(i\in H^*_c) \Big| \calF_T \Biggr )  \right)= \sup_{\pi } \E^{\pi} \left( \sum_{i=1}^K h(\bP_i^T)  \right), \nonumber
\end{align}
where $\bP_i^T=(P_{i1}^T, \ldots, P_{iC}^T)$ and $h(\bP_i^T) \doteq  \max_{c \in \{1\ldots, C\}} P_{ic}^T.$ Following Proposition \ref{prop:reward}, let $P^t_{ic}= \Pr(i \in H^*_c | \calF_t)$ and $\bP_i^t =(P_{i1}^t, \ldots, P_{iC}^t)$, we define incremental reward function at each stage:
\begin{equation*}
R(S^t, i_t)=\E \left( h(\bP_{i_t}^{t+1}) -   h(\bP_{i_t}^{t}) |S^{t}, i_t \right).
\end{equation*}
The value function can be re-written as:
\begin{align*}
 V(S^0)   = G_0(S^0) + \sup_{\pi} \E^{\pi} \left( \sum_{t=0}^{T-1}  R(S^{t}, i_t ) \Big| S^0  \right),
\end{align*}
where $G_0(S^0)=\sum_{i=1}^K h(\bP_i^0) $. Since the reward function only depends on $S_{i_t}^t=\balpha_{i_t}^t \in \mathbb{R}_{+}^C$, we can define the reward function in a more explicit way by defining:
\begin{align*}
  R(\balpha)&=\sum_{c=1}^C \frac{\alpha_c}{\sum_{\tilde{c}=1}^C \alpha_{\tilde{c}}} h(I(\balpha+\bdelta_c))  -h(I(\balpha)).
\end{align*}
Here $\bdelta_c$ be a row vector of length $C$ with one at the $c$-th entry and zeros at all other entries; and $I(\balpha) =(I_1(\balpha), \ldots,  I_C(\balpha))$ where
\begin{eqnarray}
I_c(\balpha) =\Pr(\theta_c \geq \theta_{\tilde{c}}, \forall \tilde{c} \neq c | \theta \sim \Dir(\balpha)).
\label{eq:multi_I}
\end{eqnarray}
Therefore, we have $R(S^t, i_t)=R(\balpha_{i_t}^t)$.

To evaluate the reward $R(\balpha)$, the major bottleneck is how to compute $I_c(\balpha)$ efficiently.  Directly taking the $C$-dimensional integration on the region $\{\theta_c \geq \theta_{\tilde{c}}, \forall \tilde{c} \neq c\} \cap \Delta_C$ will be computationally very expensive, where $\Delta_C$ denotes the $C$-dimensional simplex. Therefore, we propose a method to convert the computation of  $I_c(\balpha)$ into a one-dimensional integration. It is known that to generate $\theta \sim \Dir(\balpha)$, it is equivalent to generate $\{X_c\}_{c=1}^C$ with $X_c \sim \text{Gamma}(\alpha_c, 1)$ and let $\theta_c \equiv \frac{X_c}{\sum_{c=1}^C X_c}$.  Then $\theta=(\theta_1, \ldots, \theta_C)$ will follow $\Dir(\balpha)$. Therefore, we have:
\begin{eqnarray}
  I_c(\balpha) =\Pr(X_c \geq X_{\tilde{c}}, \forall \tilde{c} \neq c | X_c \sim \text{Gamma}(\alpha_c, 1)).
\end{eqnarray}
It is easy to see that
\begin{align}
  \label{eq:I_C}
   I_c(\balpha)
    = &  \int_{0 \leq x_1 \leq x_c} \cdots   \int_{x_c \geq 0} \cdots  \int_{0 \leq  x_C \leq x_c} \prod_{c=1}^C f_\text{Gamma}(x_c ; \alpha_c, 1) \mathrm{d}x_1 \ldots \mathrm{d}x_C  \\
  = & \int_{x_c \geq 0} f_\text{Gamma}(x_c ; \alpha_c, 1) \prod_{\tilde{c} \neq c}  F_{\text{Gamma}}(x_c; \alpha_{\tilde{c}}, 1) \mathrm{d} x_c, \nonumber
\end{align}
where $f_{\text{Gamma}}(x; \alpha_{c}, 1)$  is the density function of Gamma distribution with the parameter $(\alpha_c, 1)$ and $F_{\text{Gamma}}(x_c; \alpha_{\tilde{c}}, 1)$ is  the CDF of Gamma distribution at $x_c$ with the parameter $(\alpha_{\tilde{c}}, 1)$. In many softwares,  $F_{\text{Gamma}}(x_c; \alpha_{\tilde{c}}, 1)$ can be calculated very efficiently without an explicit integration. Therefore, we can evaluate $I_c(\balpha)$ by  performing  only a one-dimensional numerical integration as in \eqref{eq:I_C}.
We could also use Monte-Carlo approximation to further accelerate the computation in \eqref{eq:I_C}.

\vskip 0.2in
\bibliographystyle{abbrv}
\bibliography{ref}

\end{document}